\documentclass[twoside,11pt]{article}

\usepackage{blindtext}

\usepackage{arxiv}

\usepackage{amsmath}

\usepackage{algorithm}
\usepackage{algpseudocode}

\usepackage{titletoc}
\newcommand\DoToC{%
  \startcontents
  \printcontents{}{1}{\vskip 1.5em\hrule\vskip .75em}
  \vskip .75em\hrule\vskip 2em
}

\usepackage{xcolor}

\usepackage{tikz}
\usetikzlibrary{decorations.pathreplacing, calc}
\usepackage{float} 

\newcommand{\tikzmark}[1]{\tikz[remember picture] \coordinate (#1);}

\newcommand{\EOT}{\operatorname{OT}_\varepsilon}
\newcommand{\OTL}{\operatorname{OT}^{\textnormal{loc}}_\varepsilon}

\usepackage{lastpage}

\ShortHeadings{Kokot and Luedtke}{Coreset Selection for Generic Smooth Divergences}
\firstpageno{1}

\begin{document}

\title{Coreset selection for the Sinkhorn divergence\\ and generic smooth divergences}
\author{\name Alex Kokot \email akokot@uw.edu \\
       \addr Department of Statistics\\
       University of Washington\\
       Seattle, WA 98195-4322, USA
       \AND
       \name Alex Luedtke \email aluedtke@uw.edu \\
       \addr Department of Statistics\\
       University of Washington\\
       Seattle, WA 98195-4322, USA}

\maketitle
              
\begin{abstract}
We introduce CO2, an efficient algorithm to produce convexly-weighted coresets with respect to generic smooth divergences. 
By employing a functional Taylor expansion, we show a local equivalence between sufficiently regular losses and their second order approximations, reducing the coreset selection problem to maximum mean discrepancy minimization. 
We apply CO2 to the Sinkhorn divergence, providing a novel sampling procedure that requires poly-logarithmically many data points to match the approximation guarantees of random sampling. 
To show this, we additionally verify several new regularity properties for entropically regularized optimal transport of independent interest. 
Our approach leads to a new perspective linking coreset selection and kernel quadrature to classical statistical methods such as moment and score matching.
We showcase this method with a practical application of subsampling image data, and highlight key directions to explore for improved algorithmic efficiency and theoretical guarantees.
\end{abstract}


\section{Introduction}
We consider the problem of compressing a dataset by selecting a convexly weighted subset of the observations, optimized for an error metric $D$. More concretely, given $n$ independent draws from a compactly supported $\mathbb{P}$ on $\mathcal{X}\subseteq\mathbb{R}^d$ with empirical distribution $\mathbb{P}_n$, we seek to construct $P_{m}$ supported on a `coreset' of $m$ of these observations
such that $D(\mathbb{P}_n, P_{m})$ is on the order of $D(\mathbb{P}_n, \mathbb{P})$ as $m,n\rightarrow\infty$. We call such a weighted coreset (asymptotically) lossless.
A typical example of $D$ is an integral probability metric \citep{10.1214/12-EJS722}, $D(\mathbb{P}_n, P_{m})=\sup_{f\in \mathcal{F}} |\int f d(\mathbb{P}_n- P_m)|$ --- when $\mathcal{F}$ is Donsker \citep{van2000asymptotic}, we shall seek a $P_m$ such that this quantity is $O_p(n^{-1/2})$. Another example of interest is the Sinkhorn divergence which we introduce in detail later \citep{ramdas2015wasserstein}. As we are concerned with empirical approximation of the population distribution, we frequently adopt the notation $D(Q) := D(Q, \mathbb{P}).$

Compression techniques are commonly employed in settings where it is costly to work with a large sample. In computational cardiology, for each configuration $x$ of a simulated heart, thousands of machine hours must be used to evaluate a function $f(x)$ that approximates the mechanism for calcium signaling \citep{strocchi2020simulating, dwivedi2023kernel}. This precludes evaluating common statistics of interest, such as $\mathbb{P}_n f$, if $\mathbb{P}_n$ is defined using a large sample of heart configurations. However, if $m$ configurations could be identified and weighted so that $P_m f\approx \mathbb{P}_n f$, the computational task could be made tractable. If we had knowledge that $f$ belonged to a function class $\mathcal{F}$, then this would be achieved by selecting $P_m$ to make the corresponding integral probability metric small.

The construction of such $P_m$ has been well studied in the case where $\mathcal{F}$ is the unit ball of a reproducing kernel Hilbert space \citep{scholkopf2002learning}. An IPM in this setting is referred to as maximum mean discrepancy (MMD) \citep{gretton2008kernel}. As these spaces are Donsker, the optimal error rate is $O(n^{-1/2})$.

Kernel herding was shown to achieve this error rate with $m = O(n^{1/2})$ in the case where the kernel is finite dimensional \citep{chen2012super}. 
Kernel thinning subsequently met this bound (up to logarithmic factors) in more general cases, depending on the kernel's tail behavior \citep{dwivedi2023kernel}. 
An alternative approach based on quadrature was developed in \citep{NEURIPS2022_2dae7d1c} and refined in \citep{li2024debiaseddistributioncompression}, leading to conditions depending on the spectral decay of the kernel when viewed as an integral operator. Key to our analysis is a kernel quadrature algorithm similar to that of \citep{li2024debiaseddistributioncompression}, tailored to data dependent kernels.

Often in unsupervised learning settings, data compression algorithms optimize for non-MMD error metrics. 
For example, $K$-means forms clusters to minimize the average squared deviation from the centroids \citep{steinley2006k}. 
This algorithm can be equivalently expressed in terms of the 2-Wasserstein distance, $W_2(\mu,\nu):= \inf_{\pi\in \Pi(\mu,\nu)} \left(\int d(x,y)^2\pi(x,y)\right)^{1/2}$ \citep{pollard1982quantization, canas2012learning}. In this framework, a distribution $P_K$ supported on $K$ centroids is 
selected to minimize $W_2(\mathbb{P}_n, P_K)$. 
Works such as \cite{claici2020wassersteinmeasurecoresets} fit more complex parametric distributions for the surrogate measure $P_K$, with Gaussian mixtures being a primary example.
By considering compressed data, we make the more restrictive assumption that $P_K \ll \mathbb{P}_n$. This forces the support of $P_K$ to retain any structure enjoyed by that of $\mathbb{P}_n$---for example, if $\mathbb{P}_n$ is supported on images of faces, then so too is $P_K$.

In practice, the Wasserstein distance is often too computationally expensive to be computed on large datasets, but it can be approximated efficiently using its entropically regularized counterpart \citep{cuturi2013sinkhorn}. The Sinkhorn divergence similarly approximates the Wasserstein loss, but with slight modifications to satisfy the properties of a statistical divergence \citep{ramdas2015wasserstein}. 
Beyond their computational benefits, these quantities are also easier to estimate than the Wasserstein distance, especially in high dimensions \citep{vacher2021dimension}, and are more robust to noise \citep{rigollet2018entropic, bigot2019datadriven}.

In this work, we explore whether divergences such as the Sinkhorn divergence can also yield benefits over the Wasserstein distance for coreset selection. 
This problem is important since, despite the Wasserstein distance's appealing relationship to familiar clustering algorithms, no compression algorithm can asymptotically improve on a random sample with respect to this loss \citep{weed2019sharp}.
Our approach culminates in a procedure applicable to arbitrary divergences, where MMD compression algorithms can be applied to analytically derived Taylor expansions of the functional of interest, a two-step procedure we call CO2 (Coresets of Order 2).

Our main contributions are as follows:
\begin{enumerate}
   \item We develop a framework for lossless data compression with respect to generic smooth divergences.

    \item We verify several new regularity results regarding quantities of interest in entropic optimal transport, particularly the entropic optimal transport potentials

    \item Using these results, we show that poly-logarithmically many data points $m$ are sufficient to ensure $D(P_m)=O_p(n^{-1})$ when $D$ is the Sinkhorn divergence.

\end{enumerate}
Precisely, we show $m=\omega(\log^d n)$ samples suffice for asymptotically lossless Sinkhorn reconstruction. To our knowledge, ours is the first work in the literature to even consider whether $m = o(n)$ samples suffices. The recent work of \cite{yinwasserstein}, which also developed a method for constructing Sinkhorn coresets, instead focused on providing stability and convergence guarantees for their proposed optimization scheme.

We highlight in our experiments how Sinkhorn-CO2 applied to MNIST data not only improves performance with respect to the Sinkhorn divergence, but also better approximates concrete quantities of interest, such as the proportions of each label present in the resulting coreset.
Other recent works have related coreset selection to improved efficiency in learning downstream tasks \citep{xiong2024fairwassersteincoresets, chenunified, gong2024supervised}.
This places these methods under the umbrella of  data distillation \citep{sachdeva2023datadistillationsurvey}, the study of compactly representing large datasets with the goal of improving algorithmic efficiency.
Our experiments demonstrate that Sinkhorn coresets provide a similar enhancement.

\section{The General Case}
\subsection{Notation and definitions}

Our results apply to sufficiently smooth divergences $D$, made precise by a second-order Hadamard differentiability condition \citep{Fernholz1983}. Denote by $\mathcal{P}(\mathcal{X})$ the space of probability distributions supported on $\mathcal{X}$.
For an RKHS $\mathcal{F}$ on $\mathcal{X}$ with norm $\|\cdot\|_{\mathcal{F}}$, let $\mathcal{F}_1:= \{f \in \mathcal{F}: \|f\|_{\mathcal{F}} \leq 1\}$ denote the unit ball and $\ell^\infty(\mathcal{F}_1)$ be the space of distributions on $\mathcal{X}$ with pseudometric $\|\mu\|_{\ell^\infty(\mathcal{F}_1)}:= \sup_{\|f\|_\mathcal{F} \leq 1}|\int f d\mu|$. 
 We say that a functional $D$ with $\operatorname{Domain}(D)\subseteq \ell^{\infty}(\mathcal{F}_1)$ is Hadamard differentiable at $\mathbb{P}$ (relative to $\ell^{\infty}(\mathcal{F}_1)$) if there exists a continuous linear functional $\dot{D}:\ell^\infty(\mathcal{F}_1)\to\mathbb{R}$ such that, for all $\{\mathbb{P}+th_t : t\in [-1,1],h_t\in \ell^\infty(\mathcal{F}_1)\}\subseteq  \operatorname{Domain}(D)$ with $\|h_t - h\|_{\ell^\infty(\mathcal{F}_1)} =o(1)$ for some $h$,

$$
\lim_{t\to 0} \frac{D(\mathbb{P} + th_t) - D(\mathbb{P})}{t} = \dot{D}(h).
$$
The functional $D$ is second order Hadamard differentiable at $\mathbb{P}$ if there exists a continuous quadratic form $\ddot{D}:\ell^\infty(\mathcal{F}_1)\to\mathbb{R}$ such that
$$
\lim_{t\to 0} \frac{D(\mathbb{P} + th_t) - D(\mathbb{P}) - \dot{D}(th_t)}{t^2} = \ddot{D}(h).
$$
 The Hadamard operator of $D$ is defined as the bounded linear operator $G$ such that $\ddot{D}(h) = \int Gh\, dh$, which exists by Theorem~2.7 in \citep[][pg. 323]{kato2013perturbation}. 
 We can also express $G$ via the convolution with its corresponding kernel function, $g: \mathcal{X}^2\rightarrow\mathbb{R}$, $Gh = \int g(\cdot ,y)\, dh(y)$, $g(\cdot, y) \in \mathcal{F}$. See Appendix \ref{sec: FTE} for further detail on why Hadamard differentiability is natural in the compression setting.

As the dual RKHS topology $\ell^\infty(\mathcal{F}_1)$ will be of key importance in our study, we introduce the corresponding kernel function, $k$, and the integral operator $K:\ell^\infty(\mathcal{F}_1) \to \mathcal{F}$ defined as $Kh:= \int k(\cdot, y) dh(y)$. We denote the $L^2(\mathbb{P})$ Mercer spectrum  and eigenfunctions of $K$ by $\{\sigma_i\}_{i=1}^\infty$ and $\{\psi_i\}_{i=1}^\infty$, respectively. Depending on which is clearer in context, we will denote RKHS norms by subscripting by either the class of functions or its kernel integral operator---for example, we will denote the norm on $\mathcal{F}$ by both $\|\cdot\|_\mathcal{F}$ by $\|\cdot\|_K$ at different points. The $\ell^\infty(\mathcal{F}_1)$ norm can be more familiarly expressed as a quadratic form
\[
\|\mu\|_{\ell^\infty(\mathcal{F}_1)} = \left(\int k(x,y) d\mu(x) d\mu(y)\right)^{1/2} =: Q_K(\mu).
\]
In particular, $\|\mu - \nu\|_{\ell^\infty(\mathcal{F}_1)} = \operatorname{MMD}_K(\mu, \nu)$, with $\operatorname{MMD}_{K}$ denoting the maximum mean discrepancy relative to the RKHS $\mathcal{F}$.

Much of our study will center around $\mathbb{P}_n$ and $P_m$ supported on it.
In this setting, we can typically reduce the integral operator $K$ to the finite dimensional kernel Gram matrix $K_n := [k(x_i, x_j)]_{i,j=1}^n$, with ordered spectrum $\{n\sigma_i^n\}_{i=1}^n$, $\sigma_i^n$ being the eigenvalues of the normalized integral operator $K_n/n$.
To ensure that the spectral decay of the Gram matrix is comparable to that of the Mercer decomposition of $k$, we assume that $\|\psi_i\|_\infty:= \sup_{x\in \mathcal{X}} |\psi_i(x)| \leq M$ for all $i$. This is a classical nonparameteric simplifying assumption known to hold for Matern (Sobolev) and, with small adjustment, Gaussian kernels \citep{yang2020function}. 

When studying operators, $A: F\to G$, we will denote $\|A\|_{F\to G} = \sup_{\|f\|_F \leq 1} \|Af\|_G$, and when $F=G$ we shorten this to $\|A\|_F$. In technical arguments, we may often express a scalar constant on each bound by $C$ that may represent different values in different arguments. 
We will make use of multiplication operators $\mathfrak{M}_f : \ell^\infty(\mathcal{F}_1) \to \ell^\infty(\mathcal{F}_1)$, $\mathfrak{M}_f(\mu) = f \mu$, that is, $ \mathfrak{M}_f(\mu)[g] = \int fg\, d\mu.$ For linear operators (embeddings) $K_1, K_2 : \ell^\infty(\mathcal{F}_1) \to \mathcal{F}$, we say that $K_1 \succeq K_2$ if $\int (K_1 - K_2)\mu\, d\mu \geq 0$ for all $\mu \in \ell^\infty(\mathcal{F}_1).$

\subsection{Main results}
The main idea of our method is that compression with respect to a carefully chosen MMD leads to compression with respect to $D$. This leads to the CO2 meta-algorithm in Algorithm \ref{main_alg}.

\begin{algorithm}[h!]
    \caption{CO2 meta-algorithm}
    \label{main_alg}
    \begin{algorithmic}[1] 
        \Require $D$, $\operatorname{KernelSelection}$, $\operatorname{KernelCompress}$
        \Ensure $P_m$
        \State $K,\, \text{args} \gets$ $\operatorname{KernelSelection}(D)$
        \State $P_m \gets \operatorname{KernelCompress}(K, \text{args})$
    \end{algorithmic}
\end{algorithm}

The KernelSelection algorithm associates to the divergence $D$ a kernel that guarantees rapid compression, and the KernelCompress method performs MMD compression. In the remainder of this section we give an initial indication of how these two tasks can be achieved, ultimately leading to a more concrete CO2 algorithm in Algorithm \ref{CO2}.

\begin{lemma}[Quadratic domination]\label{Functional Delta Method}
    Let $D$ be second order Hadamard differentiable on $\ell^\infty(\mathcal{F}_1)$, where $\mathcal{F}_1$ is the unit ball in an RKHS with characteristic kernel $k$. Then, $\operatorname{MMD}_K(\mathbb{P}_n, P_{m}) = o_p(n^{-1/2})$ implies
    $D(P_{m}) = O_p(n^{-1})$ and
    \[
    D(P_m) = \ddot{D}(P_m - \mathbb{P}) + o_p(n^{-1})  = D(\mathbb{P}_n) + o_p(n^{-1}).
    \]
\end{lemma}

Our proof of the above uses that $D(\mathbb{P}) = 0$ and $\dot{D} \equiv 0$ for a divergence about the null. Combining this with the fact that $P_m=\mathbb{P} + tH_t$ when $H_t:=[P_m - \mathbb{P}]/t$ yields
\begin{align*}
    D(P_m) = t^2\left[\frac{D(\mathbb{P} + tH_t) - D(\mathbb{P}) - t\dot{D}(H_t)}{t^2}\right].
\end{align*}
Taking $t=n^{-1/2}$, using that $\operatorname{MMD}_K(\mathbb{P}_n, P_{m}):=\|\mathbb{P}_n- P_{m}\|_{\ell^\infty(\mathcal{F}_1)}$, and applying the second order functional delta method yields Lemma \ref{Functional Delta Method}.

The preceding result shows that, to control compression error with respect to $D$, it suffices to compress with respect to a corresponding MMD. 
To this end, we modify the method of \citep{hayakawa2023samplingbased} to provide improved computational and theoretical guarantees in terms of the RKHS spectral decay. 
Concretely, we compute $P_m$ such that the moments of certain RKHS eigenfunctions are equal to their moments with respects to $\mathbb{P}_n$; this is the recombination algorithm \citep{dd3000ff-da87-33f3-8345-ae7d639d8010}. We estimate these eigenfunctions efficiently via a Nystr\"{o}m approximation \citep{tropp2017fixedrank}, which achieves the same guarantees (up to a fixed scalar) as using the full singular value decomposition.

Define the spectral tail sum $T(i) := \sup_{x_1,\dots,x_n \in \mathcal{X}} \sum_{j=i+1}^\infty \sigma_j^n$, where in a slight abuse of notation, $\sigma_j^n$ are the singular values of the normalized Gram matrix at the deterministic values $x_1,\dots,x_n$.

\begin{lemma}[Fast MMD compression]\label{lem: MMD compress}
    Fix $\theta\in \mathbb{N}$ such that $\theta\ge 2$. With $O(n^2m + \theta^3 m^3 + m^3 \log(n/m))$ time complexity, a distribution $P_m$ can be randomly constructed such that, 
    \[
    \mathbb{E}[\operatorname{MMD}_K^2(\mathbb{P}_n, P_{m})]  \leq 4\left(1 + \frac{m}{m(\theta - 1) - 1}\right) T(m).
    \]
\end{lemma}
The oversampling parameter $\theta$ increases the quality of the approximation at the expense of additional computational overhead. 

A similar approach appears in a recent paper \citep[Algorithm 4]{li2024debiaseddistributioncompression}. The primary difference in our algorithms is the choice of Nystr\"{o}m approximation, where they used the randomly pivoted Cholesky \citep{chen2025randomly} while we used one based on random projections \citep{tropp2017fixedrank}. Comparing these two approximations, the method of \cite{tropp2017fixedrank} yields a sharper approximation at the cost of increased memory and time complexity, respectively at $O(n^2)$ and $O(n^2m)$. 
This discrepancy leads to improved algorithmic efficiency in \cite{li2024debiaseddistributioncompression}, with $O(nm)$ memory requirements and a Nystr\"{o}m implementation with reduced a $O(nm^2)$  time complexity, plus the time it takes to evaluate $mn$ entries of the Gram matrix. 
This makes this procedure particularly effective for typical MMD coreset selection problems, where kernel matrices have closed forms and entries can be evaluated without querying the full dataset.
This is not generally true in our setting, as typically the kernels of interest will be data dependent, as discussed in Section \ref{sec: EKS}. For example, for the Sinkhorn divergence, we demonstrate in Section \ref{sec: Sinkhorn} that evaluating $mn$ Gram matrix entries requires solving an entropic optimal transport problem and inverting an $n\times n$ matrix.
This dominates the computational complexity of randomly pivoted Cholesky, yielding superquadratic complexity. Thus, in our settings of interest, the sharper approximation given by a random projection based Nystr\"{o}m approximation comes with little additional computational overhead. For additional state-of-the-art MMD compression algorithms, see \cite{dwivedi2025generalizedkernelthinning} and \cite{carrell2025lowrankthinning}.

In light of Lemma \ref{Functional Delta Method}, it suffices to choose $m$ large enough so that the approximation error $T(m)$ is negligible compared to the empirical error. This leads to our main result.

\begin{theorem}[Main result]\label{theo: functional compression}
    Let $D$ be second order Hadamard differentiable at $\mathbb{P}$ relative to $\ell^\infty(\mathcal{F}_1)$. If $T(m) = o(1/n)$, then $P_m$ can be constructed in $O(n^2m + m^3 \log(n/m))$ time with $D(P_m) = D(\mathbb{P}_n) + o_p(n^{-1})$.
\end{theorem}
\begin{proof}
    This is an immediate consequence of Lemmas \ref{Functional Delta Method} and \ref{lem: MMD compress}.
\end{proof}
Explicitly, $P_m$ can be constructed via Algorithm \ref{CO2}, with $K$ used in place of $G$. 
In Section \ref{sec: Sinkhorn}, we apply Theorem~\ref{theo: functional compression} to provide compression guarantees for the Sinkhorn divergence $S$.

For the following function classes $\mathcal{F}$, we show in Appendix \ref{spectralDecay} that $T(m) = o(1/n)$ if $m$ meets the given sampling rates:
    \begin{enumerate}
        \item Sobolev space of order $\alpha$ \citep{bach2015equivalence}: $m = \omega(n^{1/(1-\beta)})$, since $\sup_n \sigma_i^n = O(i^{-\beta})$ a.s. for $\beta = (2\alpha/d - 1)/2$; 
        \item Gaussian RKHS \citep{ma2017diving}: $m = \omega(\log^d n )$, since $\sup_n \sigma_i^n= O(e^{-\beta i^{1/d}})$ a.s. for some $\beta>0$.
    \end{enumerate}
    Our approach to bound Gram matrix spectral decay bears similarities to techniques seen in \cite{altschuler2019massivelyscalablesinkhorndistances} and \cite{ yang2020function}. While we \textit{a priori} assume decay rates for the population kernel, recent techniques developed in \cite{li2024debiaseddistributioncompression} can be applied in more general circumstances. They relate properties of the kernel function to spectral decay as well as the quality of resulting coresets.

\subsection{Selecting $m$}

Our theoretical results match with earlier works in MMD compression and coreset selection, which aim to select $m$ so that $D(\mathbb{P}_n, P_m)$ is on the order of $D(\mathbb{P}_n, \mathbb{P})$. Practically, however, rather than selection based on these asymptotic rates---like  $m = \Omega(\sqrt{n})$ or $m = \omega(\log^d n)$---we favor a coreset selection size procedure based off direct computation of the error $D(\mathbb{P}_n, P_m)$, or for an efficient implementation, based off $\operatorname{MMD}_K^2(\mathbb{P}_n, P_m)$. We illustrate empirically in Section \ref{sec:quad approx} that in settings where the latter of these is sufficiently small, its residual compared to the divergence is similarly negligible. Hence, we recommend selecting the minimal compression size $m$ such that $Q_K(\mathbb{P}_n - P_m)$---or $D(\mathbb{P}_n, P_m)$, if not computationally prohibitive---is below a specified tolerance $\tau$. 

\begin{algorithm}[tb]
    \caption{$m$ selection: recombination sweep}\label{Recomb sweep}
    \begin{algorithmic}
        \State $U \gets m_{\max}-1$ leading Nystr\"{o}m eigenvectors

        \State $K \gets$ kernel matrix

        \State $V \gets$ orthogonal complement of $U \oplus \mathbf{1}$
        
        \State $\tau \gets$ specified tolerance

        \State err $\gets 0$, \% initialized to 0 
        \State count $\gets 0$, \% initialized to 0 
        \State $\mu \gets$ initial empirical distribution
        \While{err $\leq \tau$ \textbf{and} $\text{count}\leq n -m_{\max}+1$}
            \State $\mu_{\text{prev}} \gets \mu$
            \State count$\gets$ count + 1
            \State $\alpha \gets \min d\mu/V[:, \text{count}]$ \% choose $\alpha$ so that next iterate has one less non-zero entry
            \State $\mu \gets \mu - \alpha V[:, \text{count}]$
            \State $\text{err} \gets (\mu - \mathbb{P}_n)^T K (\mu - \mathbb{P}_n)$
        \EndWhile\\
        \Return $\mu_{\text{prev}}$
  \end{algorithmic}
\end{algorithm}

There are several ways to determine this minimal $m$. Binary search is a simple strategy. Algorithm \ref{Recomb sweep} provides a more computationally efficient approach when there is an upper bound $m_{\max}$ on the number of samples the user is willing to use in their coreset. That algorithm selects $m := \min \{m': m'< m_{\max}, \ddot{D}(\mathbb{P}_n - P_{m'}) \leq \tau\}$ by performing a single sweep of the recombination algorithm, taking $m=m_{\max}$ if the minimum is over an empty set.

In light of Lemma~\ref{Functional Delta Method}, the threshold $\tau$ should be $o(n^{-1})$. While ad hoc approaches---such as letting $\tau=n^{-2}$---satisfy this asymptotic requirement, there is no guarantee they will perform well in practice. The proof of Lemma~\ref{Functional Delta Method} suggests a more principled approach for choosing $\tau$. In particular, for $R_{m,n}:=\operatorname{MMD}_G(P_m,\mathbb{P}_n)/\operatorname{MMD}_G(\mathbb{P}_n,\mathbb{P})$, that proof concludes with a bound of the form
\begin{align*}
    |D(P_m)-D(\mathbb{P}_n)|&\le (R_{m,n}^2 + 2R_{m,n})\operatorname{MMD}_G^2(\mathbb{P}_n,\mathbb{P}) + o_p(1/n).
\end{align*}
The right-hand side is $o_p(1/n)$ whenever $R_{m,n}=o_p(1)$; this follows from the fact that $n \operatorname{MMD}_G^2(\mathbb{P}_n,\mathbb{P}) = O(n\operatorname{MMD}_K^2(\mathbb{P}_n,\mathbb{P}))$ and $n\operatorname{MMD}_K^2(\mathbb{P}_n,\mathbb{P})\rightsquigarrow \sum_{i=1}^\infty \sigma_i \chi_i^2(1)$ for $\chi_i^2(1)$ 
independent Chi-squared random variables with 1 degree of freedom and $\sigma_i$ the eigenvalues of $K$, repeated according to their multiplicity \citep[][Theorem 6.32]{grenander2008probabilities}. The weak convergence of $\operatorname{MMD}_K^2(\mathbb{P}_n,\mathbb{P})$ also shows that $R_{m,n}$ can be made small by choosing the upper bound $\tau$ on $\operatorname{MMD}_K^2(P_m,\mathbb{P}_n)$ to be equal to $q_\beta/n$, where $q_\beta$ is the $\beta$ quantile of  $\sum_{i=1}^\infty \sigma_i \chi_i^2(1)$ for small $\beta>0$. The condition $\tau=o(1/n)$ will then be satisfied whenever $\beta=o(1)$, as is the case, for example, when $\beta=1/\log n$. For practical purposes, $q_\beta$ can be approximated using the eigenvalues of the Nystr\"{o}m approximation to the kernel Gram matrix and random samples of the $\chi_i^2(1)$ variables.

\section{Efficient Kernel Selection}\label{sec: EKS}

\subsection{Motivation}
In the previous section, we provided a means to compress with respect to second order Hadamard differentiable functionals. However, a close inspection of Theorem \ref{theo: functional compression} reveals an unsettling ambguity. There is no unique norm $\ell^\infty(\mathcal{F}_1)$ for which one can establish second order Hadamard differentiability! As a simple example, if we can verify that $D$ is Hadamard differentiable in $\ell^\infty(\mathcal{F}_1)$, then for any $K'$ such that $K' \succeq K$, $D$ is also Hadamard differentiable in $\ell^\infty(\mathcal{F}'_1).$ In light of Lemma \ref{lem: MMD compress}, this is not a serious issue, as our guarantees encourage us to select $K$ to have maximally rapid spectral decay. As shown in Lemma \ref{Hadamard Optimality}, the natural choice is the Hadamard operator $G$.

\begin{lemma}[Hadamard optimality]\label{Hadamard Optimality}
    If $k$ is a characteristic kernel and $D$ is second order Hadamard differentiable on $\ell^\infty(\mathcal{F}_1)$, then there exists a constant $C>0$ such that $CK \succeq G$.
\end{lemma}
Thus, optimal guarantees for MMD compression as in Lemma \ref{lem: MMD compress} are achieved if $D$ is second order Hadamard differentiable relative to $G$ itself. 
Unfortunately, this still does not make the selection of a kernel for MMD compression unique, as there are many kernels equivalent to $G$. In Section \ref{sec:Had}, we provide additional arguments supporting the use of $G$ in particular.

\subsection{The Hadamard Operator}\label{sec:Had}

\subsubsection{The Remainder}

By compressing with respect to an appropriate MMD, Theorem \ref{theo: functional compression} provides a means to construct $P_m$ such that $D(P_m) = \ddot{D}(P_m - \mathbb{P}) + o_p(1/n)$. We can further express
\begin{align*}
    \ddot{D}(P_m - \mathbb{P})  &= \ddot{D}([P_m - \mathbb{P}_n] + [\mathbb{P}_n - \mathbb{P}])\\
    &=: \int G([P_m - \mathbb{P}_n] + [\mathbb{P}_n - \mathbb{P}]) d([P_m - \mathbb{P}_n] + [\mathbb{P}_n - \mathbb{P}])\\
    &\leq \left[\ddot{D}(\mathbb{P}_n - \mathbb{P})^{1/2} + \ddot{D}(P_m - \mathbb{P}_n)^{1/2}\right]^2,
\end{align*}
 via an application of Cauchy-Schwarz. This decomposition puts a clear emphasis on reducing the empirical error $\ddot{D}(P_m - \mathbb{P}_n)$, suggesting optimality of the Hadamard operator $G$. However, recall that we typically target $\ddot{D}(P_m - \mathbb{P}_n) = o_p(1/n)$, the same order as the remainder $R(P_m) = D(\mathbb{P}, P_m) - \ddot{D}(\mathbb{P}, P_m)$. Hence, it may be that selection of a kernel to more sharply reduce $\ddot{D}(P_m - \mathbb{P}_n)$ will inflate the remainder $R(P_m)$.

 To avoid this issue, it is desirable for the remainder to be negligible compared to $\ddot{D}(P_m - \mathbb{P}_n)$. For this, it suffices that $D$ be third order Hadamard differentiable, which implies that the remainder is $O_p(n^{-3/2})$. Thus, when we desire $\ddot{D}(P_m - \mathbb{P}_n) = \omega(n^{-3/2})$, that is, where one seeks a minimal coreset size $m$ such that lossless compression is achieved, the selection of $G$ is justified.

\subsubsection{Empirical Estimation}

The Hadamard operator $G$ depends on $\mathbb{P}$ in many settings of interest, and so must be estimated. 
In this section, we provide a criteria for $G_n$, the Hadamard operator at the empirical distribution, to adequately approximate the population operator, in the sense that
$|\ddot{D}_n(\mathbb{P}_n - P_m) - \ddot{D}(\mathbb{P}_n - P_m)| = o_p(n^{-1})$, a quantity that is negligible compared to the target error. In a later illustration, we will use this result to verify a residual that is
$O_p(n^{-3/2})$, thus on the same order as the remainder. 
The key idea of our argument is that both $\ddot{D}_n \to \ddot{D}$, as well as $\ddot{D}(\mathbb{P}_n - P_m)\to 0$, which results in the difference of quadratic forms being a higher-order residual. We introduce the operators $T, T_n$
such that $\|T \mu\|_{\ell^\infty(\mathcal{F}_1)}^2 = \ddot{D}(\mu)$ and $\|T_n \mu\|_{\ell^\infty(\mathcal{F}_1)}^2 = \ddot{D}_n(\mu)$. For general quadratic forms we maintain this notation for the root of the generating operators.

\begin{lemma}[Kernel convergence]\label{lem: emp kernel}
    Let $D$ be second order Hadamard differentiable at $\mathbb{P}, \mathbb{P}_n$ for all $n$ relative to a fixed topology $\ell^\infty(\mathcal{F}_1)$. If $\|T_n - T\| = O_p(r_n)$, $r_n \to 0$, then
    \[
    |\ddot{D}(\mu)- \ddot{D}_n(\mu)| = O_p(r_n) \|\mu\|_{\ell^\infty(\mathcal{F}_1)}^2.
    \]
\end{lemma}

\begin{proof}
    It immediately follows that
    \[
   \|(T_n-T)(\mu)\|_{\ell^\infty(\mathcal{F}_1)} \leq \|T_n - T\| \| \mu\|_{\ell^\infty(\mathcal{F}_1)}= O_p(r_n) \| \mu\|_{\ell^\infty(\mathcal{F}_1)}.
    \]

    Now, by the reverse triangle inequality
    \begin{align*}
     |\|T_n(\mu)\|_{\ell^\infty(\mathcal{F}_1)} - \|T(\mu)\|_{\ell^\infty(\mathcal{F}_1)}| \leq \|(T_n - T)(\mu)\|_{\ell^\infty(\mathcal{F}_1)} = O_p(r_n) \| \mu\|_{\ell^\infty(\mathcal{F}_1)},
     \end{align*}
     and so
     \begin{align*}
         |\ddot{D}(\mu)- \ddot{D}_n(\mu)|&= |\|T_n(\mu)\|_{\ell^\infty(\mathcal{F}_1)}^2 - \|T(\mu)\|_{\ell^\infty(\mathcal{F}_1)}^2|\\
         &= |\|T_n(\mu)\|_{\ell^\infty(\mathcal{F}_1)} - \|T(\mu)\|_{\ell^\infty(\mathcal{F}_1)}|\,|\|T_n(\mu)\|_{\ell^\infty(\mathcal{F}_1)}+ \|T(\mu)\|_{\ell^\infty(\mathcal{F}_1)}|\\
         &= O_p(r_n) \| \mu\|_{\ell^\infty(\mathcal{F}_1)}^2.
     \end{align*}
    Note the third equality follows as $\|T_n\|_{\ell^\infty(\mathcal{F}_1)}$ is uniformly bounded, as $T_n \to T$.
\end{proof}

\begin{corollary}
    Let $D$ be second order Hadamard differentiable at $\mathbb{P}, \mathbb{P}_n$ for all $n$ relative to a fixed topology $\ell^\infty(\mathcal{F}_1)$. If $\|T_n - T\| = O_p(r_n)$, $r_n \to 0$, and $\|\mathbb{P}_n - P_m\|_{\ell^\infty(\mathcal{F}_1)} = O_p(n^{-1/2})$, then
    \[
    |\ddot{D}(\mathbb{P} - P_m)- \ddot{D}_n(\mathbb{P} - P_m)| = O_p(r_n n^{-1}).
    \]
\end{corollary}
\begin{proof}
    By the continuous mapping theorem \cite[Theorem 2.3]{van2000asymptotic} and the fact that $\mathcal{F}_1$ is Donsker \ref{lem: rkhs donsker}, we have $\|\sqrt{n}(\mathbb{P}_n - \mathbb{P})\|_{\ell^\infty(\mathcal{F}_1)} = O_p(1)$, hence by assumption we also have $\|\sqrt{n}(P_m - \mathbb{P})\|_{\ell^\infty(\mathcal{F}_1)} = O_p(1)$, from which the result follows immediately in combination with Lemma \ref{lem: emp kernel}.
\end{proof}

\subsection{CO2-Recombination Algorithm}\label{sec:CO2}

Combining the considerations of the past few sections, we present the CO2-recombination algorithm as a special case of Algorithm \ref{main_alg}. We will simply refer to this as CO2 when it will not cause confusion. See Appendix \ref{app: Algs} for implementations of the recombination and Nystr\"{o}m algorithms.

\begin{algorithm}[H]
    \caption{CO2-recombination}
    \label{CO2}
    \begin{algorithmic}[1]
        \State $G \gets$ The Hadamard operator associated to $G$ at $\mathbb{P}_n$\tikzmark{startKS}
        \State $\theta\gets$ Oversampling parameter\tikzmark{midKS}
        \State $G_{\text{nys}}, U, S \gets \operatorname{\text{Nystr\"{o}m}}(G, \theta m)$\tikzmark{endKS}
        \State $P_m \gets \operatorname{Recombination}\big(U[:, :m-1], [g(X_i, X_i)]_{i=1}^n], S[:m-1]\big)$\ \ \ \ \  \text{\}\ \ KernelCompress}
    \end{algorithmic}
\end{algorithm}

\begin{tikzpicture}[overlay, remember picture]
\draw [thick, decoration={brace, raise=-5pt}, decorate] 
    ($(startKS) + (6.6em, 0.7ex)$) -- 
    ($(endKS) + (16.5em, -0.7ex)$) 
    node [midway, right=4pt] {KernelSelection};
\end{tikzpicture}

\section{Illustration: Sinkhorn Coresets}

\subsection{Overview}
As an illustration of our general framework, we consider compression with respect to the Sinkhorn divergence \citep{ramdas2015wasserstein} 
\[
S(\mu,\nu) :=  \EOT(\mu,\nu) - \frac{1}{2}(\EOT(\mu,\mu) + \EOT(\nu,\nu)),
\]
with the regularization parameter $\varepsilon>0$ suppressed in the notation. Here, $\EOT$ denotes the entropically regularized Wasserstein-2 distance
\begin{align}\label{eq: EOT}
\EOT(\mu,\nu) := \min_{\pi \in \Pi(\mu,\nu)} \int \|x-y\|^2 d\pi(x,y) + \varepsilon \operatorname{KL}(\pi\parallel \mu\otimes \nu), 
\end{align}
for $\Pi(\mu,\nu)$ the set of joint distributions with marginals $\mu,\nu$.

In this section, we verify that lossless Sinkhorn compression requires $m$ to grow only slightly faster than $\log^d n$.

\begin{theorem}[Sinkhorn compression] \label{thm: Sink compression}
    For $m = \omega(\log^d n)$, $P_m$ can be constructed in $O(n^2m + m^3)$ time such that $S(\mathbb{P}, P_m) = S(\mathbb{P}, \mathbb{P}_n)+o_p(1/n)$.
\end{theorem}
Our proof of the above consists of two parts. In Section \ref{sec: EOT potentials}, we verify novel differentiability results for the entropic optimal transport potentials. In Section \ref{sec: Sinkhorn}, we leverage these to show $S$ is second order Hadamard differentiable with respect to the Gaussian kernel RKHS. The geometric spectral decay of this kernel ensures $m=\omega(\log^d n)$ suffices for lossless compression.

\subsection{Entropic optimal transport potentials}\label{sec: EOT potentials}

Alternatively to expression \eqref{eq: EOT}, the $\EOT$ functional can be expressed in terms of the entropic optimal transport potentials $\phi_{\mu,\nu}, \psi_{\mu,\nu}:\mathcal{X}\rightarrow\mathbb{R}$:
\begin{align}\label{def: EOT first}
\EOT(\mu,\nu) = \int \phi_{\mu,\nu}\, d\mu + \int \psi_{\mu,\nu}\, d\nu.
 \end{align}
 Thus, to study the differentiability of $S$, we first provide corresponding results for $\phi_{\mu,\nu}, \psi_{\mu,\nu}$.

 Equation \eqref{def: EOT first} can equivalently can be expressed as the following decomposition of the entropic optimal transport plan $\pi_{\mu,\nu}^\varepsilon$, the joint law that achieves the minimum of \eqref{eq: EOT},
\begin{gather}\label{eq: EOT plan}
    \xi_{\mu,\nu}^\varepsilon(x,y) := \frac{d\pi_{\mu,\nu}^\varepsilon}{d(\mu \otimes \nu)}(x,y) = \exp\left(\frac{\phi_{\mu,\nu}(x) + \psi_{\mu,\nu}(y) - \|x-y\|^2}{\varepsilon}\right).
\end{gather}
This decomposition originates from \cite{csiszar1975divergence} and \cite{ruschendorf1993note}, and
the functions $\phi_{\mu,\nu}$ and $\psi_{\mu,\nu}$ are called the entropic optimal transport potentials.
Recently, asymptotic distributions of plug-in estimates of the potentials and related quantities have been verified via the implicit function theorem \citep{goldfeld2022limit,gonzalezsanz2022weak}. Specifically, evaluating the marginal constraints in \ref{eq: EOT} leads to the Schr\"odinger system of equations,
\begin{align*}
    \phi_{\mu,\nu}(x) &= -\varepsilon \log \int \exp\left(\frac{\psi_{\mu,\nu}(y) - \|x-y\|^2}{\varepsilon}\right) d\nu(y),\\
    \psi_{\mu,\nu}(x) &= -\varepsilon \log \int \exp\left(\frac{\phi_{\mu,\nu}(y) - \|x-y\|^2}{\varepsilon}\right) d\mu(y). 
\end{align*}
In \cite{goldfeld2022limit} they cast this problem through the lens of Hadamard differentiability, and \cite{gonzalezsanz2022weak} provide a direct, analytic expression for the  limiting distribtion. 
These analyses verify that the potentials satisfy regularity with respect to Sobolev RKHSs, which exhibit polynomial spectral decay \citep{bach2015equivalence}. 
We draw aspects of our arguments from both of these works, however we establish Hadamard differentiability relative a finer space of functions, yielding faster decay and improved compression guarantees via Theorem \ref{theo: functional compression}.
 Our innovation is to lift the analysis to the functions $(u_{\mu,\nu}, v_{\mu,\nu}) := (\exp(-\phi_{\mu,\nu}/\varepsilon), \exp(-\psi_{\mu,\nu}/\varepsilon))$, which belong to the RKHS $\mathcal{H}$ with Gaussian kernel $k(x,y)=\exp(-\|x-y\|^2/\varepsilon)$. 
We establish the Hadamard differentiability of $(\mu,\nu)\mapsto (u_{\mu,\nu}, v_{\mu,\nu})$ by applying an implicit function theorem to
\begin{align*}
    u_{\mu,\nu}(x) &= \int \exp\left(\frac{ - \|x-y\|^2}{\varepsilon}\right)\frac{d\nu(y)}{v_{\mu,\nu}(y)},\quad v_{\mu,\nu}(x) = \int \exp\left(\frac{- \|x-y\|^2}{\varepsilon}\right) \frac{d\mu(y)}{u_{\mu,\nu}(y)}.
\end{align*}
Hadamard differentiability of $(\mu,\nu)\mapsto (\phi_{\mu,\nu}, \psi_{\mu,\nu})$, and therefore also $\EOT$ and $S$, can then be recovered through the chain rule.
 
 Formally stating our result requires introducing notation. Suppressing the dependence on $\varepsilon$ and the pair $(\mu,\nu)$, define operators on $C(\mathcal{X})$, 
 \begin{align*}
    \mathcal{A} f &:=  \int v_{\mu,\nu}^{-1}(x)\exp\left( \frac{- \|\cdot-y\|^2}{\varepsilon}\right)u_{\mu,\nu}^{-1}(y)f(y) d\mu(y) = \int  \xi_{\mu,\nu}(x,y) f(x) d\mu(x),\\
    \mathcal{A}^* f  &:=  \int u_{\mu,\nu}^{-1}(x)\exp\left( \frac{- \|\cdot-y\|^2}{\varepsilon}\right)v_{\mu,\nu}^{-1}(y)f(y) d\nu(y) =  \int \xi_{\mu,\nu}(x,y) f(y) d\nu(y),
\end{align*}
and similarly operators on distributions
 \begin{align*}
    \xi \gamma &=  \int v_{\mu,\nu}^{-1}(x)\exp\left( \frac{- \|\cdot-y\|^2}{\varepsilon}\right)u_{\mu,\nu}^{-1}(y)d\gamma(y) = \int  \xi_{\mu,\nu}(x,y) d\gamma(x),\\
    \xi^* \gamma  &=  \int u_{\mu,\nu}^{-1}(x)\exp\left( \frac{- \|\cdot-y\|^2}{\varepsilon}\right)v_{\mu,\nu}^{-1}(y)d\gamma(y) =  \int \xi_{\mu,\nu}(x,y) d\gamma(y),
\end{align*}
where we use the shorthand $f^{-1} = 1/f$. Denote by $\pi_V$ the projection onto $V \subseteq \mathcal{H}$.

\begin{theorem}[Potential differentiation]\label{thm: EOT Potentials}
    $\eta: \ell^\infty(\mathcal{H}_1)\times \ell^\infty(\mathcal{H}_1) \to C(\mathcal{X})\times C(\mathcal{X})$, $\eta(\mu,\nu) := (\phi_{\mu,\nu},\psi_{\mu,\nu})$ is Hadamard differentiable at $(\mu, \nu)$ with derivative
    \begin{align}
     \dot{\eta}_{\mu,\nu}[\gamma^1,\gamma^2]&=\begin{bmatrix}
     \varepsilon[(I - \mathcal{A}^* \mathcal{A})^{-1}  \mathcal{A}^* \xi \gamma^1 - (I - \mathcal{A}^* \mathcal{A})^{-1} \xi^* \gamma^2]\\
    \varepsilon [(I - \mathcal{A} \mathcal{A}^*)^{-1}  \mathcal{A} \xi^* \gamma^2 - (I - \mathcal{A} \mathcal{A}^*)^{-1} \xi \gamma^1]
    \end{bmatrix} \label{eq:EOT Potentials main text}\\
    &\quad \quad + 
    \varepsilon
    \begin{bmatrix}
        \frac{1}{u} & 0\\
        0 & \frac{1}{v}
    \end{bmatrix} \pi_{\operatorname{span}\{(u, -v)\}^\perp} \begin{bmatrix}
        (I - \pi_{\operatorname{span}\{u\}^\perp}) K\left(\frac{d \gamma^2}{v}\right)\\
        0
    \end{bmatrix} + \varepsilon \rho (\mathbf{1}, -\mathbf{1}), \nonumber
    \end{align}
    with $\rho$ being defined in Theorem \ref{theo:tau had} in Appendix \ref{app:regPotential}.
    In the case where $\mu = \nu$ the second term in equation \ref{eq:EOT Potentials main text} vanishes.
\end{theorem}
Above its proof in Appendix~\ref{app:regPotential}, we describe how this result relates to other results in the literature (Theorem 2.2 in \citealp{gonzalezsanz2022weak}; Theorem 4 in \citealp{harchaoui2022asymptoticsdiscreteschrodingerbridges}).

\subsection{Sinkhorn divergence}\label{sec: Sinkhorn}
  To leverage the above first order differentiability, we use a local approximation from \cite{gonzalezsanz2022weak}, verifying that
\begin{align*}
    \OTL(\mu,\nu) &:=  \iint [\phi_{\nu,\nu}(x)+ \psi_{\nu,\nu}(y)]\, d\mu(x) d\nu(y) \\
    &\quad\quad +\frac{1}{2\varepsilon}\iint ([\phi_{\mu,\nu} - \phi_{\nu,\nu}](x) + [\psi_{\mu,\nu} -\psi_{\nu,\nu}](y))^2\, d\mu(x) d\nu(y)
\end{align*}
agrees with $\operatorname{OT}_\varepsilon(\mu,\nu)$ up to second order. This approximation allows us to differentiate each of the potentials individually, leading to the following quadratic expansion of the Sinkhorn divergence.

\begin{theorem}[Sinkhorn kernel]\label{thm: Sink had approx}
    $D(\nu) := S(\mathbb{P},\nu)$ is second order Hadamard differentiable at $\mathbb{P}$ relative to $\ell^\infty(\mathcal{H}_1)$ with Hadamard operator
    \[
    G := G_{S,\mathbb{P}} := \varepsilon(I - \mathcal{A}^2)^{-1} \xi.
    \]
    Moreover, the Sinkhorn kernel $G_{S,\mathbb{P}}$ is equivalent to the Gaussian kernel.
\end{theorem}
As indicated previously, the Sinkhorn kernel is not computable at $\mathbb{P}$ since the population distribution is unknown. Centering instead at $\mathbb{P}_n$, for $\mu \ll \mathbb{P}_n$ the empirical Hadamard operator can be expressed as
\[
G_n = \varepsilon n(I - {\pi_{\mathbb{P}_n,\mathbb{P}_n}^\varepsilon}^2)^{-1} \pi_{\mathbb{P}_n,\mathbb{P}_n}^\varepsilon,
\]
for $\pi_{\mathbb{P}_n,\mathbb{P}_n}^\varepsilon$ the entropic optimal transport plan from the empirical distribution to itself. Due to the previously established convergence of entropic potentials, we can apply Lemma \ref{lem: emp kernel} to show that the plug-in kernel has negligible error compared to its population version.

\begin{lemma}[Sinkhorn embedding convergence]\label{lem: xi quad comp}
If $\|\mathbb{P} - P_m\|_{\ell^\infty(\mathcal{F}_1)} = O_p(n^{-1/2})$, then
    \[
    |Q_{\xi_n}(\mathbb{P} - P_m) - Q_{\xi}(\mathbb{P} - P_m)| = \|\mathbb{P} - P_m\|_{\ell^\infty(\mathcal{H}_1)}^2 \,O_p(n^{-1/2})  =  O_p(n^{-3/2})
    \]
\end{lemma}
\begin{proof}
    Let $u_n := u_{\mathbb{P}_n, \mathbb{P}_n}, u:= u_{\mathbb{P},\mathbb{P}}$. Observe that $T_n = \mathfrak{M}_{1/u_n}, T = \mathfrak{M}_{1/u}$, hence by Lemma \ref{lem: emp kernel}, it suffices to verify that $\|T_n - T\| = O_p(n^{-1/2})$. Expanding this, we have
    \[
    \sqrt{n}(T_n - T) = \mathfrak{M}_{\sqrt{n}(1/u_n - 1/u)}.
    \]
    By Lemma 3.10.26 in \cite{van1996wellner}, we can express $\sqrt{n}(1/u_n - 1/u) = \dot{u}[\sqrt{n}(\mathbb{P}_n - P)] + R_n$ where $\dot{u}$ is a continuous linear operator with image in $C^k(\mathcal{X})$ and $\|R_n\|_{C^k(\mathcal{X})} = o_p(1)$. Hence, by Corollary \ref{cor: op bound},
    \[
    \sqrt{n}\|T_n - T\| \leq C(\|R_n\|_{C^k(\mathcal{X})} + \sqrt{n}\|\mathbb{P}_n - P\|_{\ell^\infty(\mathcal{F}_1)}) = O_p(1).
    \] 
\end{proof}

The full proof of the theorem below entails studying the operator $\mathcal{A}$. This requires a longer argument, and so we defer it to Appendix \ref{sec: sink emp kernel}. 

\begin{theorem}[Sinkhorn kernel convergence]\label{theo: Sink kernel comp}
    If $\|\mathbb{P} - P_m\|_{\ell^\infty(\mathcal{F}_1} = O_p(n^{-1/2})$, then
    \[
    |Q_{G_n}(\mathbb{P} - P_m) - Q_{G}(\mathbb{P} - P_m)| = \|\mathbb{P} - P_m\|_{\ell^\infty(\mathcal{F}_1)}^2 \,O_p(n^{-1/2})  =  O_p(n^{-3/2})
    \]
\end{theorem}
Computing $G_n$ is quite cumbersome computationally due to the presence of the pseudo-inverse in the expression. We offer two practical means to circumvent this problem. The first is that, in Algorithm \ref{CO2}, we see that the recombination Algorithm requires only knowledge of principal eigenvectors/eigenvalues and the diagonal of the matrix. The eigenvalues and eigenvectors can be recovered from $\pi_{\mathbb{P}_n,\mathbb{P}_n}^\varepsilon$, and, for the diagonal, Hutchinson's diagonal estimator can be used to provide a computationally efficient estimator \citep{dharangutte2022tightanalysishutchinsonsdiagonal}. Even more simply, this problem can be circumvented altogether by focusing on the quadratic form $Q_{\xi_n}$ generated by $n \pi_{\mathbb{P}_n, \mathbb{P}_n}^\varepsilon$. For one, $Q_{G_n}(\mathbb{P}_n - \mathbb{P}_m) = O_p(\varepsilon Q_{\xi_n}(\mathbb{P}_n - \mathbb{P}_m))$ as $\|(I - {\pi_{\mathbb{P}_n,\mathbb{P}_n}^\varepsilon}^2)^{-1}\| = O_p(1)$, but we expect a tighter result. We provide a basic argument in the following which we empirically validate in Section \ref{sec:quad approx}. We leave a complete argument to future research, as the necessary Nystr\"{o}m manipulations are outside the scope of this paper. 

\begin{lemma}[Sinkhorn fast computation]\label{lem:cheap quad}
    For $P_m$ as constructed in Algorithm \ref{CO2}, with the Nystr\"{o}m approximation replaced by the rank $r$ singular value decomposition,
    \[
    |Q_{G_n}(\mathbb{P}_n - P_m) - \varepsilon Q_{\xi_n}(\mathbb{P}_n - P_m)| = Q_{\xi_n}(\mathbb{P}_n - P_m)O_p(\exp(-\gamma n))
    \]
    for some $\gamma > 0$.
\end{lemma}

\begin{proof}
    We can view $\xi_n, G_n$ as finite dimensional linear operators as we only consider distributions supported on $\mathbb{P}_n$. The distribution $\mathbb{P}_n - P_m = \sum_{i=1}^n w_i \delta_{x_i}$ then corresponds to a vector $w = [w_1,\dots, w_n]^T$, and we can express $Q_{\xi_n}(\mathbb{P}_n - P_m) = w^T (n\pi_{\mathbb{P}_n, \mathbb{P}_n}^\varepsilon) w$ and $Q_{G_n}(\mathbb{P}_n - P_m) = w^T \varepsilon n(I - {\pi_{\mathbb{P}_n,\mathbb{P}_n}^\varepsilon}^2)^{-1} \pi_{\mathbb{P}_n,\mathbb{P}_n}^\varepsilon w$.     
    Taking a spectral decomposition, we decompose $w$ over the basis of eigenvectors of $\pi_{\mathbb{P}_n, \mathbb{P}_n}^\varepsilon$, given by $\{\phi_i\}_{i=1}^n$. Thus, we can identify $w = \sum_{i=1}^{n} a_i \phi_i$ and
    \begin{align*}
    Q_{\xi_n}(\mathbb{P}_n - P_m) = \sum_{i=1}^n \lambda_i a_i^2,\hspace{3em}
    Q_{G_n}(\mathbb{P}_n - P_m) = \varepsilon \sum_{i=1}^n \frac{\lambda_i}{1-\lambda_i^2} a_i^2.
     \end{align*}
     By the construction in Algorithm \ref{CO2}, $a_1,\dots, a_m = 0$, leaving us with the tails of these series. Thus,
     \begin{align*}
     &|[Q_{G_n} - \varepsilon Q_{\xi_n}](\mathbb{P}_n - P_m)| = \varepsilon \left|\sum_{i=m+1}^n \frac{\lambda_i^3}{1-\lambda_i^2} a_i^2\right|\leq \varepsilon \frac{\lambda_m^2}{1-\lambda_m^2}Q_{\xi_n}(\mathbb{P}_n - P_m)\\
     &\quad\quad =  Q_{\xi_n}(\mathbb{P}_n - P_m) o_p(\exp(-\gamma n))
     \end{align*}
     as $\lambda_i$ monotonically decreasing to 0, $\lambda_i$ is geometrically decaying by Lemmas \ref{uniform kernel equiv} and \ref{lem:LN_sob} (upon noting that the Sinkhorn scalings are uniformly smooth almost surely by Theorem \ref{thm: EOT Potentials}), and $m = \omega(\log^d n).$ 
\end{proof}

\section{Numerical experiments}

For efficient computation of the Sinkhorn divergence and entropic optimal transport potentials we utilize the geomloss Python package \citep{feydy2019interpolating}. CO2 is performed with respect to $\pi_{\mathbb{P}_n, \mathbb{P}_n}^\varepsilon$ for improved computational efficiency as detailed in Section \ref{sec: Sinkhorn}. In our Nystr\"{o}m approximation, we select a $\theta = 3$ oversampling parameter except in our MNIST experiment in Section \ref{sec: mnist}, where we set $\theta=5$. Code for all simulations can be found at
\href{https://github.com/amkokot/sinkhorn\_coresets}{https://github.com/amkokot/sinkhorn\_coresets}.

\subsection{Visualization}\label{sec:visualization}

We first qualitatively demonstrate that our compression algorithm rapidly recovers key characteristics of input datasets. We consider a Gaussian mixture generative model $\frac{1}{K} \sum_{i=1}^K N(\mu_i, I)$ in two dimensions, which we approximate with $\mathbb{P}_n$, $n=10^4$. We consider a mixture of $K=8$ Gaussians with means placed at $[\pm 3, \pm 3], [0,\pm 6], [\pm 6, 0]$, which we compress to $m=16$ points based on the Sinkhorn divergence with $\varepsilon=0.75$.
As seen in Figure \ref{fig:grid}, these coresets typically recover the primary clusters of the dataset, with samples being reasonably spread across each of the 8 clusters.  

\begin{figure}[h!]
    \centering
    \includegraphics[width=0.8\linewidth]{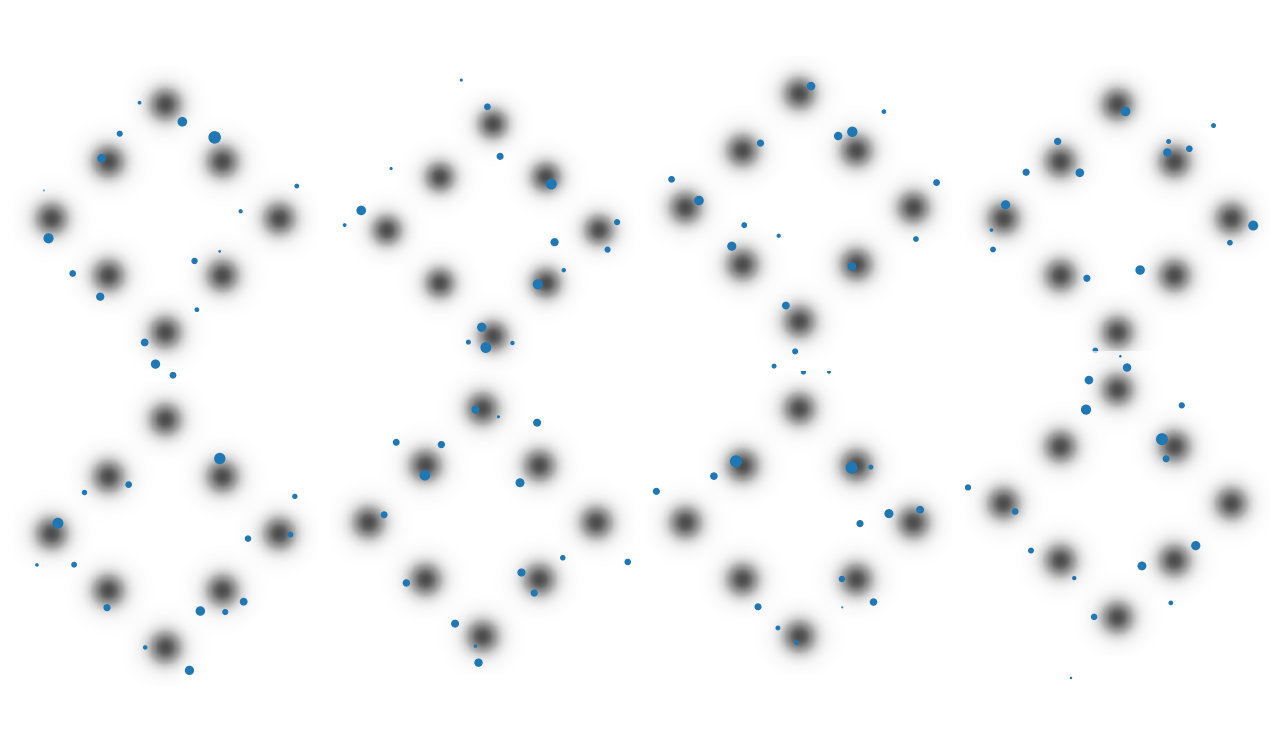}

    \caption{Eight repetitions of the simulation described in Section \ref{sec:visualization}.}
    \label{fig:grid}
\end{figure}

\subsection{Distribution recovery}\label{sec: dist rec}

We now demonstrate how our method quickly reduces the reconstruction error $S_\varepsilon(\mathbb{P}_n, P_m)$. Following Section 7.3 of \citep{dwivedi2023kernel}, we take $N(0, I)$ as our generative distribution in dimensions $d = 2, 5, 10$, and we set the regularization parameter to be $\varepsilon = 2d$ in each setting. We then compare $S_\varepsilon(\mathbb{P}_n, P_m)$ for $P_m$ generated by Sinkhorn compression to $P_m$ generated by random sampling. We perform 40 trials for each configuration of $m,d,n$. In Figure \ref{fig:sink combined} we see a strong dependence on the dimension, and minimal dependence on $n$.

\begin{figure}[tb]
    \centering
    \includegraphics[width=0.99\linewidth]{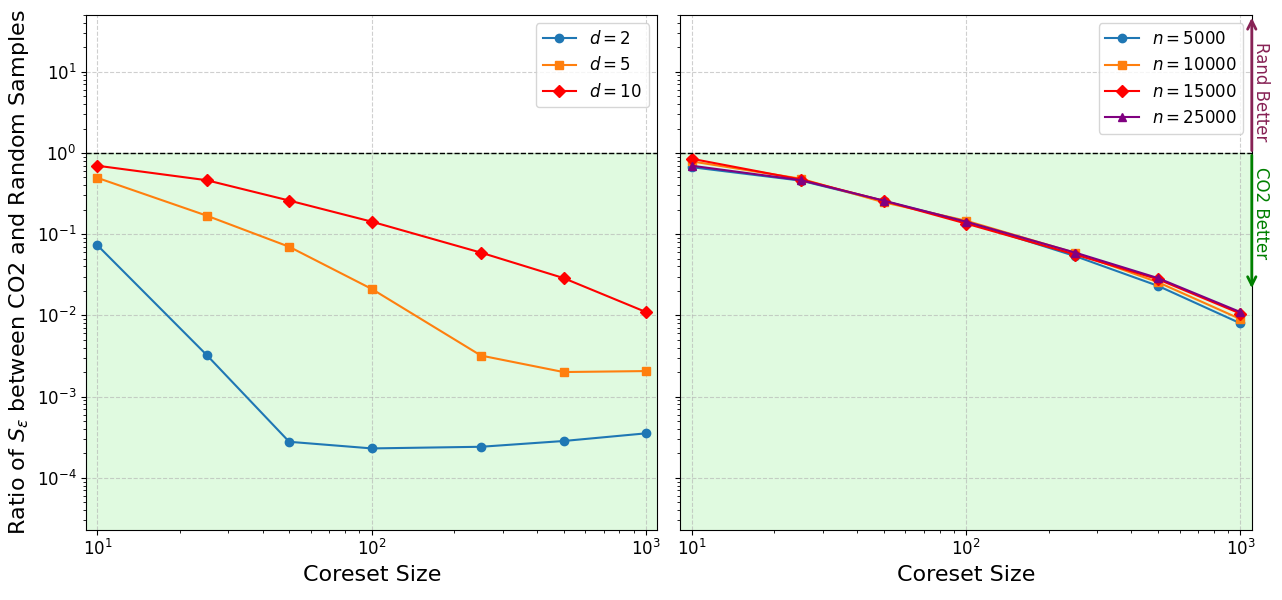}
    \caption{The reconstruction error $S_\varepsilon(\mathbb{P}_n, P_m)$ in various dimensions (left) and dataset sizes $n$ (right). In the first plot the sample size is fixed at $n=2.5 \times 10^4$, for the latter the dimension is fixed at $d=10$.}
    \label{fig:sink combined}
\end{figure}

\subsection{Quadratic approximation}\label{sec:quad approx}

A key element of our theory is that compression with respect to the second order expansion of a divergence is sufficient for compression with respect to the actual divergence of interest. We demonstrate this empirically by computing the relative error
\[
\frac{|S(\mathbb{P}_n, P_m) - \ddot{S}(\mathbb{P}_n, P_m)|}{S(\mathbb{P}_n, P_m)}
\]
for $P_m$ generated via our recombination method, and samples $\mathbb{P}_n$ generated with respect to the generative process in Section \ref{sec: dist rec}, with $d=10$, $n=2.5 \times 10^4$. In fact, we go one step further, as we utilize the cheaper quadratic approximation highlighted in Lemma \ref{lem:cheap quad}. We perform 80 trials for each $m$. As shown in Figure \ref{fig:kern}, the approximation error quickly vanishes, with negligible relative error compared to the actual divergence.
\begin{figure}[h!]
    \centering
        \includegraphics[width=0.8\linewidth]{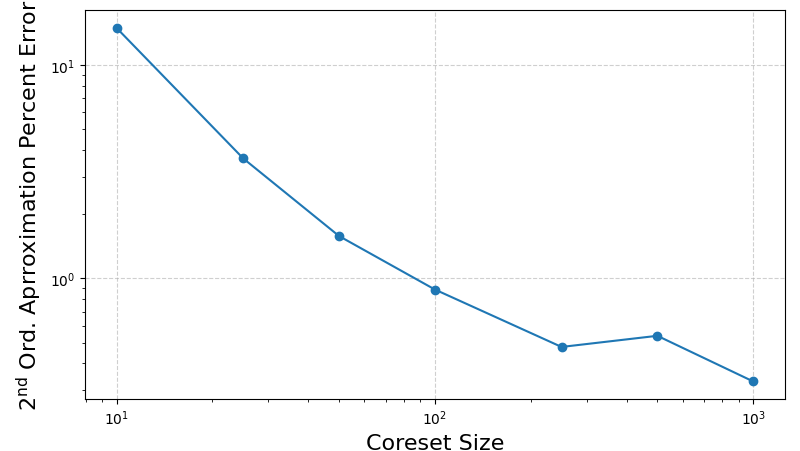}
    \caption{Percent relative error of $\ddot{S}(\mathbb{P}_n, P_m)$ compared to $S(\mathbb{P}_n, P_m)$ with $P_m$ generated via recombination compression.}
    \label{fig:kern}
\end{figure}

\subsection{MMD Compression Algorithms}

We demonstrate how a variety of MMD compression algorithms can be paired with CO2 to provide efficient coresets with respect to the Sinkhorn divergence. Our second order approximation results allow one to reduce the coreset selection problem for smooth divergences such as the Sinkhorn to a corresponding MMD, allowing practitioners to select an algorithm of their preference at this step of our method. This is particularly useful as this is a growing area of research, and we expect the development of algorithms that greatly improve our basic recombination approach, particularly in terms of practical finite sample performance. 

We draw $n=2.5\times 10^4$ samples uniformly from the unit cube of dimension 10, and compute the Sinkhorn divergence with regularization $\varepsilon = 20$. We compare 3 different coreset selection procedures. Two of these algorithms are aimed at directly optimizing the MMD compression error, kernel herding and recombination, and we optimize the weights of the selected samples post selection. Halton sampling is a more generic discrepancy minimization procedure used in quasi-Monte Carlo \citep{jank2005quasi}. We compare these procedures to the baseline random sampling, and perform 80 trials for each $m$ and sampling algorithm.

As seen in Figure \ref{fig:cube}, procedures that directly targeted MMD minimization of the second order term displayed the best performance. The recombination algorithm, our recommended MMD compression procedure, performed near optimally in each setting, with particularly impressive results for smaller targeted compression sizes.

\begin{figure}[h!]
    \centering
    \includegraphics[width=0.8\linewidth]{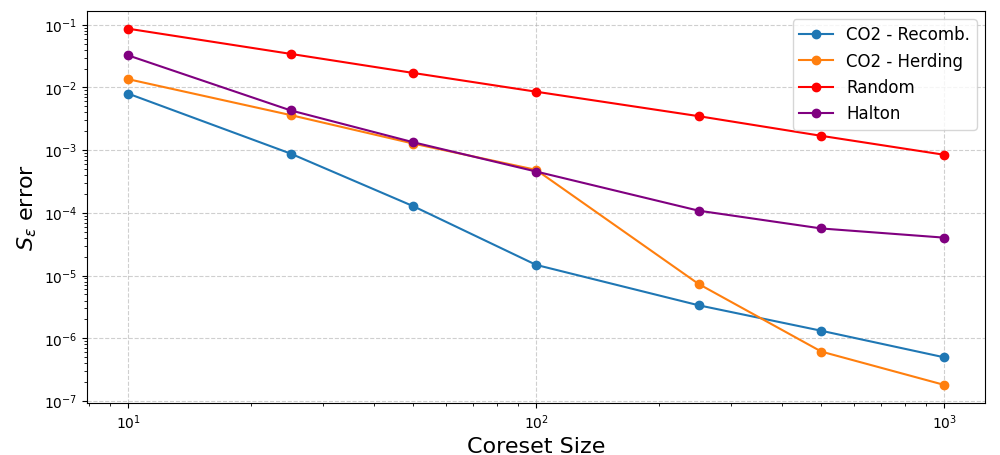}
    \caption{MMD compression algorithms as well as other sampling methods employed at the secondary stage of our Sinkhorn compression method and the resulting compression error $S(\mathbb{P}_n, P_m)$ induced by these algorithms.}
    \label{fig:cube}
\end{figure}

\subsection{Image processing}\label{sec: mnist}

We employ our algorithm on the entire 70,000 image MNIST dataset, where after standardizing the raw features to have mean 0 and unit variance, we select a coreset of size $m=1000$ for the Sinkhorn divergence with regularization $\varepsilon = 1.5 \times 10^4$. This regularization is large, but not atypical in this high-dimensional setting. As a point of comparison, a typical heuristic for kernel bandwidth selection is to use the median squared distance between observations \citep{garreau2018largesampleanalysismedian}, and for our data this is $\approx 1000$, and thus of a comparable magnitude. We perform 500 trials of coreset generation. High dimensional image data is representative of a particularly relevant setting for data compression. By enforcing that our coreset belongs to the observed dataset, there is no ambiguity regarding which of the labels these datapoints correspond to, as compared to a procedure that selects centroids from the ambient space. Our selection thus respects the underlying structure of the data, which we expect to practically be of a much lower intrinsic dimension. 

\begin{figure}[tb]
    \centering
    \includegraphics[width=0.99\linewidth]{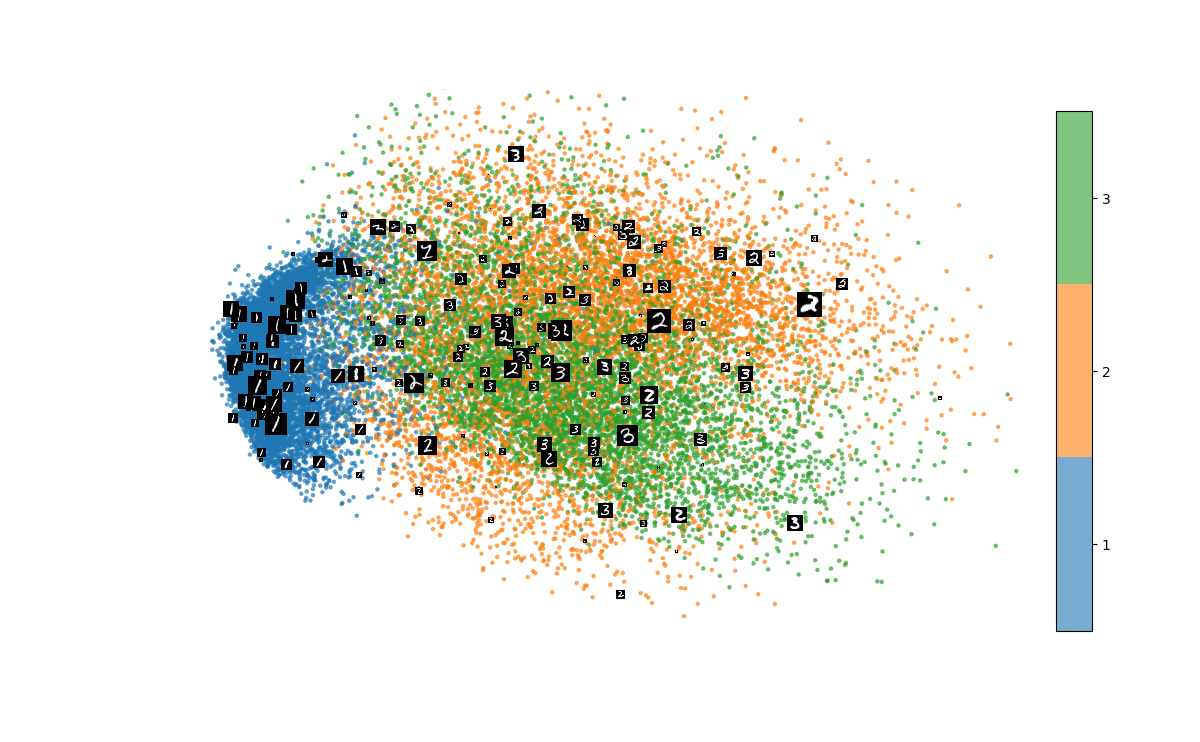}
    \caption{PCA coordinates of MNIST data with a Sinkhorn CO2 coreset overlaid on the sample, with figure sizes proportional to the sample weights. For interpretability, only digits 1, 2, and 3 are displayed here.}
    \label{fig:MNIST2}
\end{figure}

As a visualization of the resulting weighted coreset, in Figure \ref{fig:MNIST2} we plot the first 2 PCA coordinates of the MNIST dataset corresponding to a subset of the labels considered, $\{1, 2, 3\}$, in a scatter plot, and overlay the images corresponding to the selected datapoints with size proportionate to their weights. This figure gives a sense of the geometry of the sample, as we see a tight cluster of 1s, compared to the 2 and 3 digits which have more significant overlap. In Figure \ref{fig:MNIST_tree} we display a treemap of the 1000 selected digits realized by one trial of the CO2 algorithm. This visualization reveals that each digit receives about 1/10 of the total weight in the convexly weighted coreset.

\begin{figure}[tb]
    \centering
    \includegraphics[width=0.99\linewidth]{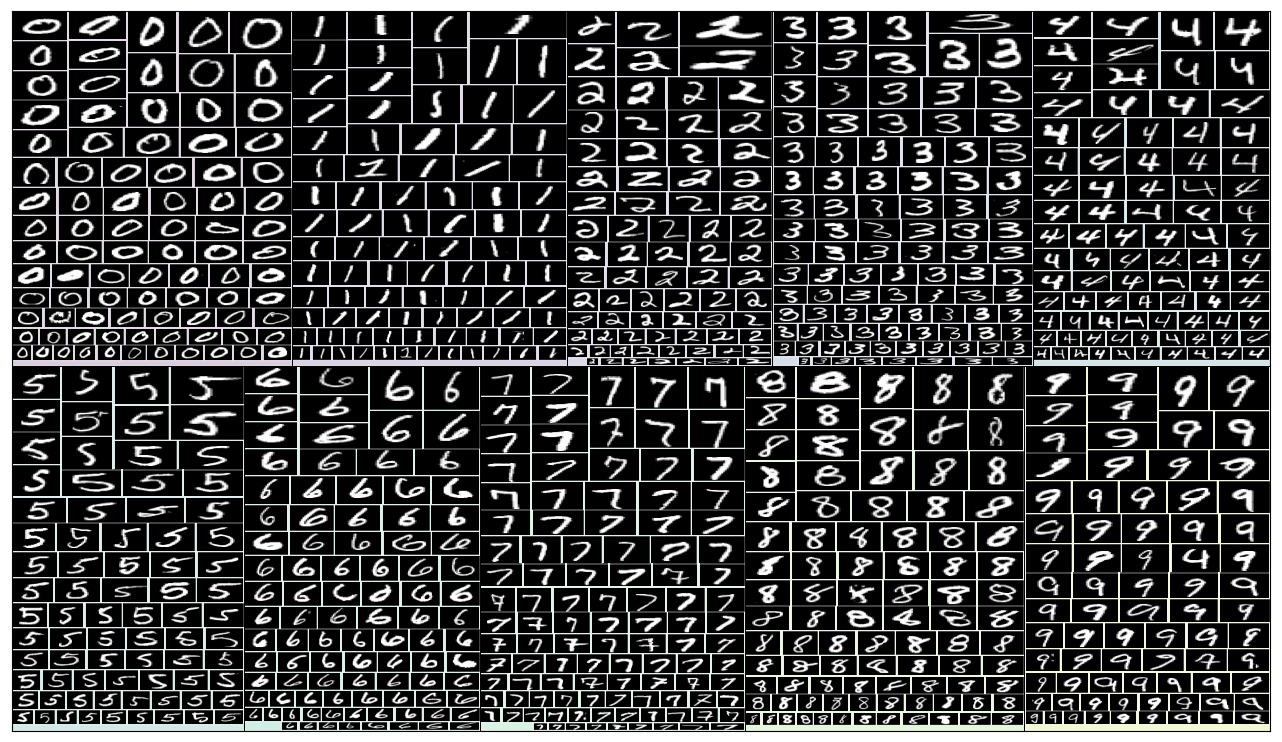}
    \caption{1000 MNIST digits selected via Sinkhorn CO2, sorted by digit and with areas proportional to their weights.\\     \textit{Note: the treemap software used to produce this visualization could not maintain the 1:1 aspect ratio of the original MNIST images.}}
    \label{fig:MNIST_tree}
\end{figure}

The results of this experiment are shown in Figure \ref{fig:MNIST1}, which displays Q-Q (quantile-quantile) plots comparing the percentiles of compression error as compared to random sampling, as well as the $\ell_1$ error of sampled label proportions relative to the population values, $\sum_{i=0}^9 |P(\text{label}=i) - \sum_{j=1}^n w_j \mathbf{1}\{Y_j = i\}|$.
The Q-Q plots capture the variation of both the random samples as well as the random projection employed in the Nystr\"{o}m approximation in Algorithm \ref{CO2}.
Due to the significant regularization level, it is unsurprising that we see our compression scheme dramatically reduces the Sinkhorn error as compared to random subsampling. However, we additionally see that this compression adequately encapsulates the dataset relative to the more concrete $\ell_1$ error on label proportions defined above. This indicates that the quality of the approximation is meaningful beyond this particular divergence.

\begin{figure}[tb]
    \centering
    
    \includegraphics[width=0.49\linewidth]{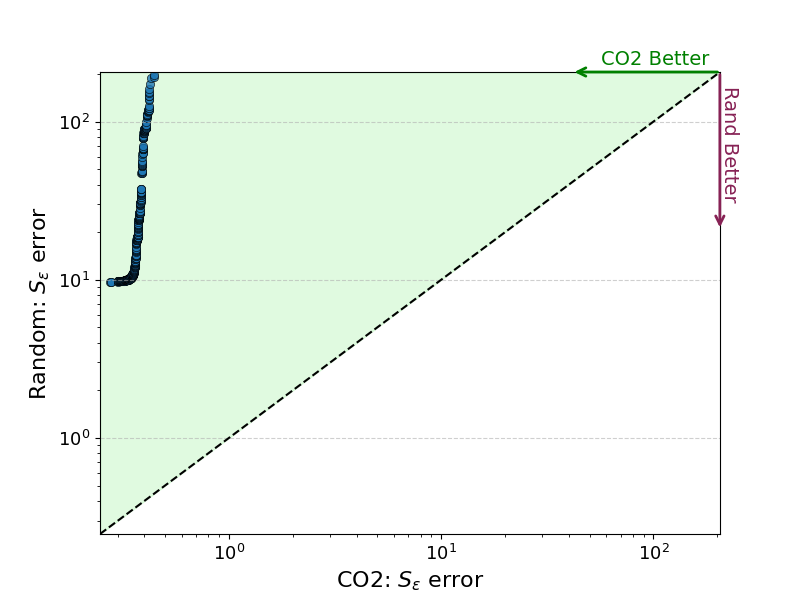}
    \includegraphics[width=0.49\linewidth]{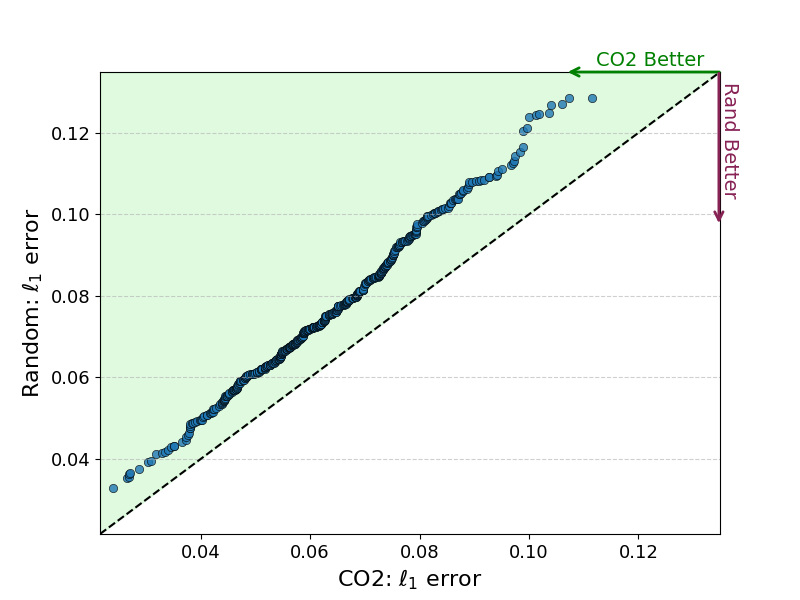}
    \caption{Q-Q plots of the Sinkhorn reconstruction error (left) and $\ell_1$ error between the label proportions (right) of the compressed data as compared to random samples. }
    \label{fig:MNIST1}
\end{figure}

\section{Discussion}

The CO2 algorithm attains rapid, asymptotically lossless compression by targeting the second order expansion of a divergence of interest. This framework invites many natural follow-up questions, and paths towards improved algorithmic efficiency.

One problem of interest is the development of finite sample guarantees for this procedure. The second stage of the algorithm, MMD compression, is related to kernel spectral decay, which is amenable to such an analysis, however the relationship of these bounds to the performance of the coreset for the desired divergence are not immediate. 
Key to our analysis were asymptotic properties of the divergence of interest via a functional delta method. 
To achieve finite sample guarantees, explicit bounds on the remainder of the divergence relative to its quadratic approximation need to be developed.

Our procedure can benefit from future developments in convex MMD compression.
Our current preferred instantiation of CO2 uses a recombination procedure. 
However, we anticipate that this algorithm can be improved as the state-of-the-art in this domain improves. 
In particular, we expect our error bounds to be suboptimal in polynomial spectral decay settings, as our analysis leads to looser error bounds than are suggested by the Mercer decomposition (see Lemma \ref{lem:LN_sob}).
As methodology improves in this domain to achieve minimax MMD compression guarantees, we believe pairing CO2 with such a procedure would lead to asymptotically minimax compression for all sufficiently regular divergences. 
Likewise, we speculate that in finite sample settings minimax MMD compression would lead to minimax CO2 compression whenever the remainder term satisfies appropriate regularity conditions.

A drawback of our procedure is that CO2 requires the explicit derivation of the Hadamard operator, which can be technically demanding. Further, guarantees for this procedure require specialized results, such as the verification of second order Hadamard differentiability in a specific topology. Thus, theoretically, it would be desirable to streamline this process, verifying for certain classes of divergences that the conditions we impose always hold. Practically, the development of algorithms to circumvent the analytic computation of the second order expansion would be useful. 
Automatic differentiation provides one way of achieving this, and such an approach has been recently used to estimate first order derivatives of statistical functionals \citep{luedtke2024simplifyingdebiasedinferenceautomatic}.

\acks{This work was supported by NSF Grant DMS-2210216. AK was supported by the Department of Defense (DoD) through the National Defense Science \& Engineering Graduate (NDSEG) Fellowship Program. The authors thank Lester Mackey for helpful discussions.}

\bibliography{refs}

\newpage

\appendix

\appendix

\DoToC

\section{Functional Taylor expansion}\label{sec: FTE}

Here we provide further details on the intuition of our method, linking compression with respect to a smooth divergence $D$ to a carefully chosen MMD. 

 As we are in the setting of coreset selection, we desire the size of our compressed dataset $m \ll n$. The  condition $D(P_m) = O_p(n^{-1})$ places a natural constraint on the class of functionals of interest.
Typically, such $D$ are required to be sufficiently smooth, made precise by the notion of Hadamard differentiability \citep{Fernholz1983}. 

To understand how to minimize a Hadamard differentiable divergence, we first consider minimizing a smooth function $f$. If we Taylor expand about a point $x$, then we get a local expansion 
$$f(y)-f(x) = \nabla f(x)^T(y-x) + \tfrac{1}{2}(y-x)^T \nabla^2 f(x)(y-x) + o(\|y-x\|^2).$$
If $x$ was a minimizer of $f$, then the first order gradient would vanish, leaving a dominant quadratic term from the Hessian. Now, suppose we generated a sequence $\{y_m\}$ such that $\|(y_m-x)^T \nabla^2 f(x)(y_m-x)\| = O(n^{-1}).$ If the Hessian were non-degenerate, this would necessitate that the remainder term be $o(n^{-1})$, hence the overall error would be $|f(y_m) - f(y)| = O(n^{-1}).$

Hadamard differentiability allows us to recover an analogous expansion for functionals, coined the von Mises expansion \cite[Chapter~9]{van2000asymptotic}. In this setting, for a Hadamard differentiable functional $D$,
\[
D(\mathbb{P}_n) - D(\mathbb{P}) = \dot{D}(\mathbb{P}_n - \mathbb{P}) + \ddot{D}(\mathbb{P}_n - \mathbb{P}) + o_p(n^{-1}).
\]
For a divergence, just like in the case of a smooth function, the first order term vanishes, leaving the dominant MMD. We conjecture that in most cases of interest compressing with respect to it is asymptotically equivalent to the functional of interest. There are obvious counterexamples to this, however, consider
\[
D(\mu) := Q_{K^2}(\mu - \nu) + Q_{K}^2(\mu - \nu).
\]
$K^2$ has more rapid spectral decay than $K$, hence it is not guaranteed that the remainder will decay appropriately even if this is verified for the second order term. This indicates that additional regularity is required. As one alternative, if $D$ was second order \textit{Fr\'{e}chet} differentiable in the topology induced by $Q_G$, then we arrive at the expansion
\[
D(\mu,\nu) = Q_G(\mu - \nu) + o(Q_G(\mu - \nu))
\]
providing control over these higher-order terms. Hadamard differentiability, a weaker notion than Fr\'{e}chet differentiability is also a sufficient degree of regularity,  which we verify in Lemma \ref{Functional Delta Method} by a quick application of the second order functional delta method.

\section{Algorithms}\label{app: Algs}

In this section, we briefly present the internal algorithms used in our implementation of Algorithm \ref{CO2}. Algorithm~\ref{alg:fixed_rank_psd} presents a Nystr\"{o}m approximation algorithm. Algorithm~\ref{Recomb simple} presents a basic version of the recombination algorithm. For a more detailed and efficient implementation, see \href{https://github.com/satoshi-hayakawa/kernel-quadrature}{https://github.com/satoshi-hayakawa/kernel-quadrature}.

\begin{algorithm}[h!]
    \caption{Nystr\"{o}m approximation \citep[][Algorithm 3]{tropp2017fixedrank}}\label{alg:fixed_rank_psd}
    \begin{algorithmic}
        \State \textbf{Input:} matrix \( A \), must be PSD; $\theta \in \mathbb{N},\ 2\leq \theta \leq n$;  rank parameter \( 1 \leq r \leq \theta \)

        \State
        $\Omega \gets$ $\operatorname{randn}(n, \theta)$ 
        \State $\mu \gets$ \text{small value to improve numerical stability}
        \State
        $\Omega,\sim\,  \gets$ $\operatorname{QR}(\Omega)$
        \State 
        $Y \gets  A\Omega$

            \State \( \nu \leftarrow \mu \text{norm}(Y) \)
            \State \( Y \leftarrow Y + \nu \Omega \)
            \State \( B \leftarrow \Omega^* Y \)
            \State \( C \leftarrow \text{chol}\left((B + B^*)/2\right) \) \% Cholesky decomposition
            \State \( (U, \Sigma, \sim) \leftarrow \text{svd}(Y / C, \text{'econ'}) \) \% thin SVD
            \State \( U \leftarrow U(:, 1:r) \), \( \Sigma \leftarrow \Sigma(1:r, 1:r) \) \% truncate to rank \( r \)
            \State \( \Lambda \leftarrow \max\left\{0, \Sigma^2 - \nu I\right\} \)

            \State \Return \( U^T \Lambda U, U, \Lambda \)
    \end{algorithmic}
\end{algorithm}

\begin{algorithm}[tb]
    \caption{Simple recombination \citep[][Section 1.3.3.3]{tchernychova2016caratheodory}}\label{Recomb simple}
    \begin{algorithmic}
        \State $m \gets $ coreset size
        \State $U \gets m -1$ leading Nystr\"{o}m eigenvectors

        \State $k_{\text{diag}} \gets$ diagonal of kernel matrix

        \State $V \gets$ orthogonal complement of $U \oplus \mathbf{1} \oplus k_{\text{diag}}$

        \State $\mu \gets$ initial empirical distribution
        \For{$i \leq n -m +1$}
            \State $\alpha \gets \min d\mu/V[:, i]$ \% choose $\alpha$ so that next iterate has one less non-zero entry
            \State $\mu \gets \mu - \alpha V[:, i]$
        \EndFor
        \State $U \gets [\mathbf{1}, U[\mu > 0, m-1] \in \mathbb{R}^{n \times (n+1)}$

        \State $u \gets$ element of nullspace $U$ (computed by SVD or QR decomposition)

        \State $u \gets u\operatorname{sign}(u^T k_{\text{diag}})$

        \State $\alpha \gets \min d\mu/Vu$
        
        \State $\mu \gets \mu - \alpha u$
        \State
        \Return $\mu$
  \end{algorithmic}
\end{algorithm}

\section{General facts}
Where convenient, we denote $k_\varepsilon(x,y) := \exp\left(\frac{-\|x-y\|^2}{\varepsilon}\right),$ the kernel function of $\mathcal{H}$.
\subsection{Entropic optimal transport}

We briefly summarize key properties of entropic optimal transport used in our argument. See Section \ref{sec: EOT potentials} for basic definitions. The foundational results were established in \citep{csiszar1975divergence, ruschendorf1993note}.

\begin{lemma}\label{EOT Operator}
    \[
    \int f(x,y)\, d \pi_{\mu,\nu}^\varepsilon(x,y) = \int \exp\left(\frac{\phi_{\mu,\nu}(x) + \psi_{\mu,\nu}(y) - \|x-y\|^2}{\varepsilon}\right) f(x,y)\, d\mu(x) d\nu(y).
    \]
    In particular,
    \begin{gather*}
    \int f(x)g(y)\, d \pi_{\mu,\nu}^\varepsilon(x,y) = \int f \mathcal{A}^* g \,d\mu = \int g \mathcal{A} f\, d\nu\\
    \int f(x) + g(y)\, d \pi_{\mu,\nu}^\varepsilon(x,y) = \int f\, d\mu + \int g\, d\nu.
    \end{gather*}
\end{lemma}

\begin{proof}
    The first identity is definitional, and the second and third follow by integrating along marginals.
\end{proof}

\begin{corollary}
    \[
    \int \exp\left(\frac{\phi_{\mu,\nu}(x) + \psi_{\mu,\nu}(y) - \|x-y\|^2}{\varepsilon}\right) d\mu(x) d\nu(y) = 1
    \]
    In fact, we have,
    \[
    \int \exp\left(\frac{\phi_{\mu,\nu}(x) + \psi_{\mu,\nu}(\cdot) - \|x-\cdot \|^2}{\varepsilon}\right) d\mu(x) \equiv 1 \equiv \int \exp\left(\frac{\phi_{\mu,\nu}(\cdot) + \psi_{\mu,\nu}(y) - \|\cdot -y\|^2}{\varepsilon}\right) d\nu(y)
    \]
\end{corollary}

\begin{proof}
    This first is immediate as
    \[
    \int \exp\left(\frac{\phi_{\mu,\nu}(x) + \psi_{\mu,\nu}(y) - \|x-y\|^2}{\varepsilon}\right) d\mu(x) d\nu(y) = \int d\pi_{\mu,\nu}^\varepsilon = 1,
    \]
    as $\pi_{\mu,\nu}^\varepsilon$ is a probability distribution. The second claim is a consequence of the Schr\"{o}dinger equations.
\end{proof}

\begin{proposition}\label{EOT rep}
    \[
    \EOT(\mu,\nu) = \int \phi_{\mu,\nu} d\mu + \int \psi_{\mu,\nu} d\nu.
    \]
\end{proposition}

\begin{proof}
    \begin{align*}
        \EOT(\mu,\nu) &= \int \|x-y\|^2 d\pi_{\mu,\nu}^\varepsilon + \varepsilon \int  \log \left(\frac{d\pi_{\mu,\nu}^\varepsilon}{d\mu \otimes d\nu}\right) d\pi_{\mu,\nu}^\varepsilon\\
        &= \int \|x-y\|^2 d\pi_{\mu,\nu}^\varepsilon + \varepsilon \int  \log \circ \exp\left(\frac{\phi_{\mu,\nu}(x) + \psi_{\mu,\nu}(y) - \|x-y\|^2}{\varepsilon}\right) d\pi_{\mu,\nu}^\varepsilon\\
        &= \int \phi_{\mu,\nu}(x) +  \psi_{\mu,\nu}(y) d\pi_{\mu,\nu}^\varepsilon\\
        &= \int \phi_{\mu,\nu}d\mu + \int \psi_{\mu,\nu} d\nu.
    \end{align*}
\end{proof}

\begin{lemma}\label{Double Marginal}
    $\phi_{\mu,\mu} = \phi_{\nu,\nu}$ under the constraint $\phi_{\mu,\mu}(x_0) = \psi_{\mu,\mu}(x_0).$

    Additionally, for $\mu,\nu$ we have
    \begin{align*}
    \int \exp\left(\frac{\phi_{\mu,\mu}(x) + \psi_{\mu,\mu}(y) - \|x-y\|^2}{\varepsilon}\right)d\mu(x) d\nu(y) = 1,\\
    \int \exp\left(\frac{\phi_{\nu,\nu}(x) + \psi_{\nu,\nu}(y) - \|x-y\|^2}{\varepsilon}\right)d\mu(x) d\nu(y) = 1.
    \end{align*}
\end{lemma}

\begin{proof}
    The first identity was observed in \citep{feydy2019interpolating} and \citep{goldfeld2022limit}. The second comes from the Schr\"{o}dinger equations, as 
    \[
    \int \exp\left(\frac{\phi_{\mu,\mu}(x) + \psi_{\mu,\mu}(y) - \|x-y\|^2}{\varepsilon}\right)d\mu(x) = 1
    \]
    holds identically for all $y$.
\end{proof}

\subsection{Reproducing kernel Hilbert space}
Here we recall properties of the RKHS which will be essential for our analysis.

\begin{lemma}\label{kernel embedding} 
    For $f \in \mathcal{F}$,
    \[
    \langle f, K(g\,d\mu)\rangle_{\mathcal{F}} = \langle f, g\rangle_{L^2(\mu)}.
    \]
\end{lemma}

\begin{proof}
    This result is standard for distributional embeddings \citep[Chapter 4]{berlinet2011reproducing}, however we provide a brief sketch. 
    \begin{align*}
        \langle f, g\rangle_{L^2(\mu)} &= \int f(x) g(x) \,d\mu(x)\\
        &= \int \langle f, k(x,\cdot)\rangle_{\mathcal{F}} g(x)\,d\mu(x)\\
        &= \int \langle f, g(x) k(x,\cdot)\rangle_{\mathcal{F}}\, d\mu(x)\\
        &= \langle f, \int g(x)k(x,\cdot) \,d\mu(x)\rangle_{\mathcal{H}}\\
        &= \langle f, K(g d\mu)\rangle_{\mathcal{F}}.
    \end{align*}
\end{proof}

\begin{lemma}\label{cts eval}
    For $h\in \mathcal{H}$,
    \[
    \|h\|_\infty \leq \|h\|_\mathcal{H}.
    \]
\end{lemma}
\begin{proof}
    \begin{align*}
        \|h\|_\infty &= \sup_{y\in \mathcal{X}} |h(y)| = \sup_{y\in \mathcal{X}} \langle h, k_\varepsilon(y,\cdot)\rangle_\mathcal{H}\leq \|h\|_\mathcal{H} \sup_{y\in \mathcal{X}} \|k_\varepsilon(y,\cdot)\|_\mathcal{H}\\
        &= \|h\|_\mathcal{H}\sup_{y\in \mathcal{X}} \sqrt{k_\varepsilon(y,y)}= \|h\|_\mathcal{H}.
    \end{align*}
\end{proof}

\begin{lemma}\label{operator bound}
    For $\mu\in \mathcal{P}(\mathcal{X})$ and $f\in C(\mathcal{X})$,
    \[
    \|K(fd\mu)\|_\mathcal{H} \leq \|f\|_\infty.
    \]
\end{lemma}
\begin{proof}
    Applying Lemma \ref{kernel embedding},
    \begin{align*}
        \|K(fd\mu)\|_\mathcal{H}^2 &= \left\|\int k_\varepsilon(\cdot,y)f(y) d\mu(y) \right\|_\mathcal{H}^2 \\
        &=\left\langle \int k_\varepsilon(\cdot,y)f(y) d\mu(y),  f \right\rangle_{L^2(\mu)}\\
        &\leq \left\|\int k_\varepsilon(\cdot,y)f(y) d\mu(y)\right\|_{L^2(\mu)} \left\| f \right\|_{L^2(\mu)}\\
        &\leq \left\|\int k_\varepsilon(\cdot,y)f(y) d\mu(y)\right\|_{L^2(\mu)} \left\| f \right\|_\infty.
    \end{align*}
    Hence, the claim follows if $\left\|\int k_\varepsilon(\cdot,y)f(y) d\mu(y)\right\|_{L^2(\mu)} \leq \left\| f \right\|_\infty$. To verify this, we apply H\"older's inequality
    \begin{align*}
        \left\|\int k_\varepsilon(\cdot,y)f(y) d\mu(y)\right\|_{L^2(\mu)}&\leq  \left\|\int k_\varepsilon(\cdot,y)f(y) d\mu(y)\right\|_\infty\\
        &= \sup_{x\in \mathcal{X}} \left|\int k_\varepsilon(x,y)f(y) d\mu(y)\right|\\
        &\leq \sup_{x\in \mathcal{X}} \int k_\varepsilon(x,y)|f(y)| d\mu(y)\\
        &\leq \int|f(y)| d\mu(y)\\
        &\leq \|f\|_\infty.
    \end{align*}

\end{proof}

\begin{lemma}\label{kernel equiv}
    If $k$ is a characteristic kernel, then the multiplication operator $\mathfrak{M}_f : \ell^\infty(\mathcal{F}_1) \to \ell^\infty(\mathcal{F}_1),$ $\mathfrak{M}_f \mu = f\mu$ is bounded for $f : \mathcal{X}\rightarrow\mathbb{R}$ continuous.
\end{lemma}
\begin{proof}
    By the closed graph theorem \citep[Theorem 5.2, pg. 166]{kato2013perturbation}, to show $\mathfrak{M}_f$ is bounded, it suffices to show that it is closed --- that is, for $\mu_n \to \mu$ in $\ell^\infty(\mathcal{F}_1)$, $\mathfrak{M}_f \mu_n \to \nu$ in $\ell^\infty(\mathcal{F}_1)$, then $\nu = \mathfrak{M}_f \mu$. As $k$ is characteristic, convergence in $\ell^\infty(\mathcal{F}_1)$ is equivalent to weak convergence, hence for all bounded, continuous $h$
    \[
    \mathfrak{M}_f \mu_n (h) := \int h\, d(\mathfrak{M}_f \mu_n) := \int hf\, d\mu_n \to \nu(h).
    \]
    But, $hf$ is also a bounded, continuous function, hence $\mu_n \to \mu$ implies
    \[
    \int hf\, d\mu_n \to \int hf\, d\mu = \int h\, d(\mathfrak{M}_f \mu) = \mathfrak{M}_f\mu(h).
    \]
    Thus $\nu(h) = \mathfrak{M}_f \mu(h)$ for all bounded, continuous $h$, and therefore $\nu = \mathfrak{M}_f \mu$.     
\end{proof}

\begin{lemma}\label{function conv}
    Let $k$ be a characteristic kernel. Suppose that $f_n \to f$ in $C(\mathcal{X})$. Then, for all $\mu \in \ell^\infty(\mathcal{F}_1)$, $\mathfrak{M}_{f_n} \mu \to \mathfrak{M}_f \mu$.
\end{lemma}

\begin{proof}
    Consider $h\in C(\mathcal{X})$. By definition, $\mathfrak{M}_{f_n} \mu(h) = \int h f_n\, d\mu.$ Now, $hf_n \to hf$ uniformly, hence the dominated convergence theorem \citep[Theorem 2.24]{folland1999real} verifies $\int h f_n\, d\mu \to \int hf\, d\mu = \mathfrak{M}_f \mu(h)$ as desired.
\end{proof}

For an operator $A: F \to F'$, we define $\|A\|_{F \to F'} := \sup_{\|f\|_F \leq 1}\|A f\|_{F'}$. When $F=F'$, we shorten this to $\|A\|_F$.

\begin{lemma}\label{uniform kernel equiv}
    Let $B$ be compact in $C(\mathcal{X})$ and $k$ be characteristic. There exists a constant $C_1>0$ such that, for all $\mu \in \ell^\infty(\mathcal{F}_1)$ and $f\in B$,
    \begin{align}\label{uniform op bound}
    \|\mathfrak{M}_f \mu\|_{\ell^\infty(\mathcal{F}_1)} \leq C_1  \| \mu\|_{\ell^\infty(\mathcal{F}_1)}.
    \end{align}
    Further, if all $f\in B$ are bounded away from 0, then there exists $C_2 > 0$ such that 
    \begin{align}\label{inv op bound}
    \| \mu\|_{\ell^\infty(\mathcal{F}_1)} \leq C_2 \|\mathfrak{M}_f \mu\|_{\ell^\infty(\mathcal{F}_1)}.
    \end{align}
        In other words, for the kernel $k_f(x,y) = k(x,y)f(x) f(y)$, $Q_K$ and $Q_{K_f}$ are equivalent, and there exists $C_3, C_4>0$ such that, for all $\mu \in \ell^\infty(\mathcal{F}_1)$ and $f\in B$, 
    \[
    C_3Q_{K_f}(\mu) \leq Q_{K}(\mu) \leq C_4 Q_{K_f}(\mu)
    \]
\end{lemma}
\begin{proof}
    We first focus on \eqref{uniform op bound}. From Lemma \ref{kernel equiv}, we know that each $\mathfrak{M}_f$ is a bounded linear operator, hence it suffices to show that the collection of operators $\{\mathfrak{M}_f : f\in B\}$ is uniformly bounded. We argue by contradiction. Suppose not, and take a $B$-valued sequence $f_n$ such that $\mathfrak{M}_{f_n}$ has operator norm greater than $n$. By the compactness of $B$, we can extract a subsequence, which we relabel to be $f_n$, such that $f_n$ uniformly converges to a continuous function $f$. By Lemma \ref{function conv}, it follows that $\mathfrak{M}_{f_n} \to \mathfrak{M}_f$ pointwise, hence by the Banach-Steinhaus theorem \citep[Theorem 33.1]{treves1967topological} $\|\mathfrak{M}_{f_n}\|_{\ell^\infty(\mathcal{F}_1) \to \ell^\infty(\mathcal{F}_1)}$ is uniformly bounded. This gives us a contradiction.

    For \eqref{inv op bound}, notice that
    \[
    \|\mu\|_{\ell^\infty(\mathcal{F}_1)} = \|\mathfrak{M}_{1/f} \mathfrak{M}_f \mu\|_{\ell^\infty(\mathcal{F}_1)} \leq \|\mathfrak{M}_{1/f}\|_{\ell^\infty(\mathcal{F}_1) \to 
    \ell^\infty(\mathcal{F}_1)} \|\mathfrak{M}_f \mu\|_{\ell^\infty(\mathcal{F}_1)},
    \]
    hence it suffices to show that the operators $\mathfrak{M}_{1/f}$ are uniformly bounded. By \eqref{uniform op bound}, it suffices to show that $\operatorname{inv}(B) := \{1/f: f\in B\}$ is compact. We show that it is sequentially compact. Consider a sequence $1/f_n$, and extract a subsequence where the functions in the denominator, relabeled to be $f_n$, converge to some continuous $f$ uniformly --- this is necessarily possible since $B$ is compact. Now, 
    \[
    \|1/f_n - 1/f\|_\infty = \left\|\frac{f - f_n}{f_n f}\right\|_\infty \leq m^{-2} \|f - f_n\|_\infty
    \]
    for $m:= \min_{x\in\mathcal{X}, b\in B} |b(x)|$, hence $1/f_n \to 1/f$ establishing \eqref{inv op bound}.

    The final assertion is immediate, as we have the relation $Q_{K_f}(\mu) = Q_K(\mathfrak{M}_f \mu) = \|\mathfrak{M}_f \mu\|_{\ell^\infty(\mathcal{F}_1)}^2.$
\end{proof}

\begin{corollary}\label{cor: op bound}
    Suppose that $\phi:\ell^\infty(\mathcal{F}_1) \to B \subseteq C(\mathcal{X})$ is bounded and linear, and that $B$ is totally bounded in $C(\mathcal{X})$. Then there exists a constant $C>0$ depending on $\phi$ such that $\|\mathfrak{M}_{\phi(\mu)}\|_{\ell^\infty(\mathcal{F}_1) \to \ell^\infty(\mathcal{F}_1)} \leq C \|\mu\|_{\ell^\infty(\mathcal{F}_1)}$.
\end{corollary}
\begin{proof}
    By assumption $\phi$ is continuous, hence $\sup_{\|\mu\|_{\ell^\infty(\mathcal{F}_1)} \leq 1} \|\phi(\mu)\|_{B} < \infty$. As $B$ is totally bounded, the closure of $\{\phi(\mu) : \|\mu\|_{\ell^\infty(\mathcal{F}_1)} \leq 1\}$ is compact in $C(\mathcal{X})$, and by Lemma \ref{uniform kernel equiv}, 
    \[
    \sup_{\|\mu\|_{\ell^\infty(\mathcal{F}_1)} \leq 1} \|\mathfrak{M}_{\phi(\mu)}\|_{\ell^\infty(\mathcal{F}_1) \to \ell^\infty(\mathcal{F}_1)} = C < \infty.
    \]
    It follows that for all $\mu$,
    \[
    \|\mathfrak{M}_{\phi(\mu)}\|_{\ell^\infty(\mathcal{F}_1) \to \ell^\infty(\mathcal{F}_1)} = \|\mu\|_{\ell^\infty(\mathcal{F}_1)} \|\mathfrak{M}_{\phi(\frac{\mu}{\|\mu\|_{\ell^\infty(\mathcal{F}_1)}})}\|_{\ell^\infty(\mathcal{F}_1) \to \ell^\infty(\mathcal{F}_1)} \leq C \|\mu\|_{\ell^\infty(\mathcal{F}_1)}.
    \]
\end{proof}

\section{Functional compression}
\begin{lemma}\label{app: Local Min}
    If $D$ is a divergence and $D(\cdot):=D(\,\cdot\,,\mathbb{P})$ is Hadamard differentiable, then the Hadamard derivative vanishes at $\mathbb{P}$.
\end{lemma}
\begin{proof}
    This argument is analogous to the Euclidean case where the gradient of a differentiable function vanishes at a local minimum. Indeed, observe for any $t,h_t\to h$, $D(\mathbb{P} + th_t) - D(\mathbb{P}) \ge 0 $, and so
    \[
    \lim_{t\uparrow 0} \frac{D(\mathbb{P} + th_t) - D(\mathbb{P})}{t} \leq 0,\quad \lim_{t\downarrow 0} \frac{D(\mathbb{P} + th_t) - D(\mathbb{P})}{t} \geq 0.
    \]
    The Hadamard differentiability of $D$ implies the two limits above agree, and so
    \[
    \dot{D}(h):=\lim_{t\to 0} \frac{D(\mathbb{P} + th_t) - D(\mathbb{P})}{t} = 0.
    \]
    As $h$ was arbitrary, $\Phi = 0$.
\end{proof}

\begin{lemma}\label{lem: rkhs donsker}
    Suppose that $\mathcal{F}$ is a RKHS with continuous kernel. Then $\mathcal{F}_1$, the unit ball of this space, is $\mathbb{P}$-Donsker.
\end{lemma}
\begin{proof}
    This is an immediate consequence of Mercer's theorem \citep[Theorem 40]{berlinet2011reproducing} in combination with \cite[Theorem 2.13.2]{van1996wellner}.
\end{proof}

\begin{proof}[Proof of Lemma~\ref{Functional Delta Method}]
 $\sqrt{n}(P_m - \mathbb{P}) = \sqrt{n}(\mathbb{P}_n - \mathbb{P}) + \sqrt{n}(P_m - \mathbb{P}_n) = \sqrt{n}(\mathbb{P}_n - \mathbb{P}) + o_p(1) \rightsquigarrow \mathbb{G}$ in $\ell^\infty(\mathcal{F}_1)$ by Lemma \ref{lem: rkhs donsker}. Hence, we can apply the second order functional delta method \cite[Lemma~13]{goldfeld2022limit} to get
 \begin{align*}
     nD(P_m) = n[D(P_m) - D(\mathbb{P}) - \dot{D}(P_m - \mathbb{P})] \rightsquigarrow \frac{1}{2}Q_G(\mathbb{G}) = O_p(1).
 \end{align*}

To get sharper control, we can apply the second part of \cite[Theorem 2]{romisch2004delta} to get
\begin{align*}
    n[D(P_m)-D(\mathbb{P})-\ddot{D}(P_m-\mathbb{P})]&=o_p(1),\\
    n[D(\mathbb{P}_n)-D(\mathbb{P})-\ddot{D}(\mathbb{P}_n-\mathbb{P})]&=o_p(1).
\end{align*}
Hence, by the continuous mapping theorem,
\begin{align*}
    n[D(P_m)-D(\mathbb{P}_n)-\ddot{D}(P_m-\mathbb{P})+\ddot{D}(\mathbb{P}_n-\mathbb{P})]&=o_p(1),
\end{align*}
and so
\begin{align*}
    D(P_m)-D(\mathbb{P}_n)&= \ddot{D}(P_m-\mathbb{P})-\ddot{D}(\mathbb{P}_n-\mathbb{P})+ o_p(1/n) \\
    &= \int G(P_m-\mathbb{P})d(P_m-\mathbb{P}) - \int G(\mathbb{P}_n-\mathbb{P})d(\mathbb{P}_n-\mathbb{P}) + o_p(1/n) \\
    &= \int G(P_m-\mathbb{P}_n)d(P_m-\mathbb{P}_n) + 2\int G(P_m-\mathbb{P}_n)d(\mathbb{P}_n-\mathbb{P}) + o_p(1/n) \\
    &= \operatorname{MMD}_G^2(P_m,\mathbb{P}_n) + 2\int G(P_m-\mathbb{P}_n)d(\mathbb{P}_n-\mathbb{P}) + o_p(1/n).
\end{align*}

The second term above is an inner product in the RKHS with kernel $g$ between the differences of the $g$-kernel mean embeddings of $P_m$ and $\mathbb{P}_n$ and $\mathbb{P}_n$ and $\mathbb{P}$. Hence, applying Cauchy-Schwarz yields
\begin{align*}
    |D(P_m)-D(\mathbb{P}_n)|&\le \operatorname{MMD}_G^2(P_m,\mathbb{P}_n) + 2\operatorname{MMD}_G(P_m,\mathbb{P}_n)\operatorname{MMD}_G(\mathbb{P}_n,\mathbb{P}) + o_p(1/n).
\end{align*}
Since $\mathbb{P}_n$ is the empirical distribution of an iid sample from $P$, $\operatorname{MMD}_G(\mathbb{P}_n,\mathbb{P})=O_p(n^{-1/2})$, thus in combination with Lemma \ref{Hadamard Optimality},
\[
D(P_m) = D(\mathbb{P}_n) + o_p(1/n).
\]
 \end{proof}

 \begin{proof}[Proof of Lemma~\ref{Hadamard Optimality}]
      Let $h \in \ell^\infty(\mathcal{F}_1)$, and $Kh := \int k(\cdot,y) dh \in \mathcal{F}$ the kernel mean embedding. Define $\Tilde{Q}_{G}: K(\ell^\infty(\mathcal{F}_1)) \to \mathbb{R}$, $\Tilde{Q}_{G}(Kh) = Q_G(h)$. Note that, by assumption, $K$ is injective, hence this map is well-defined.  Let $K h_t \to K h$ in $\mathcal{F}$, hence $\operatorname{MMD}_K(h_t, h) = \|Kh_t - Kh\|_{\mathcal{F}} \to 0$ and thus $h_t\to h$ in $\ell^{\infty}(\mathcal{F}_1)$. We see then that $\Tilde{Q}_{G}(Kh_t) = Q_G(h_t) \to Q_G(h)=\Tilde{Q}_{G}(Kh)$ as $Q_G$ is continuous on $\ell^\infty(\mathcal{F}_1)$. Hence $\Tilde{Q}_{G}$ is continuous with respect to the RKHS topology.

      To see that $K(\ell^\infty(\mathcal{F}_1))$ is dense in $\mathcal{F}$, note that the finite linear combinations $\sum_{i=1}^j a_i \delta_{x_i}$ are contained in $\ell^\infty(\mathcal{F}_1)$, and $K\left(\sum_{i=1}^j a_i \delta_{x_i}\right)  =  \sum_{i=1}^j a_i k(x_i, \cdot)$, $\{x_i\}_{i=1}^j$ and $j\in\mathbb{N}$, is dense in $\mathcal{F}$ by \citep[Theorem 3]{berlinet2011reproducing}. Hence $\Tilde{Q}_G$ is continuous and densely defined, and therefore admits a continuous extension to $\mathcal{F}$. By \cite[Theorem~2.7]{kato2013perturbation}, there exists a bounded linear operator $\Tilde{G}$ so that $\Tilde{Q}_G(f) = \langle f, \Tilde{G} f\rangle_{\mathcal{F}}$. Thus, for $C:=\|\tilde{G}\|$,
      \[
      \Tilde{Q}_G(f) \leq C \|f\|_{\mathcal{F}}^2,
      \]
      and, for any $h\in\ell^\infty(\mathcal{F}_1)$
      \[
      \int Gh\, dh = Q_G(h) = \Tilde{Q}_G(Kh) \leq C\| Kh\|_{\mathcal{F}}^2 = C\int Kh\, dh, 
      \]
      or in other words, $CK \succeq G$ as desired.
 \end{proof}

\section{Entropic optimal transport}
 \subsection{Regularity of potentials}\label{app:regPotential}

We will now seek to show that $\tau(\mu,\nu):=(u_{\mu,\nu},v_{\mu,\nu})$ is Hadamard differentiable with respect to the weak topology induced by $\mathcal{H}$, the ball of the Gaussian kernel RKHS. Hereafter we fix a generic $(\mu,\nu)\in \mathcal{P}(\mathcal{X}) \times \mathcal{P}(\mathcal{X}) $ and establish differentiability of $\tau$ at this point. Implicitly, we take as domain $\ell^\infty(\mathcal{H}_1) \cap \mathcal{P}(\mathcal{X})$. In other words, as entropic optimal transport is only well-posed for measures of the same mass, when considering quantities such as $\tau(\mu + \delta, \nu)$, we require $\delta(\mathbf{1}) = 0$, i.e. $\mu$ and $\mu+ \delta$ have the same total mass. Not only is this essential for the definition of entropically regularized optimal transport, but it will also have relevance in our functional analytic approach.

Our primary theoretical tool will be \cite[~Lemma 3.9.34]{van1996wellner},

\begin{proposition}\label{Wellner}
    Let $\Theta, \mathbb{L}$ be metric spaces, and $Z(\Theta, \mathbb{L})$ the set of maps from $\Theta \to \mathbb{L}$ that are uniformly norm-bounded with at least one zero. Assume $\Psi: \Theta \to \mathbb{L}$ is uniformly norm-bounded, injective, has a zero $\theta_0 \in \Theta$ and has inverse continuous at 0. Let $\Psi$ be Fr\'{e}chet-differentiable at $\theta_0$ with derivative $\dot{\Psi}_{\theta_0}$ that is continuously invertible. Then $\phi : Z(\Theta,\mathbb{L}) \to \Theta$, $\phi(\Psi) = \Psi^{-1}(0)$ is Hadamard-differentiable at $\Psi$ tangentially to the set of $z\in \ell^\infty(\Theta, \mathbb{L})$ that are continuous at $\theta_0$. The derivative is given by $\dot{\phi}_\Psi(z) = -\dot{\Psi}_{\theta_0}^{-1}(z(\theta_0)).$
\end{proposition}

Define $K$ to be the kernel embedding,
\[
K f := \int \exp\left(\frac{-\|x-y\|^2}{\varepsilon}\right) df(y).
\]

The map $\Psi$ we associate to the functional system of equations

\begin{gather*}
\Psi_{\mu,\nu}(f,g) = \left( f - K\left(\frac{d\nu}{g}\right), g - K\left(\frac{d\mu}{f}\right)\right) 
\end{gather*}
hence $(u_{\mu,\nu}, v_{\mu,\nu})$ are the zeroes we are solving for. Note that in general, $(u_{\mu,\nu}, v_{\mu,\nu})$ are not unique, as they can be modified up to positive rescaling $(Cu_{\mu,\nu}, C^{-1}v_{\mu,\nu})$, however, following \citep{goldfeld2022limit} a unique member can always be selected by specifying $u_{\mu,\nu}(x_0) = v_{\mu,\nu}(x_0)$ for a selected point $x_0$ in the domain. We adapt this constraint to our setting by appending to our system of equations $\langle k(x_0,\cdot), f - g\rangle_{\mathcal{H}} = f(x_0) - g(x_0)$, which we denote by
\begin{gather*}
\Psi_{\mu,\nu, x_0}(f,g) = \left( f - K\left(\frac{d\nu}{g}\right), g - K\left(\frac{d\mu}{f}\right), \langle k(x_0,\cdot), f - g\rangle_\mathcal{H} \right) 
\end{gather*}

At times, we will further restrict the domain of this map to $\Psi_{\mu,\nu, x_0}(\cdot,\cdot)$ to localized spaces of the form $B_r:= \mathcal{H}^2 \cap [B(u, r) \times B(v,r)]$ with $r>0$ and, for $f\in\mathcal{H}$, $B(f,r):=\{g\in \mathcal{H} : \|f-g\|_{\mathcal{H}}< r\}$. When doing this, we will always choose $r\le r^*:=\min\{ \min u, \min v\}/2$, which ensures that, for all $(f,g) \in B_r$, $f,g$ are bounded in $[m,M]$ with $m>0$ and $M<\infty$ --- see Lemma~\ref{lem:fgBdd} for details. Following Lemma~\ref{inverse function} we will put even further requirements on $r$.

\begin{lemma}\label{lem:fgBdd}
    Let $r\leq r^*$. There exists $0<m<M<\infty$ such that for all $(f,g) \in B_r$, $m \leq f,g \leq M.$
\end{lemma}

\begin{proof}
    By Lemma \ref{cts eval}, for $f \in B(0,r)$, $\|f\|_\infty < r$. Thus,
    \begin{align*}
    \inf_{x\in \mathcal{X}} u(x) + f(x) &\geq \inf_{x\in \mathcal{X}} u(x) -  \sup_{x\in \mathcal{X}} f(x) \geq 2r^* - r^* = r^*,\\
    \sup_{x\in \mathcal{X}} u(x) + f(x) &\leq \sup_{x\in \mathcal{X}} u(x) +  \sup_{x\in \mathcal{X}} f(x) \leq \|u\|_\mathcal{H} + r^*.
    \end{align*}
    The proof is similar for $v+f$. Hence, we can take $m = r^*, M = r^* + \max\{\|u\|_\mathcal{H}, \|v\|_\mathcal{H}\}$.
\end{proof}

In the style of \citep{gonzalezsanz2022weak}, we introduce the following matrix notation. For operators $A_i:\mathcal{H} \to \mathcal{H}$, we denote
\[
\begin{bmatrix}
    A_1 & A_2\\
    A_3 & A_4
\end{bmatrix} : \mathcal{H} \times \mathcal{H} \to \mathcal{H}\times \mathcal{H},\quad 
\begin{bmatrix}
    A_1 & A_2\\
    A_3 & A_4
\end{bmatrix}
\begin{bmatrix}
    h_1\\
    h_2
\end{bmatrix}
= 
\begin{bmatrix}
    A_1(h_1) + A_2(h_2)\\
    A_3(h_1) + A_4(h_2)
\end{bmatrix}
\]

For any $h\in\mathcal{H}$ and $\gamma\in \mathcal{P}(\mathcal{X})$, define $\mathcal{A}_{h,\gamma}:\mathcal{H} \to \mathcal{H}$ so that
\begin{align}\label{eq: integral operator}
\mathcal{A}_{h,\gamma}f :=  K\left(\frac{f\, d\gamma}{h^2}\right). 
\end{align}

\begin{lemma} \label{Frechet Derivative}
Let $(h_1,h_2) \in \mathcal{H}^2$ and $r\leq r^*$. The map $\Psi_{\mu,\nu}$ is continuously Fr\'{e}chet differentiable on $B_r$, with derivative 
\begin{align*}
    \dot{\Psi}_{\mu,\nu}(f,g)[h_1,h_2] &= \left(h_1 + \int \exp\left(\frac{ - \|\cdot-y\|^2}{\varepsilon}\right)\frac{h_2(y)}{g^2(y)} d\nu(y),\right.\\
    &\left.\quad\quad\quad  h_2 + \int \exp\left(\frac{- \|\cdot-y\|^2}{\varepsilon}\right)\frac{h_1(y)}{f^2(y)} d\mu(y)\right) \\
    &= \begin{bmatrix}
    I & \mathcal{A}_{g,\nu}\\
    \mathcal{A}_{f,\mu} & I
\end{bmatrix}\begin{bmatrix}
    h_1\\
    h_2
\end{bmatrix}.
\end{align*}

\end{lemma}

\begin{proof}
    We can express
    \[
    \Psi_{\mu,\nu}(f,g) = 
    \begin{bmatrix}
        I & -K\left(\frac{d\nu}{\cdot}\right)\\
        -K\left(\frac{d\mu}{\cdot}\right) & I
    \end{bmatrix}
    \begin{bmatrix}
        \ f\ \\
        g
    \end{bmatrix}.
    \]
    The identity map $I(h) = h$  is continuously Fr\'{e}chet differentiable on $B_r$ with $\dot{I} = I$. By the linearity of differentiation, it suffices to establish the continuous Fr\'{e}chet differentiability of $K(d\mu/\cdot)$ and $K(d\nu/\cdot)$. Without loss of generality, we do this for $K(d\mu/\cdot)$. 

    To establish Fr\'{e}chet differentiability at $f\in B(u, r)$, we will show that
    \begin{align*}
        \left\|K\left(\frac{d\mu}{f+h}\right) - K\left(\frac{d\mu}{f}\right) + \mathcal{A}_{f,\mu}\left(h\right)\right\|_\mathcal{H} &= o(\|h\|_\mathcal{H})
    \end{align*}
    Since this result only concerns behavior as $h\rightarrow 0$ in $\mathcal{H}$, it suffices to study the behavior of the left-hand side when $\|h\|_{\mathcal{H}}$, and therefore also $\|h\|_\infty$, is small. To this end, fix $h$ such that $\|h\|_\infty < m/2$.     

    Letting $R_f(h) := \left(\frac{1}{f + h} - \frac{1}{f}\right) + \frac{h}{f^2}$, the linearity of $K$ and Lemma \ref{operator bound} show
    \begin{align*}
        \left\|K\left(\frac{d\mu}{f+h}\right) - K\left(\frac{d\mu}{f}\right) + \mathcal{A}_{f,\mu}\left(h\right)\right\|_\mathcal{H} &= \left\|K\left( R_f(h) d\mu\right)\right\|_\mathcal{H} \leq\|R_f(h)\|_\infty.
    \end{align*}
    Hence, it suffices to show that $\|R_f(h)\|_\infty=o(\|h\|_{\mathcal{H}})$. By Taylor's remainder theorem,
    \begin{align*}
    \left\|R_f(h)\right\|_\infty &= \sup_{y\in\mathcal{X}}\left|\left(\frac{1}{f(y) + h(y)} - \frac{1}{f(y)}\right) + \frac{h(y)}{f^2(y)}\right|\leq \sup_{y\in\mathcal{X}, |c| < \|h\|_\infty}\left|\frac{1}{(f(y) + c)^3} c^2 \right|\\
    &\leq \sup_{|c| \leq \|h\|_\infty } \left(\frac{2}{m}\right)^3c^2 = \frac{8}{m^{3}} \|h\|_\infty^2 = o(\|h\|_\infty) = o(\|h\|_{\mathcal{H}}),
    \end{align*}
    where the second inequality follows from Lemma \ref{lem:fgBdd} and the final bound comes from Lemma \ref{cts eval}. 
    This verifies the form of the Fr\'{e}chet derivative. It remains to show it is continuously differentiable, that is, for $f_n \to f$ in $\mathcal{H}$, 
    \begin{align}
        \sup_{\|h\|_{\mathcal{H}} \leq 1} \| \mathcal{A}_{f_n,\mu}(h) - \mathcal{A}_{f,\mu}(h)  \|_{\mathcal{H}} \overset{n\rightarrow\infty}{\longrightarrow} 0. \label{eq:contFrechetDiff}
    \end{align}
    A similar manipulation as before gives us
    \begin{align*}
        \| \mathcal{A}_{f_n,\mu}(h) - \mathcal{A}_{f,\mu}(h)  \|_{\mathcal{H}} &= \left\| K\left(\frac{hd\mu}{f_n^2} - \frac{hd\mu}{f^2}\right)\right\|_{\mathcal{H}} = \left\| K\left(\frac{h(f^2 - f_n^2)d\mu}{f^2f_n^2}\right)\right\|_{\mathcal{H}}\leq \left\|\frac{h(f^2 - f_n^2)}{f^2f_n^2}\right\|_{\infty}\\
        &\leq \frac{\|h\|_{\mathcal{H}}}{m^4}  \|f^2 - f_n^2\|_\infty\leq \frac{2M \|h\|_{\mathcal{H}}}{m^4} \|f - f_n\|_\infty\leq \frac{2M \|h\|_{\mathcal{H}}}{m^4} \|f - f_n\|_{\mathcal{H}},
    \end{align*}
    the penultimate inequality follows as $x^2$ is Lipschitz on $[m,M]$ with Lipschitz constant $2M$. Taking a supremum over $\|h\|_\mathcal{H} \leq 1$ on both sides establishes \eqref{eq:contFrechetDiff}. Note that the definition of Fr\'{e}chet differentiability requires $\mathcal{A}_{f,\mu}$ to be a bounded linear operator, but this is immediate by a similar computation as the previous.
\end{proof}

\begin{corollary}
    The map $\Psi_{\mu,\nu, x_0}$ is continuously Fr\'{e}chet differentiable on $B_r$.
\end{corollary}
\begin{proof}
    As the last coordinate of this map is by a definition a continuous linear functional that has no dependence of $(f,g)$ the proof follows immediately.
\end{proof}

For convenience we will often simply denote $\dot{\Psi}_{\mu,\nu} := \dot{\Psi}_{\mu,\nu}(u,v)$, and likewise for similar operators. We now compile facts about this operator.

\begin{lemma} \label{operator properties}
$\mathcal{A}_{u,\mu},\mathcal{A}_{v,\nu}$ are compact and $I - \mathcal{A}_{u,\mu}\mathcal{A}_{v,\nu},I - \mathcal{A}_{v,\nu}\mathcal{A}_{u,\mu}$ have null spaces $\operatorname{span}\{v\},\operatorname{span}\{u\}$ respectively.
\end{lemma}

\begin{proof}
   We again argue for the operator associated to $u$ without loss of generality. By Lemma \ref{lem:fgBdd}, $u\geq m > 0$. Let $\{f_i\}_{i=1}^\infty$ be a bounded sequence in $\mathcal{H}$. As bounded subsets of $\mathcal{H}$ are compact in the supnorm topology, we can take a subsequence $\{f_{j}'\}$ that is Cauchy, and, as
   \[
   \left\| \frac{f_i' - f_j'}{u}\right\|_\infty\leq \left\| \frac{f_i' - f_j'}{m}\right\|_\infty,
   \]
   $\{f_i'/u\}$ is also Cauchy. By Lemma \ref{operator bound},
   \begin{align*}
   \| \mathcal{A}_{u,\mu}(f_i' - f_j')) \|_{\mathcal{H}} &= \left\|K\left(\frac{f_i' - f_j'}{u}d\mu\right)\right\|_\mathcal{H} < \left\|\frac{f_i' - f_j'}{u}\right\|_\infty.
   \end{align*}
    Thus, for all $\varepsilon>0$, we can select $N>0$ such that for $i,j>N$, $\|K(f_i' - f_j')d\mu) \|_{\mathcal{H}}^2 < \varepsilon$, that is, $\mathcal{A}_{u,\mu}(f_i')$ is Cauchy in $\mathcal{H}$, and thus the operator is compact. 

    For the last claim, notice that by definition, $\mathcal{A}_{u,\mu}\mathcal{A}_{v,\nu}v = \mathcal{A}_{u,\mu} u = v$, hence $\operatorname{span}\{v\}$ is a subspace of the nullspace. As $v$ is a positive function, and $\mathcal{A}_{u,\mu}\mathcal{A}_{v,\nu}$ is a positive and compact operator from $\mathcal{H}\to \mathcal{H}$, by the Krein-Rutman theorem \cite[Theorem~1.2]{deimling2010nonlinear}, $v$ is the unique eigenvector with eigenvalue 1, and thus there are no other solutions to $I - \mathcal{A}_{u,\mu}\mathcal{A}_{v,\nu} = 0.$ A similar argument holds for $I - \mathcal{A}_{v,\nu} \mathcal{A}_{u,\mu}$.
\end{proof}

It is clear from the above that $\dot{\Psi}_{\mu,\nu}$ is not invertible. Indeed, it has a nullspace $\operatorname{span}\{(u,-v)\}$. Thus, to prove invertibility, this requires us to limit the domain of the operator, and a natural choice is $D := D_{u,v}= \operatorname{span}\{(u,-v)\}^\perp$. The range is more complicated, but we will show that the corresponding space is 
\[
R = \{(f,g): (f - \mathcal{A}_{v,\nu}g) \in \operatorname{span}\{v\}^\perp, (g - \mathcal{A}_{u,\mu} f) \in \operatorname{span}\{u\}^\perp\}
\]
This will lead to us denoting $(I - \mathcal{A}_{v,\nu} \mathcal{A}_{u,\mu})^{-1}$ as the pseudo-inverse of $I - \mathcal{A}_{v,\nu} \mathcal{A}_{u,\mu}$, and likewise for $I - \mathcal{A}_{u,\mu} \mathcal{A}_{v,\nu}$. We show that $(I - \mathcal{A}_{v,\nu} \mathcal{A}_{u,\mu}):\operatorname{span}\{u\}^\perp \to \operatorname{span}\{v\}^\perp$ has a well-defined inverse along this restriction.

\begin{lemma}\label{lem: pseudo-inversion}
    The operators $I-\mathcal{A}_{u,\mu}\mathcal{A}_{v,\nu}$, $I-\mathcal{A}_{v,\nu}\mathcal{A}_{u,\mu}$ have the following properties:
    \begin{enumerate}
        \item $\operatorname{image}(I-\mathcal{A}_{u,\mu}\mathcal{A}_{v,\nu}) = \operatorname{span}\{u\}^\perp$, $\operatorname{image}(I-\mathcal{A}_{v,\nu}\mathcal{A}_{u,\mu}) = \operatorname{span}\{v\}^\perp$
        \item These operators have well-defined inverses as maps $(I-\mathcal{A}_{u,\mu}\mathcal{A}_{v,\nu}):\operatorname{span}\{v\}^\perp \to \operatorname{span}\{u\}^\perp$, $(I-\mathcal{A}_{v,\nu}\mathcal{A}_{u,\mu}):\operatorname{span}\{u\}^\perp \to \operatorname{span}\{v\}^\perp$
        \item $(I - \mathcal{A}_{v,\nu} \mathcal{A}_{u,\mu})^{-1} \mathcal{A}_{v,\nu} = \mathcal{A}_{v,\nu} (I - \mathcal{A}_{u,\mu} \mathcal{A}_{v,\nu})^{-1}$ as operators from $\operatorname{span}\{u\}^\perp \to \operatorname{span}\{u\}^\perp$
        \item $(I - \mathcal{A}_{u,\mu} \mathcal{A}_{v,\nu})^{-1} \mathcal{A}_{u,\mu} = \mathcal{A}_{u,\mu} (I - \mathcal{A}_{v,\nu} \mathcal{A}_{u,\mu})^{-1}$ as operators from $\operatorname{span}\{v\}^\perp \to \operatorname{span}\{v\}^\perp$
        \item $(I-\mathcal{A}_{u,\mu} \mathcal{A}_{v,\nu})^{-1} = I + \mathcal{A}_{u,\mu}(I - \mathcal{A}_{v,\nu} \mathcal{A}_{u,\mu})^{-1} \mathcal{A}_{v,\nu}$ as operators from $\operatorname{span}\{u\}^\perp \to \operatorname{span}\{v\}^\perp$
    \end{enumerate}

\end{lemma}

\begin{proof}
    We verify the claims for $I-\mathcal{A}_{u,\mu}\mathcal{A}_{v,\nu}$ as the other operator is similar.

    First, observe that $\mathcal{A}_{h,\gamma}$ is self-adjoint for arbitrary $h$ and $\gamma$. Indeed, by Lemma \ref{kernel embedding},
    \[
    \langle f, \mathcal{A}_{h,\gamma} g \rangle_{\mathcal{H}} = \int 
    \frac{fg}{h^2}\, d\gamma = \langle \mathcal{A}_{h,\gamma} f,  g \rangle_{\mathcal{H}}.
    \]

By construction this operator is injective, that is, it has no non-trivial kernel as $u$ spans the nullspace. Thus to establish continuous invertibility it suffices to show it also surjects by \citep[Theorem 17.1 and Corollary 1]{treves1967topological}. By \citep[Theorem 5.13, pg. 234]{kato2013perturbation}, $\operatorname{image}(I - \mathcal{A}_{v,\nu}\mathcal{A}_{u,\mu}) = \operatorname{ker}((I - \mathcal{A}_{v,\nu}\mathcal{A}_{u,\mu})^*)^\perp = \operatorname{ker}(I - \mathcal{A}_{u,\mu}\mathcal{A}_{v,\nu})^\perp = \operatorname{span}\{v\}^\perp$, verifying the claim.

We now compute,
\begin{align*}
    &\mathcal{A}_{v,\nu} = \mathcal{A}_{v,\nu} (I - \mathcal{A}_{u,\mu} \mathcal{A}_{v,\nu})(I - \mathcal{A}_{u,\mu} \mathcal{A}_{v,\nu})^{-1} = (I - \mathcal{A}_{v,\nu} \mathcal{A}_{u,\mu}) \mathcal{A}_{v,\nu}(I - \mathcal{A}_{u,\mu} \mathcal{A}_{v,\nu})^{-1}\\
    &\Longrightarrow (I - \mathcal{A}_{v,\nu} \mathcal{A}_{u,\mu})^{-1} \mathcal{A}_{v,\nu} = \mathcal{A}_{v,\nu}(I - \mathcal{A}_{u,\mu} \mathcal{A}_{v,\nu})^{-1},
\end{align*}
and the fourth equation is similar. Finally, 
\begin{align*}
    &I + \mathcal{A}_{u,\mu}(I - \mathcal{A}_{v,\nu} \mathcal{A}_{u,\mu})^{-1} \mathcal{A}_{v,\nu} = I + \mathcal{A}_{u,\mu} \mathcal{A}_{v,\nu} (I - \mathcal{A}_{u,\mu} \mathcal{A}_{v,\nu})^{-1}\\
    &= [(I - \mathcal{A}_{u,\mu} \mathcal{A}_{v,\nu})^{-1} - (I - \mathcal{A}_{u,\mu} \mathcal{A}_{v,\nu})^{-1}] + I + \mathcal{A}_{u,\mu} \mathcal{A}_{v,\nu} (I - \mathcal{A}_{u,\mu} \mathcal{A}_{v,\nu})^{-1}\\
    &= (I - \mathcal{A}_{u,\mu} \mathcal{A}_{v,\nu})^{-1}.
\end{align*}

\end{proof}

Other works have arrived at an inverse for the Fr\'{e}chet derivative of the form

\[
    [\dot{\Psi}_{\mu,\nu}(u,v)]^{-1} = 
    \begin{bmatrix}
        (I - \mathcal{A}_{v,\nu} \mathcal{A}_{u,\mu})^{-1} & -(I - \mathcal{A}_{v,\nu} \mathcal{A}_{u,\mu})^{-1} \mathcal{A}_{v,\nu}\\
        -(I - \mathcal{A}_{u,\mu} \mathcal{A}_{v,\nu})^{-1} \mathcal{A}_{u,\mu} & (I - \mathcal{A}_{u,\mu} \mathcal{A}_{v,\nu})^{-1}
    \end{bmatrix}.
    \]

Note however that is not true under our specified constraints. For instance,
\[
\begin{bmatrix}
        (I - \mathcal{A}_{v,\nu} \mathcal{A}_{u,\mu})^{-1} & -(I - \mathcal{A}_{v,\nu} \mathcal{A}_{u,\mu})^{-1} \mathcal{A}_{v,\nu}\\
        -(I - \mathcal{A}_{u,\mu} \mathcal{A}_{v,\nu})^{-1} \mathcal{A}_{u,\mu} & (I - \mathcal{A}_{u,\mu} \mathcal{A}_{v,\nu})^{-1}
\end{bmatrix}
\begin{bmatrix}
    v\\
    u
\end{bmatrix} = 
\begin{bmatrix}
    (I - \mathcal{A}_{v,\nu} \mathcal{A}_{u,\mu})^{-1}(u - u)\\
    (I - \mathcal{A}_{u,\mu} \mathcal{A}_{v,\nu})^{-1}(v- v)
\end{bmatrix} = 0,
\]
which is clearly not a solution to the equation. We can decompose it as
\begin{gather*}
\begin{bmatrix}
        (I - \mathcal{A}_{v,\nu} \mathcal{A}_{u,\mu})^{-1} & -(I - \mathcal{A}_{v,\nu} \mathcal{A}_{u,\mu})^{-1} \mathcal{A}_{v,\nu}\\
        -(I - \mathcal{A}_{u,\mu} \mathcal{A}_{v,\nu})^{-1} \mathcal{A}_{u,\mu} & (I - \mathcal{A}_{u,\mu} \mathcal{A}_{v,\nu})^{-1}
\end{bmatrix} = \quad \quad \quad\\
\begin{bmatrix}
    (I - \mathcal{A}_{v,\nu} \mathcal{A}_{u,\mu})^{-1} & 0 \\
    0 & (I - \mathcal{A}_{u,\mu} \mathcal{A}_{v,\nu})^{-1}
\end{bmatrix}
\begin{bmatrix}
    I & -\mathcal{A}_{v,\nu}\\
    -\mathcal{A}_{u,\mu} & I
\end{bmatrix},
\end{gather*}
so that $[\dot{\Psi}_{\mu,\nu}(u,v)]^{-1}:R \to \operatorname{span}\{v\}^\perp \times \operatorname{span}\{u\}^\perp \to \operatorname{span}\{u\}^\perp \times \operatorname{span}\{v\}^\perp$ which is a strict subset of $D$. Thus we see that this map cannot recover elements of $D \cap \operatorname{span}\{(u,0), (0,v)\}$, necessitating an additional component.

Let $\pi_{u}, \pi_D$ denote the projection into $\operatorname{span}\{u\}^\perp, D$ respectively.

\begin{lemma}\label{Frechet Inverse}
    $\dot{\Psi}_{\mu,\nu}(u,v) : D \to R$ is continuously invertible, and
    \[
    [\dot{\Psi}_{\mu,\nu}(u,v)]^{-1} = 
    \begin{bmatrix}
        (I - \mathcal{A}_{v,\nu} \mathcal{A}_{u,\mu})^{-1} & -(I - \mathcal{A}_{v,\nu} \mathcal{A}_{u,\mu})^{-1} \mathcal{A}_{v,\nu}\\
        -(I - \mathcal{A}_{u,\mu} \mathcal{A}_{v,\nu})^{-1} \mathcal{A}_{u,\mu} & (I - \mathcal{A}_{u,\mu} \mathcal{A}_{v,\nu})^{-1}
    \end{bmatrix}
    + \pi_D 
    \begin{bmatrix}
        (I - \pi_u) & 0\\
        0 & 0
    \end{bmatrix}.
    \]
\end{lemma}

\begin{proof}
We first show that the designated domain and range are correctly specified. Let $(f,g)\in D$. Then,
\[
\dot{\Psi}_{\mu,\nu}[f, g] = (f + \mathcal{A}_{v, \nu} g, \mathcal{A}_{u,\mu}f + g).
\]
We check only the first constraint, as the other is similar.
\[
f + \mathcal{A}_{v, \nu} g - \mathcal{A}_{v,\nu}(\mathcal{A}_{u,\mu}f + g) = (I - \mathcal{A}_{v,\nu}\mathcal{A}_{u,\mu})f,
\]
which satisfies the desired property by Lemma \ref{lem: pseudo-inversion}.

Let us now similarly verify that $[\dot{\Psi}_{\mu,\nu}(u,v)]^{-1}$, as defined above, maps $R\to D$. The claim is trivial for the second term, so we consider only the first. Let $(f,g) \in R$, then the action of this operator can be expressed as
\[
((I - \mathcal{A}_{v,\nu}\mathcal{A}_{u,\mu})^{-1}(f - \mathcal{A}_{v,\nu} g), (I - \mathcal{A}_{u,\mu}\mathcal{A}_{v,\nu})^{-1}(g - \mathcal{A}_{u,\mu} f)).
\]
By construction, $f - \mathcal{A}_{v,\nu} g \in \operatorname{span}\{v\}^\perp$, $g - \mathcal{A}_{u,\mu} f \in \operatorname{span}\{u\}^\perp$, hence the claim follows by Lemma \ref{lem: pseudo-inversion}.

 To assess the injectivity of $\dot{\Psi}_{\mu,\nu}$, consider a solution $(h_1,h_2)$ to
    \[
    \dot{\Psi}_{\mu,\nu}(u,v)[h_1,h_2] = 
    \begin{bmatrix}
        h_1 + \mathcal{A}_{v,\nu}h_2\\
        h_2 + \mathcal{A}_{u,\mu}h_1
    \end{bmatrix} =
    0.
    \]
      The first equation implies $h_1 = - \mathcal{A}_{v,\nu} h_2$, and substituting this into the second equation, we see $h_2 = 
    \mathcal{A}_{u,\mu}\mathcal{A}_{v,\nu} h_2$, that is, $(I- \mathcal{A}_{u,\mu}\mathcal{A}_{v,\nu})h_2 = 0$. Thus $h_2$ is in the nullspace of this operator, hence by Lemma \ref{operator properties} it is $\alpha v$ for some $\alpha\in\mathbb{R}$, and $h_1 = - \mathcal{A}_{v,\nu} h_2 = -\alpha u$. Thus the only solutions are in $\operatorname{span}\{(-u, v)\}$, which is excluded from our domain.

    By \citep[Theorem 17.1 and Corollary 1]{treves1967topological}, our analysis is finished if we can show surjectivity of $\dot{\Psi}_{\mu,\nu}$, and that the solution operator has the desired form. Let $(h, \xi) \in R$, and suppose that $\dot{\Psi}_{\mu,\nu}[f,g] = (h, \xi)$. The first equation reads $f = h - \mathcal{A}_{v,\nu} g$, thus, substituting this into the second equation yields
    \[
    g + \mathcal{A}_{u,\mu}(h - \mathcal{A}_{v,\nu} g) = \xi\quad \Longleftrightarrow \quad (I-\mathcal{A}_{u,\mu} \mathcal{A}_{v,\nu})g = (\xi - \mathcal{A}_{u,\mu} h).
    \]
    By Lemma \ref{lem: pseudo-inversion} this has a solution, in particular, we can take the pseudo-inverse,
    \[
    g = (I-\mathcal{A}_{u,\mu} \mathcal{A}_{v,\nu})^{-1}(\xi - \mathcal{A}_{u,\mu} h).
    \]
    Going back to the first equation, selecting $f = h - \mathcal{A}_{v,\nu}g$ gives us one element of the pre-image of $(h, \xi)$. Substituting in the value for $g$, this yields
    \begin{align*}
        f &= h - (I -\mathcal{A}_{v,\nu} (I-\mathcal{A}_{u,\mu} \mathcal{A}_{v,\nu})^{-1}[\xi - \mathcal{A}_{u,\mu} h] = h - (I - \mathcal{A}_{v,\nu} \mathcal{A}_{u,\mu})^{-1}[\mathcal{A}_{v,\nu} \xi - \mathcal{A}_{v,\nu} \mathcal{A}_{u,\mu} h]\\
         &=  [I - (I - \mathcal{A}_{v,\nu} \mathcal{A}_{u,\mu})^{-1} (I - \mathcal{A}_{v,\nu} \mathcal{A}_{u,\mu})]h + (I - \mathcal{A}_{v,\nu} \mathcal{A}_{u,\mu})^{-1}[h - \mathcal{A}_{v,\nu} \xi]\\
        &= [I - \pi_u]h + (I - \mathcal{A}_{v,\nu} \mathcal{A}_{u,\mu})^{-1}[h - \mathcal{A}_{v,\nu} \xi]
    \end{align*}
    
    where we used Lemma \ref{lem: pseudo-inversion} in the second equality, and the definition of the pseudo-inverse in the last.
    Thus projecting this solution into $D$, which only causes the solution to differ by an element of the null-space, preserving its validity, shows that the equation can indeed be solved. This only changes the $[I - \pi_u]h$ term as the remainder of the solution is an element of $\operatorname{span}\{u\}^\perp\times \operatorname{span}\{v\}^\perp \subseteq D$.
    
\end{proof}

\begin{lemma}\label{Frechet Inverse}
    $\dot{\Psi}_{\mu,\nu, x_0}(u,v) : \mathcal{H}^2 \to R \times \mathbb{R}$ is continuously invertible. In particular, for $(f, g, \tau) \in R\times \mathbb{R}$,
    \[
    [\dot{\Psi}_{\mu,\nu, x_0}(u,v)]^{-1}[f, g, \tau] = \dot{\Psi}_{\mu,\nu}^{-1}[f,g] + \frac{\tau - \langle (k_{x_0}, -k_{x_0}), \dot{\Psi}_{\mu,\nu}^{-1}[f,g]\rangle_{\mathcal{H}^2}}{2\langle k_{x_0}, u\rangle} (u, -v)
    \]
\end{lemma}

\begin{proof}
    We verify continuous invertibility by showing $\dot{\Psi}_{\mu,\nu, x_0}$ is bijective. We first consider injectivity, which we verify by showing the operator has no non-trivial kernel. $\dot{\Psi}_{\mu,\nu}$ has nullspace $\operatorname{span}\{(u, -v)\}$ so it suffices to check $(u, -v)$ does not map to 0. Indeed,
    \[
    \Psi_{\mu,\nu, x_0}(u, -v) = (0, 0, \langle k_{x_0}, u + v\rangle_{\mathcal{H}}) = (0, 0, 2 u(x_0)) \neq 0
    \]
    by the constraint $u(x_0) = v(x_0) > 0.$ We now check that it is surjective. Let $(h,\xi) \in R$ be arbitrary, and let $(f,g) \in D$ be such that $\dot{\Psi}_{\mu,\nu}[f,g] = (h, \xi)$. Then
    \[
    \Psi_{\mu,\nu, x_0}(f + \alpha u, g - \alpha v) = (h, \xi, [f(x_0) - g(x_0)] + 2\alpha u(x_0))
    \]
    and as $\alpha$ is arbitrary any value can be achieved for the last coordinate. 

    Now, to explicitly construct the inverse, we already see that $\dot{\Psi}_{\mu,\nu}^{-1}$ recovers the part in $D$, thus it suffices to compute the component in $D^\perp = \operatorname{span}\{(u,-v)\}$, the value denoted $\alpha$ above. From the equation we just worked out, we see that
    \[
    \tau = [f(x_0) - g(x_0)] + 2\alpha u(x_0) \quad \Longleftrightarrow \quad \alpha = \frac{\tau - [f(x_0) - g(x_0)]}{2 u(x_0)},
    \]
    which is equivalent to the claim.
\end{proof}

We now apply an inverse function theorem to verify the final remaining property needed to apply Proposition \ref{Wellner}.

\begin{lemma}\label{inverse function}
    $\Psi_{\mu,\nu,x_0}$ is continuous on $B_r \times \mathbb{R}$ for $r\leq r^*$, and there exists $r_*$ such that $\Psi_{\mu,\nu, x_0}$ has continuous inverse on $B_r$ for $r\leq r_*$.
\end{lemma}

\begin{proof}
    That $\Psi_{\mu,\nu, x_0}$ is continuous follows immediately from the existence of the Fr\'{e}chet derivative shown in Lemma \ref{Frechet Derivative}. That its inverse exists and is continuous for sufficiently small $r$ follows from the inverse function theorem on Banach spaces \cite[10.2.5]{dieudonne1960treatise}, after noting that the Fr\'{e}chet derivative of $\Psi_{\mu,\nu,x_0}$ is continuous in a neighborhood of $(u,v, 0)$ (Lemma~\ref{Frechet Derivative}) and invertible at this point (Lemma~\ref{Frechet Inverse}).
\end{proof}

We can now apply Proposition \ref{Wellner}, which we will then leverage with the chain rule to establish Hadamard differentiability of $\tau$.

\begin{lemma}\label{lem: phi Had}
    Let $r<\min\{r^*,r_*\}$ and define $\phi: Z(B_r, \mathcal{H} \times \mathcal{H}) \to B_r$ with $\phi(\Psi):= \Psi^{-1}(0)$. The map $\phi$ is Hadamard differentiable at $\Psi_{\mu,\nu}$ tangentially to the set of $z\in \ell^\infty(B_r, \mathcal{H}\oplus \mathcal{H})$ that are continuous at $(u,v)$. In particular, $\dot{\phi}_{\Psi_{\mu,\nu}} = -\dot{\Psi}_{\mu,\nu}^{-1}(z(u,v)).$
\end{lemma}

\begin{proof}
    This follows immediately from Lemmas \ref{Frechet Derivative}, \ref{Frechet Inverse}, and \ref{inverse function} combined with Proposition \ref{Wellner}.
\end{proof}

To apply the chain rule we must establish Hadamard differentiability of $\Psi_{\mu,\nu}$ in $\ell^\infty(\mathcal{H}_1)$. To this end, we prove the following auxiliary claims.

\begin{lemma}\label{operator conv}
    Suppose $\gamma_t \to \gamma$ in $\ell^\infty(\mathcal{H}_1)$. Then $K\left(\frac{d\gamma_t}{f}\right) \to K\left(\frac{d\gamma}{f}\right)$, $K\left(\frac{d\gamma_t}{g}\right) \to K\left(\frac{d\gamma}{g}\right)$ uniformly over $(f,g) \in B_r$.
\end{lemma}

\begin{proof}
    We argue for $f$ without loss of generality. Applying Lemma \ref{kernel embedding}, we see
    \begin{align*}
    \left\|K\left(\frac{d\gamma_t}{f}\right) -K\left(\frac{d\gamma}{f}\right)\right\|_\mathcal{H}^2 &= \left\|K\left(\frac{d\gamma_t - d\gamma}{f}\right)\right\|_\mathcal{H}^2\\
    &= \left\langle  K\left(\frac{d\gamma_t - d\gamma}{f}\right), \frac{1}{f} \right\rangle_{L^2(\gamma_t - \gamma)}\\
    &= \iint \frac{k_\varepsilon(x,y)}{f(x) f(y)} d(\gamma_t - \gamma)(x)d(\gamma_t - \gamma)(y)\\
    &\leq C \iint k_\varepsilon(x,y) d(\gamma_t - \gamma)(x)d(\gamma_t - \gamma)(y)\\
    &= C \|\gamma_t - \gamma\|_{\ell^\infty(\mathcal{H}_1)}^2,
    \end{align*}
    where the existence of the uniform constant $C$ in the final inequality is shown in Lemma \ref{kernel equiv}, where the $f$ are bounded away from 0 by Lemma \ref{lem:fgBdd}. Hence,
    \[
    \sup_{(f,g)\in B_r}\left\|K\left(\frac{d\gamma_t}{f}\right) -K\left(\frac{d\gamma}{f}\right)\right\|_\mathcal{H} \leq \sqrt{C} \|\gamma_t - \gamma\|_{\ell^\infty(\mathcal{H}_1)}.
    \]
\end{proof}

\begin{lemma}\label{continuity of potentials}
    If $(\mu_n, \nu_n) \rightsquigarrow (\mu,\nu)$, then $u_{\mu_n,\nu_n}\to u_{\mu,\nu}$ and $v_{\mu_n,\nu_n}\to v_{\mu,\nu}$ in $\mathcal{H}$.
\end{lemma}

\begin{proof}
    We argue for $u_{\mu,\nu}$ without loss of generality. Define $u_n:= u_{\mu_n, \nu_n}$ and $v_n := v_{\mu_n, \nu_n}$. From \cite[Lemma~1]{goldfeld2022limit} we have that $u_{\mu_n,\nu_n} \to u_{\mu,\nu}$ in $C^s(\mathcal{X})$ for any $s$. Hence, $\|u_{\mu_n,\nu_n} - u_{\mu,\nu}\|_\infty\to 0$. 
    Using the defining relation $u = K\left(\frac{d\nu}{v}\right)$, the linearity of $K$, Lemma \ref{kernel embedding}, and H\"{o}lder's inequality,
    \begin{align*}
        \|u - u_n\|_{\mathcal{H}}^2 &= \left\langle u - u_n,K\left(\frac{d\nu}{v}\right) - K\left(\frac{d\nu_n}{v_n}\right)\right\rangle_{\mathcal{H}}\\
        &= \left\langle u - u_n,K\left(\frac{(v_n - v)d\nu}{v v_n}\right) - K\left(\frac{d\nu - d\nu_n}{v_n}\right)\right\rangle_{\mathcal{H}}\\
        &= \left\langle u - u_n,  \frac{v_n - v}{v v_n}\right\rangle_{L^2(\nu)} + \left\langle u - u_n, \frac{1}{v_n} \right\rangle_{L^2(\nu - \nu_n)}\\
        &\leq \|u - u_n\|_\infty \left\|\frac{v - v_n}{ v v_n}\right\|_\infty + \|u- u_n\|_{L^2(\nu - \nu_n)} \left\| \frac{1}{v_n}\right\|_{L^2(\nu - \nu_n)}.
    \end{align*}
    As $u,v\geq m$, the denominator of the first term will be bounded away from 0 for large $n$, hence the term converges to 0 by $\|u - u_n\|_\infty \to 0.$  
    For the remaining term, observe that for large $n$, $ \|u-u_n\|_\infty \leq 1$ and $\left\|1/v_n\right\|_\infty \leq 2/m$ almost everywhere, hence H\"older's inequality yields
    \[
     \|u- u_n\|_{L^2(\nu - \nu_n)} \left\| \frac{1}{v_n}\right\|_{L^2(\nu - \nu_n)} \leq \|u- u_n\|_{L^2(\nu - \nu_n)}\le \|u- u_n\|_{\infty}\|\nu-\nu_n\|_{\mathrm{tv}}^2 \frac{2}{m} = \|u- u_n\|_{\infty} \frac{8}{m} \to 0,
    \]
    where $\|\cdot\|_{\mathrm{tv}}$ denotes the total variation norm, which is bounded above by two for the difference of two probability distributions.
\end{proof}

Define $B^\delta := \{(\gamma^1, \gamma^2) \in \ell^\infty(\mathcal{H}_1)^2 \cap \mathcal{P}(\mathcal{X})^2: \|\mu - \gamma^1\|_{\ell^\infty(\mathcal{H}_1)}, \|\nu - \gamma^2\|_{\ell^\infty(\mathcal{H}_1)} < \delta\}.$

\begin{lemma}\label{lem: close zero}
    There exists $\delta>0$ such that for all $\gamma = (\gamma^1, \gamma^2) \in B^\delta$, $\Psi_{\gamma} \in Z(B_r, \mathcal{H}\oplus \mathcal{H})$.
\end{lemma}

\begin{proof}
    That $\Psi_\gamma \in Z(\mathcal{H}^2, \mathcal{H} \oplus \mathcal{H})$ is immediate as the 0, $(u_\gamma, v_\gamma)$, necessarily belongs to $\mathcal{H}^2$ by its relation to the integral operator $K$. Thus it suffices to show that $(u_\gamma, v_\gamma) \in B_r.$ We argue by contradiction. Suppose no such $\delta$ exists. Then we can take a sequence $\gamma_n := (\gamma^1_n, \gamma^2_n)$ such that $\gamma_n \in B^{1/n}$ and $\|(u_{\gamma_n}, v_{\gamma_n}) - (u,v)\|_{\mathcal{H} \oplus \mathcal{H}} \geq r$. However, by assumption $\gamma_n \rightsquigarrow (\mu,\nu)$, so by Lemma \ref{continuity of potentials} it must be that $\|(u_{\gamma_n}, v_{\gamma_n}) - (u,v)\|_{\mathcal{H} \oplus \mathcal{H}} < r$ for $n$ large enough. This is a contradiction, hence the existence of such $\delta$ follows.
\end{proof}

\begin{lemma}\label{lem: J Had}
    Let $J : B^\delta \to Z(B_r \times \mathbb{R}, \mathcal{H}^2)$ be defined to be $J(\Tilde{\mu},\Tilde{\nu}) = \Psi_{\Tilde{\mu},\Tilde{\nu}, x_0}$. $J$ is Hadamard differentiable at $(\mu,\nu)$ with Hadamard derivative
    \[
    \dot{J}_{\mu,\nu}[\gamma^1, \gamma^2](f,g) = -\left[K\left(\frac{d\gamma^2}{g}\right), K\left(\frac{d\gamma^1}{f}\right), 0\right].
    \]
    $\dot{J}_{\mu,\nu}[\gamma^1,\gamma^2]$ is continuous at $(u,v).$
\end{lemma}

\begin{proof}
    Let $\mu + t\gamma^1_t, \nu + t\gamma^2_t \in \mathcal{P}(\mathcal{X})$ be such that $\gamma^1_t \to \gamma^1, \gamma^2_t \to \gamma^2$ in $\ell^\infty(\mathcal{H}_1)$. As the last coordinate of $\Psi_{\mu,\nu, x_0}$ does not depend on $t$, it suffices to consider only $\Psi_{\Tilde{\mu},\Tilde{\nu}}$. We compute,
    \begin{align*}
        t^{-1}&\|\Psi_{\mu + t\gamma^1_t, \nu + t\gamma^2_t} - \Psi_{\mu,\nu} - t \dot{J}_{\mu,\nu}\|_{\ell^\infty(B_r, \mathcal{H}^2)} \\
        &= \left\|\left[K\left(\frac{d\gamma^2_t}{\cdot}\right) - K\left(\frac{d\gamma^2}{\cdot}\right), K\left(\frac{d\gamma^1_t}{\cdot}\right) - K\left(\frac{d\gamma^1}{\cdot}\right)\right]\right\|_{\ell^\infty(B_r, \mathcal{H}^2)}\\
        &= \sup_{f,g \in B_r} \left\|\left[K\left(\frac{d\gamma^2_t}{g}\right) - K\left(\frac{d\gamma^2}{g}\right), K\left(\frac{d\gamma^1_t}{f}\right) - K\left(\frac{d\gamma^1}{f}\right)\right]\right\|_{\mathcal{H}^2}.
    \end{align*}
    The right-hand side goes to zero as $t\rightarrow 0$ by Lemma \ref{operator conv}. We now verify that $\dot{J}_{\mu,\nu}[\gamma^1,\gamma^2]$ is continuous at $(u,v)$. Let $(f_n,g_n)\to (u,v)$ in $B_r$. We verify this for the term depending on $\gamma^2$, as the proof for the other is similar.
    \begin{align*}
         \left\|K\left(\frac{d\gamma^2}{g_n}\right) - K\left(\frac{d\gamma^2}{v}\right)\right\|_{\mathcal{H}}^2 &=\left\|K\left(\frac{(g- g_n)d\gamma^2}{gg_n}\right)\right\|_{\mathcal{H}}^2 \leq \left\|\frac{g - g_n}{gg_n} \right\|_\infty \leq \frac{1}{m^2}\|g-g_n\|_\infty \to 0,
    \end{align*}
    where we apply Lemmas \ref{cts eval} and \ref{operator bound}.
\end{proof}

\begin{theorem}\label{theo:tau had}
$\tau : (\mu,\nu) \mapsto (u_{\mu,\nu},v_{\mu,\nu})$ is Hadamard differentiable in $\ell^{\infty}(\mathcal{H}_1) \times \ell^{\infty}(\mathcal{H}_1)$, with
\begin{gather*}
    \dot{\tau}_{\mu,\nu}[\gamma^1, \gamma^2] = 
    \begin{bmatrix}
    (I - \mathcal{A}_{v,\nu} \mathcal{A}_{u,\mu})^{-1} \mathcal{A}_{\sqrt{v},\gamma^2} \mathbf{1} - (I - \mathcal{A}_{v,\nu} \mathcal{A}_{u,\mu})^{-1}  \mathcal{A}_{v,\nu} \mathcal{A}_{\sqrt{u},\gamma^1} \mathbf{1}\\
    (I - \mathcal{A}_{u,\mu} \mathcal{A}_{v,\nu})^{-1} \mathcal{A}_{\sqrt{u},\gamma^1} \mathbf{1} - (I - \mathcal{A}_{u,\mu} \mathcal{A}_{v,\nu})^{-1}  \mathcal{A}_{u,\mu} \mathcal{A}_{\sqrt{v},\gamma^2}\mathbf{1}
    \end{bmatrix}\\
    + \pi_D 
    \begin{bmatrix}
        (I - \pi_u) \mathcal{A}_{\sqrt{v}, \gamma^2} \mathbf{1}\\
        0
    \end{bmatrix} + \rho (u, -v)
\end{gather*}
where $\mathbf{1}$ is the function that is constantly one, and $\rho\in \mathbb{R}$ can be explicitly evaluated to be
\begin{align*}
\rho = &\left \langle \begin{bmatrix}
        k_{x_0}\\
        -k_{x_0}
    \end{bmatrix},  \begin{bmatrix}
    (I - \mathcal{A}_{v,\nu} \mathcal{A}_{u,\mu})^{-1} \mathcal{A}_{\sqrt{v},\gamma^2} \mathbf{1} - (I - \mathcal{A}_{v,\nu} \mathcal{A}_{u,\mu})^{-1}  \mathcal{A}_{v,\nu} \mathcal{A}_{\sqrt{u},\gamma^1} \mathbf{1}\\
    (I - \mathcal{A}_{u,\mu} \mathcal{A}_{v,\nu})^{-1} \mathcal{A}_{\sqrt{u},\gamma^1} \mathbf{1} - (I - \mathcal{A}_{u,\mu} \mathcal{A}_{v,\nu})^{-1}  \mathcal{A}_{u,\mu} \mathcal{A}_{\sqrt{v},\gamma^2}\mathbf{1}
    \end{bmatrix}/2u(x_0)\right\rangle_{\mathcal{H}^2}\\
     &\quad\quad +\left \langle \begin{bmatrix}
        k_{x_0}\\
        -k_{x_0}
    \end{bmatrix}, \pi_D 
    \begin{bmatrix}
        (I - \pi_u) \mathcal{A}_{\sqrt{v}, \gamma^2} \mathbf{1}\\
        0
    \end{bmatrix}/2u(x_0)\right\rangle_{\mathcal{H}^2}.
\end{align*}
In the case where $\mu = \nu$, $(I - \pi_u) \mathcal{A}_{\sqrt{v}, \gamma^2} \mathbf{1} = 0$, eliminating a term from the expression.
\end{theorem}

\begin{proof}
Let $\delta$ be as in Lemma \ref{lem: close zero}. 
We begin by studying the restriction of $\tau$ to $B^\delta$, denoted by $\tau|_{B^\delta}$. 
We can decompose $\tau|_{B^\delta} = \phi \circ J$, and by the chain rule for Hadamard differentiability \cite[Lemma~3.9.3]{van1996wellner} and Lemmas \ref{lem: phi Had} and \ref{lem: J Had} we have $\tau|_{B^\delta} : B^\delta \to \mathcal{H}^2$ is Hadamard differentiable in $\ell^\infty(\mathcal{H}_1)^2$ at $(\mu,\nu)$ with differential $\dot{\phi}_{J(\mu,\nu)} \circ \dot{J}_{\mu,\nu}$. As $B^\delta$ is an open domain in $\ell^\infty(\mathcal{H}_1)^2 \cap \mathcal{P}(\mathcal{X})^2$, it follows that $\tau: \ell^\infty(\mathcal{H}_1)^2 \cap \mathcal{P}(\mathcal{X})^2 \to \mathcal{H}^2$ is Hadamard differentiable at $(\mu,\nu)$ with $\dot{\tau}_{\mu,\nu} = \dot{\phi}_{J(\mu,\nu)} \circ \dot{J}_{\mu,\nu}$. This can be expressed by the above, noting that $K\left(\frac{d\gamma}{g}\right) = \mathcal{A}_{\sqrt{g},\gamma}\mathbf{1}$ and applying Lemmas \ref{Frechet Inverse} and \ref{lem: J Had}.

We verify the final claim, that the second term vanishes in the single marginal case. In this case, $u=v$, $\mu=\nu$, so we combine them in our notation. It suffices to show that $\pi_u \mathcal{A}_{\sqrt{u}, \gamma^2} \mathbf{1} = \mathcal{A}_{\sqrt{u}, \gamma^2} \mathbf{1}$, or in other words, that $\langle \mathcal{A}_{\sqrt{u}, \gamma^2} \mathbf{1}, u\rangle_{\mathcal{H}} = 0.$ This is immediate as
\[
\langle \mathcal{A}_{\sqrt{u}, \gamma^2} \mathbf{1}, u\rangle_{\mathcal{H}} = \langle K(d\gamma^2/u), u\rangle_\mathcal{H} = \int (u/u)\, d\gamma^2 = \int \mathbf{1}\, d\gamma^2 = 0
\]
where we used Lemma \ref{kernel embedding} in the third equality.

\end{proof}

In the interest of studying the regularity of the Sinkhorn divergence, we now push these results to the corresponding potentials, where we recover a formula analogous to that of \citep{gonzalezsanz2022weak}. To make the equivalence more clear, recall the operators $\mathcal{A}, \mathcal{A}^*, \xi, \xi^*$ defined in Section \ref{sec: EOT potentials}. We establish the following basic relationships. Note that in the following we consider $C(\mathcal{X})$ to be equipped with $\|\cdot\|_\infty$ topology by default.

\begin{lemma}\label{conjugacy}
    For a function $g \in C(\mathcal{X})$, define the linear operator $\mathfrak{M}_g:C(\mathcal{X}) \to C(\mathcal{X})$ by multiplication, $\mathfrak{M}_g(f)(\cdot) = g(\cdot)f(\cdot)$. Then,
    \begin{align}
    \mathfrak{M}_{1/u}\mathcal{A}_{v,\nu} \mathfrak{M}_v &= \mathcal{A}^*,\quad
    \mathfrak{M}_{1/v}\mathcal{A}_{u,\mu} \mathfrak{M}_{u} = \mathcal{A}, \label{eq:conjugacy1} \\
    \mathfrak{M}_{1/u} \mathcal{A}_{\sqrt{v},\gamma} \mathbf{1}&= \xi \gamma  ,\quad
    \mathfrak{M}_{1/v} \mathcal{A}_{\sqrt{u},\gamma} \mathbf{1}= \xi^* \gamma  . \label{eq:conjugacy2} 
    \end{align}
    If $\gamma \ll \mu,\nu$, then we also have
    \begin{align}
    \xi^* \gamma = \mathcal{A}^* \frac{d\gamma}{d\nu},\quad\quad \xi \gamma = \mathcal{A} \frac{d\gamma}{d\mu}. \label{eq:conjugacy3} 
    \end{align}
    If $\mu = \nu$ then $\mathcal{A}$ is self-adjoint relative to the $L^2(\mu)$ inner-product.
\end{lemma}

\begin{proof}
    For \eqref{eq:conjugacy1}, observe, for any $f\in C(\mathcal{X})$,
    \[
     \mathfrak{M}_{1/u}\mathcal{A}_{v,\nu} (v f)(x)= \int k_\varepsilon(x,y) u^{-1}(x) v^{-1}(y) f(y) d\nu(y) = \mathcal{A}^* (f)(x).
    \]
    The second part of \eqref{eq:conjugacy1} follows by similar computations. For \eqref{eq:conjugacy2}, verifying again the first equation,
    \[
    \mathfrak{M}_{1/u} \mathcal{A}_{\sqrt{v},\gamma} (\mathbf{1})(x) = \int k_\varepsilon(x,y) u^{-1}(x) v^{-1}(y) d\gamma(y) = \xi(\gamma)(x).
    \]
    Equation~\ref{eq:conjugacy3} follows by observing
    \[
    \xi^* \gamma =\int k_\varepsilon(x,y) u^{-1}(x) v^{-1}(y) d\gamma(y) = \int k_\varepsilon(x,y) u^{-1}(x) v^{-1}(y) \frac{d\gamma}{d\nu}(y) d\nu(y) =\mathcal{A}^* \frac{d\gamma}{d\nu}.
    \]
    We now establish that, if $\mu=\nu$, then $\mathcal{A}$ is self-adjoint relative to the $L^2(\mu)$ inner product. By Lemma \ref{EOT Operator},
    \begin{align*}
        \langle f, \mathcal{A}^* g\rangle_{L^2(\mu)} &= \int f \mathcal{A}^* g d\mu = \int f(x)g(y) d\pi_{\mu,\nu}^\varepsilon(x,y) = \int g \mathcal{A} f d\nu = \langle \mathcal{A} f, g\rangle_{L^2(\nu)},
    \end{align*}
    that is, $\mathcal{A}$ and $\mathcal{A}^*$ are adjoint (as suggested by the notation).
    Hence, when $\mu=\nu$, and thus $\mathcal{A} = \mathcal{A}^*,$ $\langle f, \mathcal{A} g\rangle_{L^2(\mu)} = \langle \mathcal{A} f, g\rangle_{L^2(\mu)}$, proving that $\mathcal{A}$ is self-adjoint.    
\end{proof}

\begin{lemma}\label{lem:log diff}
    If $r\leq r^*$, then $\Upsilon : B_r \to C(\mathcal{X}) \times C(\mathcal{X})$ defined so that
    \[
    \Upsilon(f,g) = (\log\,\circ\, f\,,\, \log\,\circ\, g)
    \]
    is Hadamard differentiable and $\dot{\Upsilon}_{(f,g)}[h^1,h^2] = \left(\frac{h^1}{f}, \frac{h^2}{g}\right)$.
\end{lemma}

\begin{proof}
Let $f+th_t \in B_r$ for all $t\in (0,1]$ and $\lim_{t\to 0}h_t= (h^1,h^2)$. To show Hadamard differentiability, we verify
\begin{align}
    \lim_{t\to 0} \frac{\Upsilon((u,v) + t h_t) }{t}=\left(\frac{h^1}{f}, \frac{h^2}{g}\right). \label{eq:Upsilon}
\end{align}
This will then show that $\Upsilon$ is Hadamard differentiable with $\dot{\Upsilon}_{(f,g)}:\mathcal{H}^2 \to C(\mathcal{X}) \times C(\mathcal{X})$ as specified in the lemma statement.

It suffices to argue \eqref{eq:Upsilon} coordinate-wise. Focusing on the first coordinate without loss of generality, we aim to show
    \[
    \lim_{t\to 0} \frac{\log(f + t h_t^1) - \log(f)}{t} = \frac{h^1}{f}.
    \]
    For ease of notation, we let $h = h^1$ and $h_t = h^1_t.$
    By Taylor's remainder theorem, for each $t\in (0,1]$ and $x\in\mathcal{X}$ there exists $c_t(x)$ between $f(x)$ and $f(x) + th_t(x)$ such that
    \[
    \log[f(x)+th_t(x)] - \log f(x) - th_t(x)/f(x) = -t^2 h_t(x)^2/c_t(x)^2.
    \]
    Since $f$ and $f+th_t$ belong to $B_r$, $c_t(x)$ is bounded below by $m>0$ (Lemma~\ref{lem:fgBdd}) for all $x,t$. Hence,
    \[
    \lim_{t\to 0} \sup_{x\in \mathcal{X}}
     \frac{|\log(f(x)+th_t(x)) - \log f(x) - th_t(x)/f(x)|}{t} \leq \sup_{x\in\mathcal{X}} th_t(x)^2/c_t(x)^2 \to 0.
    \]
    Thus
    \[
    \lim_{t\to 0} \frac{\log(f + t h_t) - \log(f)}{t} = \lim_{t\to 0} h_t/f = h/f
    \]
    as $h_t \to h$ uniformly by Lemma \ref{cts eval}.
\end{proof}

Our analysis culminates in Corollary \ref{EOT Potentials}, which is reminiscent of Theorem 2.2 in \cite{gonzalezsanz2022weak} and Theorem 4 in \cite{harchaoui2022asymptoticsdiscreteschrodingerbridges}, although there are small differences. One discrepancy vanishes in the single marginal case $\mu = \nu$, and the other encompasses an additive factor of little consequence to the analysis of the Sinkhorn divergence. We suspect these differences can be attributed to the different constraints imposed to uniquely identify the entropic potentials.

\begin{corollary}\label{EOT Potentials}
    $\eta: \ell^\infty(\mathcal{H}_1)\times \ell^\infty(\mathcal{H}_1) \to C(\mathcal{X})\times C(\mathcal{X})$, $\eta(\mu,\nu) := (\phi_{\mu,\nu},\psi_{\mu,\nu})$ is Hadamard differentiable at $(\mu,\nu)$ with derivative
    \begin{align}
     \dot{\eta}_{\mu,\nu}[\gamma^1,\gamma^2]&=\begin{bmatrix}
     \varepsilon[(I - \mathcal{A}^* \mathcal{A})^{-1}  \mathcal{A}^* \xi \gamma^1 - (I - \mathcal{A}^* \mathcal{A})^{-1} \xi^* \gamma^2]\\
    \varepsilon [(I - \mathcal{A} \mathcal{A}^*)^{-1}  \mathcal{A} \xi^* \gamma^2 - (I - \mathcal{A} \mathcal{A}^*)^{-1} \xi \gamma^1]
    \end{bmatrix} \label{eq:EOT Potentials 1}\\
    &\quad \quad + 
    \varepsilon
    \begin{bmatrix}
        \frac{1}{u} & 0\\
        0 & \frac{1}{v}
    \end{bmatrix} \pi_D \begin{bmatrix}
        (I - \pi_u) K\left(\frac{d \gamma^2}{v}\right)\\
        0
    \end{bmatrix} + \varepsilon \rho (\mathbf{1}, -\mathbf{1}), \nonumber
    \end{align}
    with $\rho$ being defined in Theorem \ref{theo:tau had}.
    In the case where $\gamma^1\ll \nu$ and $\gamma^2 \ll \mu$, the first part of the equation can be alternately expressed as
    \begin{align}
     \begin{bmatrix}
    \varepsilon[ (I - \mathcal{A}^* \mathcal{A})^{-1}  \mathcal{A}^* \mathcal{A} \frac{d\gamma^1}{d\mu} - (I - \mathcal{A}^* \mathcal{A})^{-1} \mathcal{A}^* \frac{d\gamma^2}{d\nu}]\\
    \varepsilon [(I - \mathcal{A} \mathcal{A}^*)^{-1}  \mathcal{A} \mathcal{A}^*\frac{d\gamma^2}{d\nu} - (I - \mathcal{A} \mathcal{A}^*)^{-1} \mathcal{A} \frac{d\gamma^1}{d\mu}]
    \end{bmatrix}. \label{eq:EOT Potentials 2}
    \end{align}
    In the case where $\mu = \nu$ the second term in equation \ref{eq:EOT Potentials 1} vanishes.
\end{corollary}

\begin{proof}
    By definition,
    \[
    \eta(\mu,\nu) = -\varepsilon \Upsilon( u_{\mu,\nu},  v_{\mu,\nu}) = -\varepsilon \Upsilon \circ \tau(\mu,\nu).
    \]
    As $\Upsilon$ and $\tau$ are Hadamard differentiable (Lemma \ref{lem:log diff} and Theorem \ref{theo:tau had}) we can apply the chain rule for Hadamard differentiability \cite[Lemma~3.9.3]{van1996wellner} to show the derivative $ \dot{\eta}_{\mu,\nu}[\gamma^1,\gamma^2]$ equals $\varepsilon \mathfrak{M}_{(1/u, 1/v)} \dot{\tau}_{\mu,\nu}$. The first term in the resulting derivative can then be expressed as
    \begin{align*}
    &\begin{bmatrix}
    -\varepsilon \mathfrak{M}_{1/u}[(I - \mathcal{A}_{v,\nu} \mathcal{A}_{u,\mu})^{-1} \mathcal{A}_{\sqrt{v},\gamma^1} \mathbf{1} - (I - \mathcal{A}_{v,\nu} \mathcal{A}_{u,\mu})^{-1}  \mathcal{A}_{v,\nu} \mathcal{A}_{\sqrt{u},\gamma^2} \mathbf{1}]\\
    -\varepsilon \mathfrak{M}_{1/v}[(I - \mathcal{A}_{u,\mu} \mathcal{A}_{v,\nu})^{-1} \mathcal{A}_{\sqrt{u},\gamma^2} \mathbf{1} - (I - \mathcal{A}_{u,\mu} \mathcal{A}_{v,\nu})^{-1}  \mathcal{A}_{u,\mu} \mathcal{A}_{\sqrt{v},\gamma^1}\mathbf{1}]
    \end{bmatrix}\\
    &=\begin{bmatrix}
    -\varepsilon \mathfrak{M}_{1/u}(\mathfrak{M}_{u}\mathfrak{M}_{1/u} - (\mathfrak{M}_{u}\mathcal{A}^* \mathfrak{M}_{1/v})(\mathfrak{M}_{v} \mathcal{A} \mathfrak{M}_{1/u}))^{-1} [\mathfrak{M}_{u}\xi^* \gamma^2 - (\mathfrak{M}_{u}\mathcal{A}^* \mathfrak{M}_{1/v})\mathfrak{M}_{v} \xi \gamma^1]\\
    -\varepsilon \mathfrak{M}_{1/v}(\mathfrak{M}_{v}\mathfrak{M}_{1/v} - (\mathfrak{M}_{v}\mathcal{A} \mathfrak{M}_{1/u})(\mathfrak{M}_{u} \mathcal{A}^* \mathfrak{M}_{1/v}))^{-1} [\mathfrak{M}_{v}\xi \gamma^1 - (\mathfrak{M}_{v}\mathcal{A} \mathfrak{M}_{1/u})\mathfrak{M}_{u} \xi^* \gamma^2]
    \end{bmatrix}\\
    &= \begin{bmatrix}
    \varepsilon[ (I - \mathcal{A}^* \mathcal{A})^{-1}  \mathcal{A}^* \xi \gamma^1 - (I - \mathcal{A}^* \mathcal{A})^{-1} \xi^* \gamma^2]\\
    \varepsilon [(I - \mathcal{A} \mathcal{A}^*)^{-1}  \mathcal{A} \xi^* \gamma^2 -(I - \mathcal{A} \mathcal{A}^*)^{-1} \xi \gamma^1]
    \end{bmatrix}.
    \end{align*}
    The second equality above follows from various applications of Lemma \ref{conjugacy}, and the observation that, for $f>0$ uniformly, $\mathfrak{M}_{1/f} = \mathfrak{M}_f^{-1}$. This establishes \eqref{eq:EOT Potentials 1}. Four additional applications of Lemma \ref{conjugacy} yield \eqref{eq:EOT Potentials 2}.
\end{proof}

\subsection{Sinkhorn divergence}

We follow \citep{gonzalezsanz2022weak}, in particular their proof of Lemma 4.5 and its supporting arguments. We first argue an alternate local representation of the Sinkhorn divergence. Let $(\phi,\psi):=(\phi_{\mu,\nu}, \psi_{\mu,\nu})$. When multiple sets of potentials are considered we specify with subscripts.  

\begin{lemma}\label{Taylor Expansion}
    For any $f,g$ satisfying $\iint \exp\left( \frac{f(x) + g(y) - \|x-y\|^2}{\varepsilon} \right)\, d\mu(x) d\nu(y) = 1$, let $h(x,y) = f(x) + g(y) - \phi(x) - \psi(y)$. Then
    \begin{align}
    &\left|\frac{1}{\varepsilon}\EOT(\mu,\nu) - \iint \frac{f(x) + g(y)}{\varepsilon}\, d\mu(x) d\nu(y) - \int \frac{(\phi(x) + \psi(y) - (f(x)+ g(y)))^2}{2\varepsilon^2} d\pi_{\mu,\nu}^\varepsilon\right| \nonumber \\
    &\quad\leq \frac{1}{6} \|h\|_\infty^3 \exp(\|h\|_\infty). \label{eq:OTTaylor}
    \end{align}
\end{lemma}
\begin{proof}
This proof essentially follows by a Taylor expansion of the exponential function. To apply this expansion, we will show that the left-hand side of \eqref{eq:OTTaylor}, $\mathrm{LHS}$, satisfies
\begin{align}
    \mathrm{LHS}=\left|\iint \left[\exp\{h(x,y)\} - h(x,y)-h^2(x,y)/2-1\right] d\pi_{\mu,\nu}^\varepsilon(x,y)\right|.\label{eq:TaylorSetup}
\end{align}
 Once we have shown this, the proof will be complete since, applying a Taylor expansion pointwise to the integrand and then using H\"{o}lder's inequality yields
\begin{align*}
    &\mathrm{LHS}\le \frac{1}{6}\iint|h(x,y)|^3 \exp(|h(x,y)|)\, d\pi_{\mu,\nu}^\varepsilon\leq \frac{1}{6}\|h\|_\infty^3 \exp(\|h\|_\infty),
\end{align*}

We conclude the proof by showing that  \eqref{eq:TaylorSetup} holds. To see this, note that, by the condition on $f,g$ and the definition of $h$ and $\pi_{\mu,\nu}^\varepsilon$,
    \begin{align*}
    0&=\iint \left[\exp\left(\frac{f(x) + g(y) - \|x-y\|^2}{\varepsilon}\right)-\exp\left(\frac{\phi(x) + \psi(y) - \|x-y\|^2}{\varepsilon}\right)\right]  d\mu(x) d\nu(y) \\
    &= \iint \left[\exp\left\{h(x,y)\right\} - 1\right]\exp\left(\frac{\phi(x) + \psi(y) - \|x-y\|^2}{\varepsilon}\right) \, d\mu(x) d\nu(y)\\
        &= \iint\left[\exp\left\{h(x,y)\right\} - 1\right]   \, d\pi_{\mu,\nu}^\varepsilon(x,y).
    \end{align*}
     Combining this with the fact that $\frac{1}{\varepsilon}\EOT(\mu,\nu)=\iint \frac{\phi(x) + \psi(y)}{\varepsilon}\, d\mu(x)d\nu(y)$ shows that
    \begin{align*}
         &\left|\frac{1}{\varepsilon}\EOT(\mu,\nu) - \frac{1}{\varepsilon}\iint f(x) + g(y)\, d\mu(x) d\nu(y) - \frac{1}{2\varepsilon^2}\iint (\phi(x) + \psi(y) - (f(x)+ g(y)))^2 d\pi_{\mu,\nu}^\varepsilon\right|\\
         &\quad\quad=  \left|\iint\left[\exp\left(h(x,y)\right) - h(x,y){\varepsilon} - h^2(x,y)/2  - 1 \right]  \, d\pi_{\mu,\nu}^\varepsilon\right|,
    \end{align*}
    as desired.
\end{proof}

When studying Hadamard differentiability of the Sinkhorn divergence, we implicitly assume that all $\gamma_t$ perturbations are uniformly supported on a compact set $\mathcal{X}$.

\begin{lemma} \label{local expansion}
    Let $\alpha,\beta$ be probability distributions, and define
    \begin{align*}
    \OTL(\alpha,\beta) &:=  \iint [\phi_{\beta,\beta}(x)+ \psi_{\beta,\beta}(y)]\, d\alpha(x) d\beta(y) \\
    &\quad\quad +\frac{1}{2\varepsilon}\iint ([\phi_{\alpha,\beta} - \phi_{\beta,\beta}](x) + [\psi_{\alpha,\beta} -\psi_{\beta,\beta}](y))^2\, d\pi_{\mu,\nu}^\varepsilon(x,y),\\
    S^{\textnormal{loc}}(\alpha,\beta) &:= \frac{1}{2}[\OTL(\alpha,\beta) - \EOT(\beta,\beta)] + \frac{1}{2}[\OTL(\beta,\alpha) - \EOT(\alpha,\alpha)].
    \end{align*}
    For $\gamma^1_t \to \gamma^1$, $\gamma^2_t \to \gamma^2$ in $\ell^\infty(\mathcal{H}_1)$,
    \begin{align*}
    \lim_{t\to 0}& \frac{S(\mu + t\gamma^1_t, \mu + t\gamma^2_t) - S^{\textnormal{loc}}(\mu + t\gamma^1_t, \mu + t\gamma^2_t)}{t^2} = 0.
    \end{align*}
\end{lemma}

\begin{proof}
    We first focus on the difference $\OTL - \EOT$. For general $\alpha,\beta$, observe that $(\phi_{\alpha,\alpha}, \psi_{\alpha,\alpha})$ and $(\phi_{\beta,\beta}, \psi_{\beta,\beta})$ meet the criterion of Lemma \ref{Taylor Expansion} by Lemma \ref{Double Marginal}. Hence, for
    \[
     h_{\alpha,\beta} := \phi_{\beta,\beta}(x) - \phi_{\alpha,\beta}(x) + \psi_{\beta,\beta}(y)  - \psi_{\alpha,\beta}(y),
    \]
    we have
    \begin{align*}
    \left|\frac{\EOT(\alpha,\beta) - \OTL(\alpha,\beta)}{\varepsilon}\right| \leq  \frac{1}{6} \|h_{\alpha,\beta}\|_\infty^3\exp\left( \|h_{\alpha,\beta}\|_\infty \right).
    \end{align*}
    Let $\alpha_t:=\mu + t\gamma_t^1$, $\beta_t:=\mu+t\gamma_t^2$. As the potentials $(\phi_{\alpha_t, \beta_t}, \psi_{\alpha_t,\beta_t})$ converge uniformly to $(\phi_{\mu,\mu}, \psi_{\mu,\mu})$, we get the following convergence,
    \begin{align*}
          \frac{1}{t^2}|\EOT(\alpha_t, \beta_t) -\OTL(\alpha_t, \beta_t)| &= O\left(\frac{1}{t^2} [\|\phi_{\alpha_t, \beta_t} - \phi_{\beta_t, \beta_t}\|_\infty^3 + \| \psi_{\alpha_t, \beta_t} -\psi_{\beta_t, \beta_t}\|_\infty^3]\right)\\
          &= O\left(\frac{1}{t^2} [\|\phi_{\alpha_t, \beta_t} - \phi_{\mu, \mu}\|_\infty^3 + \| \psi_{\alpha_t, \beta_t} -\psi_{\mu, \mu}\|_\infty^3]\right)\\
          &\quad\quad +O\left(\frac{1}{t^2} [\|\phi_{\mu, \mu} - \phi_{\beta_t, \beta_t}\|_\infty^3 + \| \psi_{\mu, \mu} -\psi_{\beta_t, \beta_t}\|_\infty^3]\right)\\
          &=O\left(\left\|\frac{\phi_{\alpha_t, \beta_t} - \phi_{\mu, \mu}}{t^{2/3}}\right\|_\infty^3 + \left\| \frac{\psi_{\alpha_t, \beta_t} -\psi_{\mu, \mu}}{t^{2/3}}\right\|_\infty^3\right)\\
          &\quad\quad +O\left(\left\|\frac{\phi_{\mu, \mu} - \phi_{\beta_t, \beta_t}}{t^{2/3}}\right\|_\infty^3 + \left\| \frac{\psi_{\mu, \mu} -\psi_{\beta_t, \beta_t}}{t^{2/3}}\right\|_\infty^3\right)\\
          &= o(1),
    \end{align*}
    with the analysis of $|\EOT(\alpha_t,\beta_t) - \OTL(\beta_t,\alpha_t)|$ following similarly. Thus,
    \begin{gather*}
        \lim_{t\to 0} \frac{|S(\mu + t\gamma^1_t, \mu + t\gamma^2_t) - S^{\text{loc}}(\mu + t\gamma^1_t, \mu + t\gamma^2_t)|}{t^2}\\
        \leq\lim_{t\to 0} \frac{|\EOT(\alpha_t, \beta_t) - \OTL(\alpha_t,\beta_t)|}{2t^2} + \frac{|\EOT(\alpha_t, \beta_t) - \OTL(\beta_t,\alpha_t)|}{2t^2} = 0.
    \end{gather*}
    
    \end{proof}

This lemma shows that $S^{\text{loc}}$ agrees with $S$ up to second order, allowing us to derive the second order term by appealing to this simpler functional. As a preliminary, we first establish some identities that will be useful in the sequel.

\begin{lemma}\label{potential limits}
Let $\alpha_t:=\mu + t\gamma_t^1$, $\beta_t:=\mu+t\gamma_t^2$, for $\gamma_t^i \to \gamma^i$ in $\ell^\infty(\mathcal{H}_1)$, and let $\rho_{i, j}$ be the constants as computed in Theorem \ref{theo:tau had} setting $\gamma_1 = \gamma_i, \gamma_2 = \gamma_j$. Note that $\rho_{i,i} = 0$ for $i = 1,2$. We have the uniform convergence
\begin{align*}
    \lim_{t\to 0}\frac{\phi_{\alpha_t,\alpha_t} - \phi_{\beta_t,\beta_t}}{t} &= \varepsilon (I - \mathcal{A}^2)^{-1} (I - \mathcal{A})\xi [\gamma^2 - \gamma^1], \\
    \lim_{t\to 0}\frac{\phi_{\alpha_t,\beta_t} - \phi_{\alpha_t,\alpha_t}}{t} &= \varepsilon (I -  \mathcal{A}^2)^{-1} \xi [\gamma^1 - \gamma^2] + \varepsilon \rho_{1,2} \mathbf{1}, \\
    \lim_{t\to 0}\frac{\phi_{\alpha_t,\beta_t} - \phi_{\beta_t,\beta_t}}{t} &= \varepsilon (I -  \mathcal{A}^2)^{-1} \mathcal{A} \xi [\gamma^1 - \gamma^2] + \varepsilon \rho_{1,2} \mathbf{1}, \\
    \lim_{t\to 0}\frac{\psi_{\alpha_t,\beta_t} - \psi_{\alpha_t,\alpha_t}}{t} &= \varepsilon (I -  \mathcal{A}^2)^{-1} \mathcal{A} \xi [\gamma^2 - \gamma^1] - \varepsilon \rho_{1,2} \mathbf{1},  \\
    \lim_{t\to 0}\frac{\psi_{\alpha_t,\beta_t} - \psi_{\beta_t,\beta_t}}{t} &= \varepsilon (I -  \mathcal{A}^2)^{-1} \xi [\gamma^2 - \gamma^1]  - \varepsilon \rho_{1,2} \mathbf{1}.
\end{align*}
\end{lemma}

\begin{proof}
    It follows immediately from Corollary \ref{EOT Potentials} that 
    \begin{align*}
         \lim_{t\to 0}\frac{\phi_{\alpha_t,\alpha_t} - \phi_{\mu,\mu}}{t} &= \lim_{t\to 0}\frac{\psi_{\alpha_t,\alpha_t} - \psi_{\mu,\mu}}{t} = -\varepsilon (I - \mathcal{A}^2)^{-1} (I - \mathcal{A})\xi \gamma^1 \\
         \lim_{t\to 0}\frac{\phi_{\beta_t,\beta_t} - \phi_{\mu,\mu}}{t} &= \lim_{t\to 0}\frac{\psi_{\beta_t,\beta_t} - \psi_{\mu,\mu}}{t} = -\varepsilon (I - \mathcal{A}^2)^{-1} (I - \mathcal{A})\xi \gamma^2 \\
         \lim_{t\to 0}\frac{\phi_{\alpha_t,\beta_t} - \phi_{\mu,\mu}}{t} &= -\varepsilon[ (I -  \mathcal{A}^2)^{-1} \xi^* \gamma^2 - (I -  \mathcal{A}^2)^{-1}  \mathcal{A} \xi \gamma^1] + \varepsilon \rho_{1,2} \mathbf{1}\\
        \lim_{t\to 0}\frac{\psi_{\alpha_t,\beta_t} - \psi_{\mu,\mu}}{t} &= -\varepsilon [(I - \mathcal{A}^2)^{-1} \xi \gamma^1 - (I - \mathcal{A}^2)^{-1}  \mathcal{A} \xi \gamma^2 ] - \varepsilon \rho_{1,2} \mathbf{1}. 
    \end{align*}
    The final result follows taking appropriate differences of the above expressions.
\end{proof}

\begin{lemma}\label{lem: limit inside}
    Let $f_t \to f$ in $C(\mathcal{X})$ and $\gamma_t \to \gamma$ weakly. Then $\int f_t\, d\gamma_t \to \int f\, d\gamma$.
\end{lemma}

\begin{proof}
    We take the limits as $t\to 0$ without loss of generality. It suffices to verify the claim for an arbitrary sequence $t_i \to 0$. We first verify that $\mathcal{F}:=\{f_{t_i}\}_{i=1}^\infty \cup \{f\}$ is compact. Take an open cover of this space, and observe all but finitely many functions are within a neighborhood of $f$, hence a finite subcover can be constructed, verifying the claim.
    Now, by assumption, we have that $\int g\, d\gamma_t \to \int g\, d\gamma$ for all continuous $g$ as $\mathcal{X}$ is compact. Note that $C(\mathcal{X})$ is barrelled \citep[Chapter 33]{treves1967topological}, and as $\mathcal{P}(\mathcal{X})$ is a bounded subset of its dual, it is equicontinuous \citep[Theorem 33.2]{treves1967topological}. By \cite[Proposition 32.5]{treves1967topological}, it follows that the convergence is uniform over $\mathcal{F}$. In particular,
    \begin{align*}
        \lim_{i\to \infty} \int f_{t_i}\, d\gamma_{t_i}&= \lim_{i\to \infty}\int f_{t_i}\, d(\gamma_{t_i}-\gamma) + \int (f_{t_i}-f)\, d\gamma + \int f\, d\gamma = \int f\, d\gamma,
    \end{align*}
    where the first term vanishes by uniform covergence over $\mathcal{F}$, and the second by the uniform convergence of the functions $f_{t_i} \to f$.
\end{proof}

\begin{theorem} \label{thm:Sink2}
    The Sinkhorn divergence $S$ is second order Hadamard differentiable at $(\mu,\mu) \in \mathcal{P}(\mathcal{X}) \times \mathcal{P}(\mathcal{X})$ and
    \begin{align*}
    \ddot S(\mu,\mu)[\gamma^1,\gamma^2] &=  \frac{\varepsilon}{2} \int \xi[\gamma^2 - \gamma^1](x) \, (I-\mathcal{A}^2)^{-1} \xi[\gamma^2 - \gamma^1](x)\, d\mu(x) \\
    &\quad +\frac{\varepsilon}{2} \int  (I - \mathcal{A}^2)^{-1} (I - \mathcal{A})\xi [\gamma^2 - \gamma^1](x)\, d(\gamma^2 - \gamma^1)(x),
    \end{align*}
   If $\gamma^1,\gamma^2 \ll \mu$, this rewrites as
     \begin{align}
    \ddot S(\mu,\mu)[\gamma^1,\gamma^2] &=  \frac{1}{2} \int  \varepsilon(I - \mathcal{A}^2)^{-1} \xi [\gamma^2 - \gamma^1](x)\, d(\gamma^2 - \gamma^1)(x) \nonumber \\
    &= \frac{1}{2} \left\langle \left[\frac{d(\gamma^2 - \gamma^1)}{d\mu}\right], \varepsilon(I - \mathcal{A}^2)^{-1} \mathcal{A} \left[\frac{d(\gamma^2 - \gamma^1)}{d\mu}\right]\right\rangle_{L^2(\mu)}. \label{eq:Sink2SecondRep}
    \end{align}
\end{theorem}
\begin{proof}
    Applying Lemma \ref{local expansion}, we focus on $S^{\text{loc}}$, which we expand below. We apply Proposition \ref{EOT rep} to write it solely in terms of the potentials, evaluating it at dummy distributions $\alpha,\beta$,
    \begin{align}\label{eq: sloc}
     \nonumber S^{\text{loc}}(\alpha,\beta)&= \frac{1}{2}\iint [\phi_{\alpha,\alpha}(x) + \psi_{\alpha,\alpha}(y)]\, d\alpha(x) d(\beta-\alpha)(y) \\
     \nonumber &\quad+ \frac{1}{2}\iint [\phi_{\beta,\beta}(x) + \psi_{\beta,\beta}(y)]\, d(\alpha-\beta)(x) d\beta(y)\\ 
     \nonumber &\quad+ \frac{1}{4\varepsilon}\int[(\phi_{\alpha,\beta} - \phi_{\beta,\beta})(x)+ (\psi_{\alpha,\beta} - \psi_{\beta,\beta})(y)]^2 d\pi_{\alpha,\beta}^\varepsilon(x,y)  \\
     \nonumber &\quad+\frac{1}{4\varepsilon}\int[(\phi_{\alpha,\beta} - \phi_{\alpha,\alpha})(x) + (\psi_{\alpha,\beta} - \psi_{\alpha,\alpha})(y)]^2 d\pi_{\alpha,\beta}^\varepsilon(x,y)\\
     &=: f_1(\alpha,\beta) + f_2(\alpha,\beta) + g_1(\alpha,\beta) + g_2(\alpha,\beta). 
    \end{align}
    By Lemma \ref{Double Marginal}, $\phi_{\alpha,\alpha} = \psi_{\alpha,\alpha}$,
    \begin{align*}
        \iint [\phi_{\alpha,\alpha}(x) + \psi_{\alpha,\alpha}(y)]\, d\alpha(x) d(\beta-\alpha)(y)&= \int \psi_{\alpha,\alpha}(y) d(\beta-\alpha)(y) = \int \phi_{\alpha,\alpha}(y) d(\beta-\alpha)(y).
    \end{align*}
    A similar argument for $f_2(\alpha,\beta)$ shows
    \[
    f_1(\alpha,\beta) + f_2(\alpha,\beta) = \frac{1}{2} \int (\phi_{\alpha,\alpha} - \phi_{\beta,\beta}) d(\beta-\alpha).
    \]
    Substituting $\alpha_t:=\mu + t\gamma_t^1$, $\beta_t:=\mu+t\gamma_t^2$, for $\gamma_t^i \to \gamma^i$ in $\ell^\infty(\mathcal{H}_1)$, for $\alpha,\beta$ respectively, we can apply Lemma \ref{potential limits} to get
    \begin{align*}
        \lim_{t\to 0} \frac{f_1(\alpha_t,\beta_t) + f_2(\alpha_t,\beta_t)}{t^2} &=\lim_{t\to 0} \frac{1}{2} \int \frac{\phi_{\alpha_t,\alpha_t} - \phi_{\beta_t,\beta_t}}{t}\, d(\gamma^2_t - \gamma^1_t)\\
        &= \frac{\varepsilon}{2} \int (I - \mathcal{A}^2)^{-1} (I - \mathcal{A})\xi [\gamma^2 - \gamma^1]\, d(\gamma^2 - \gamma^1),
    \end{align*}

    with the final limit being verified in Lemma \ref{lem: limit inside}.
    
    Moving on to the $g_1$ and $g_2$ terms, observe
    \begin{align*}
        \frac{g_1(\alpha_t,\beta_t) + g_2(\alpha_t,\beta_t)}{t^2}
        &=\frac{1}{4\varepsilon}\int \Bigg(\left[\frac{(\phi_{\alpha_t,\beta_t} - \phi_{\beta_t,\beta_t})(x) + (\psi_{\alpha_t,\beta_t} - \psi_{\beta_t,\beta_t})(y)}{t}\right]^2 \\
        &\hspace{5em}+\left[\frac{(\phi_{\alpha_t,\beta_t} - \phi_{\alpha_t,\alpha_t})(x) + (\psi_{\alpha_t,\beta_t} - \psi_{\alpha_t,\alpha_t})(y)}{t}\right]^2\Bigg) d\pi_{\alpha_t,\beta_t}^\varepsilon(x,y)\\
        &=\frac{1}{4\varepsilon}\int \Bigg(\left[\frac{(\phi_{\alpha_t,\beta_t} - \phi_{\beta_t,\beta_t})(x) + (\psi_{\alpha_t,\beta_t} - \psi_{\beta_t,\beta_t})(y)}{t}\right]^2 \\
        &\hspace{5em}+\left[\frac{(\phi_{\alpha_t,\beta_t} - \phi_{\alpha_t,\alpha_t})(x) + (\psi_{\alpha_t,\beta_t} - \psi_{\alpha_t,\alpha_t})(y)}{t}\right]^2\Bigg)\\
        &\hspace{7.5em} \cdot \exp\left(\frac{\phi_{\alpha_t,\beta_t}(x) - \psi_{\alpha_t, \beta_t}(y) - \|x-y\|^2}{\varepsilon}\right)d\mu(x) d\mu(y).
    \end{align*}
    Observe that, as $t\to 0$, the integrand uniformly converges, hence we can pass to the limit. Considering $g_1$ first, we see
    \begin{align*}
        &\lim_{t \to 0} \frac{g_1(\alpha_t,\beta_t)}{t^2} \\
        &= \frac{1}{4\varepsilon} \int \left\{\varepsilon (I -  \mathcal{A}^2)^{-1} \mathcal{A} \xi [\gamma^1 - \gamma^2](x) + \varepsilon (I -  \mathcal{A}^2)^{-1} \xi [\gamma^2 - \gamma^1](y) \right\}^2 d\pi_{\mu,\mu}^\varepsilon(x,y)\\
        &=\frac{\varepsilon}{4} \int \{(I -  \mathcal{A}^2)^{-1} \mathcal{A} \xi [\gamma^1 - \gamma^2](x)\}^2 d\mu(x) + \frac{\varepsilon}{4} \int \{(I -  \mathcal{A}^2)^{-1}  \xi [\gamma^2 - \gamma^1](x)\}^2 d\mu(x)\\
        &\quad\quad+ \frac{\varepsilon}{2} \int \{(I -  \mathcal{A}^2)^{-1} \mathcal{A} \xi [\gamma^1 - \gamma^2](x)\} \{(I -  \mathcal{A}^2)^{-1} \xi [\gamma^2 - \gamma^1](y)\} d\pi_{\mu,\mu}^\varepsilon(x,y)\\
         &=\frac{\varepsilon}{4} \| (I -  \mathcal{A}^2)^{-1} \mathcal{A} \xi [\gamma^1 - \gamma^2]\|_{L^2(\mu)}^2 + \frac{\varepsilon}{4} \|(I -  \mathcal{A}^2)^{-1}  \xi [\gamma^2 - \gamma^1]\|_{L^2(\mu)}^2\\
        &\quad\quad+ \frac{\varepsilon}{2} \int \{(I -  \mathcal{A}^2)^{-1} \mathcal{A} \xi [\gamma^1 - \gamma^2](x)\} \{(I -  \mathcal{A}^2)^{-1} \xi [\gamma^2 - \gamma^1](y)\} d\pi_{\mu,\mu}^\varepsilon(x,y)\\
\end{align*}
Using that $\mathcal{A}$ is self-adjoint when $\mu=\nu$ by Lemma \ref{conjugacy} and $\int f(x) g(y) d\pi_{\mu,\mu}^\varepsilon(x,y) = \int f \mathcal{A} g\, d\mu$ by Lemma \ref{EOT Operator}, the above shows
\begin{align*}
        \lim_{t \to 0} \frac{g_1(\alpha_t,\beta_t)}{t^2}&=\frac{\varepsilon}{4} \langle \xi[\gamma^2 - \gamma^1], (I-\mathcal{A}^2)^{-2} \mathcal{A}^2 \xi[\gamma^2 - \gamma^1]\rangle_{L^2(\mu)}\\
        &\quad + \frac{\varepsilon}{4} \langle \xi[\gamma^2 - \gamma^1], (I-\mathcal{A}^2)^{-2} \xi[\gamma^2 - \gamma^1] \rangle_{L^2(\mu)}\\
          &\quad - \frac{\varepsilon}{2} \langle \xi[\gamma^2 - \gamma^1] , (I - \mathcal{A}^2)^{-2} \mathcal{A}^2 \xi[\gamma^2 - \gamma^1] \rangle_{L^2(\mu)}\\
         &= \frac{\varepsilon}{4} \langle \xi[\gamma^2 - \gamma^1], (I-\mathcal{A}^2)^{-2} (I - \mathcal{A}^2) \xi[\gamma^2 - \gamma^1] \rangle_{L^2(\mu)}\\
         &= \frac{\varepsilon}{4} \langle \xi[\gamma^2 - \gamma^1], (I-\mathcal{A}^2)^{-1} \xi[\gamma^2 - \gamma^1]\rangle_{L^2(\mu)}.
    \end{align*}
    A similar argument yields
    \[
    \lim_{t \to 0} \frac{g_2(\alpha_t,\beta_t)}{t^2} = \frac{\varepsilon}{4} \langle \xi[\gamma^2 - \gamma^1], (I-\mathcal{A}^2)^{-1} \xi[\gamma^2 - \gamma^1] \rangle_{L^2(\mu)}.
    \]
    Returning to \eqref{eq: sloc}, we have shown that
    \begin{align*}
        \lim_{t\to 0} \frac{S^{\text{loc}}(\alpha_t,\beta_t)}{t^2}&= \frac{\varepsilon}{2} \int (I - \mathcal{A}^2)^{-1} (I - \mathcal{A})\xi [\gamma^2 - \gamma^1]\, d(\gamma^2 - \gamma^1)\\
        &\quad\quad +\frac{\varepsilon}{2} \langle \xi[\gamma^2 - \gamma^1], (I-\mathcal{A}^2)^{-1} \xi[\gamma^2 - \gamma^1]\rangle_{L^2(\mu)}.
    \end{align*}
    Applying Lemma \ref{conjugacy}, and then the self-adjointness of $\mathcal{A}$ shows the second term on the right rewrites as
    \begin{align*}
        &\frac{\varepsilon}{2} \langle \xi[\gamma^2 - \gamma^1], (I-\mathcal{A}^2)^{-1} \xi[\gamma^2 - \gamma^1]\rangle_{L^2(\mu)} \\
        &= \frac{\varepsilon}{2} \langle \mathcal{A}\left[\frac{d(\gamma^2 - \gamma^1)}{d\mu}\right], (I-\mathcal{A}^2)^{-1} \xi[\gamma^2 - \gamma^1] \rangle_{L^2(\mu)}\\
        &= \frac{\varepsilon}{2} \left\langle \left[\frac{d(\gamma^2 - \gamma^1)}{d\mu}\right], (I-\mathcal{A}^2)^{-1} \mathcal{A} \xi[\gamma^2 - \gamma^1] \right\rangle_{L^2(\mu)}\\
        &= \frac{\varepsilon}{2} \int (I-\mathcal{A}^2)^{-1} \mathcal{A} \xi[\gamma^2 - \gamma^1] d(\gamma^2 - \gamma^1).
    \end{align*}
    Combining the preceding two displays shows
    \begin{align*}
        \lim_{t\to 0} \frac{S^{\text{loc}}(\alpha_t,\beta_t)}{t^2}&= \frac{1}{2} \int \varepsilon(I - \mathcal{A}^2)^{-1}\xi [\gamma^2 - \gamma^1]\, d(\gamma^2 - \gamma^1).
    \end{align*}
    The second equality in \eqref{eq:Sink2SecondRep} similarly follow from Lemma \ref{conjugacy}.
\end{proof}

\subsection{Sinkhorn Kernel Estimation}\label{sec: sink emp kernel}

\begin{lemma}\label{lem: discretization}
    For $\hat{\mathcal{A}}_n f(\cdot) := \int \xi_{\mathbb{P}, \mathbb{P}}(\,\cdot\,,y) f(y)\, d\mathbb{P}_n(y)$, $\mathcal{F}$ the RKHS with kernel $\xi_{\mathbb{P}, \mathbb{P}}(x,y)$, $\|\mathcal{A} - \hat{\mathcal{A}}_n\|_\mathcal{F} = O_p(n^{-1/2})$.
\end{lemma}

\begin{proof}
    By Lemma \ref{kernel embedding},
    \begin{align*}
    \|(\mathcal{A} - \hat{\mathcal{A}}_n) f\|_\mathcal{F}^2 &=  \int \xi_{\mathbb{P}, \mathbb{P}}(x,y) f(x)f(y)\, d(\mathbb{P} - \mathbb{P}_n)(x) d(\mathbb{P} - \mathbb{P}_n)(y)\\
    &= \|\mathfrak{M}_f (\mathbb{P} - \mathbb{P}_n) \|_{\ell^\infty(\mathcal{F}_1)}^2.
    \end{align*}
    By Lemma \ref{uniform kernel equiv} and \ref{lem: rkhs donsker}, there exists a constant $C_1>0$ such that
    \[
    \sup_{\|f\|_\mathcal{F} \leq 1} \|(\mathcal{A} - \hat{\mathcal{A}}_n) f\|_\mathcal{F}^2 \leq C_1 \|\mathbb{P} - \mathbb{P}_n \|_{\ell^\infty(\mathcal{F}_1)}^2 = O_p(n^{-1}).
    \]
\end{proof}

\begin{lemma}\label{lem: xi representer}
    For $\mathcal{F}$ the RKHs with kernel $\xi_{\mathbb{P}, \mathbb{P}}$, $\|\xi\|_{\ell^\infty(\mathcal{F}_1) \to \mathcal{F}} =1$.
\end{lemma}

\begin{proof}
    As $\xi_{\mathbb{P}, \mathbb{P}}$ is the kernel, $\xi$ is the kernel mean embedding, hence $\|\xi \mu\|_{\mathcal{F}} = \|\mu\|_{\ell^\infty(\mathcal{F}_1)}$. Thus, the claim is immediate.
\end{proof}

\begin{lemma}\label{lem: discrete quad}
    Let 
    \begin{align*}
    \ddot{S}_{n,\xi}(\gamma) := \ddot{S}_{n,\xi}(\mathbb{P},\mathbb{P})[\gamma] &=  \frac{\varepsilon}{2} \int \xi[\gamma](x) \, (I-\hat{\mathcal{A}}_n^2)^{-1} \xi[\gamma](x)\, d\mathbb{P}_n(x) \\
    &\quad +\frac{\varepsilon}{2} \int  (I - \hat{\mathcal{A}}_n^2)^{-1} (I - \hat{\mathcal{A}}_n)\xi [\gamma](x)\, d(\gamma)(x),
    \end{align*}
    and let $\mathcal{F}$ denote the RKHS with kernel $\xi_{\mathbb{P}, \mathbb{P}}(x,y)$. If $\|P_m - \mathbb{P}_n\|_{\ell^\infty(\mathcal{F}_1)} = o_p(n^{-1/2})$ then,
    \begin{align}\label{eq: sinkhorn error bound 1}
    |\ddot{S}_{n,\xi}(P_m - \mathbb{P}) - \ddot{S}(P_m - \mathbb{P}_n)| =  \|P_m - \mathbb{P}_n\|_{\ell^\infty(\mathcal{F}_1)}^2 \,O_p(n^{-1/2}) = o_p(n^{-3/2}).
    \end{align}
\end{lemma}
\begin{proof}
 By the triangle inequality, $ |\ddot{S}_{n,\xi}(P_m - \mathbb{P}_n) - \ddot{S}(P_m - \mathbb{P}_n)|$ is no more than
 \begin{align*}
 &\left|\int \xi[P_m - \mathbb{P}_n](x) [(I-\mathcal{A}^2)^{-1}-(I-\hat{\mathcal{A}}_n^2)^{-1}] \xi[P_m - \mathbb{P}_n](x)\, d\mathbb{P}(x)\right|\\
 &\quad\quad + \left|\int \xi[P_m - \mathbb{P}_n](x) (I-\mathcal{A}^2)^{-1} \xi[P_m - \mathbb{P}_n](x)\, d(\mathbb{P} - \mathbb{P}_n)(x) \right|\\
 &\quad\quad +\left|\int  [(I-\mathcal{A}^2)^{-1}(I - \mathcal{A})-(I-\hat{\mathcal{A}}_n^2)^{-1} (I - \hat{\mathcal{A}}_n)]\xi [P_m - \mathbb{P}_n](x)\, d(P_m - \mathbb{P}_n)(x)\right|.
 \end{align*}
 We bound the three terms above separately. We focus on proving the final inequality of \eqref{eq: sinkhorn error bound 1}, with the middle inequality being an immediate consequence of our approach. For the first,
 \begin{align*}
 &\left|\int \xi[P_m - \mathbb{P}_n](x) [(I-\mathcal{A}^2)^{-1}-(I-\hat{\mathcal{A}}_n^2)^{-1}] \xi[P_m - \mathbb{P}_n](x)\, d\mathbb{P}(x)\right|\\
  &\quad\quad\leq \|\xi[P_m - \mathbb{P}_n] [(I-\mathcal{A}^2)^{-1}-(I-\hat{\mathcal{A}}_n^2)^{-1}] \xi[P_m - \mathbb{P}_n]\|_\infty\\
 &\quad\quad\leq \|\xi[P_m - \mathbb{P}_n]\|_\infty\| [(I-\mathcal{A}^2)^{-1}-(I-\hat{\mathcal{A}}_n^2)^{-1}] \xi[P_m - \mathbb{P}_n]\|_\infty\\
 &\quad\quad\leq C \|\xi[P_m - \mathbb{P}_n]\|_\mathcal{F}\| [(I-\mathcal{A}^2)^{-1}-(I-\hat{\mathcal{A}}_n^2)^{-1}] \xi[P_m - \mathbb{P}_n]\|_\mathcal{F}\\
 &\quad\quad\leq C \|P_m - \mathbb{P}_n\|_{\ell^\infty(\mathcal{F}_1)}^2 \|[(I-\mathcal{A}^2)^{-1}-(I-\hat{\mathcal{A}}_n^2)^{-1}]\|_\mathcal{F}\\
 &\quad\quad= o_p(n^{-1}) \|(I-\mathcal{A}^2)^{-1}-(I-\hat{\mathcal{A}}_n^2)^{-1}\|_\mathcal{F},
 \end{align*}
  where the third inequality is by the reproducing property that the kernel is bounded and the fourth is by the operator norm and Lemma \ref{lem: xi representer}. By Lemma \ref{lem: discretization} and \citep[Lemma 8.6]{haase2018lectures} $\|[(I-\mathcal{A}^2)^{-1}-(I-\hat{\mathcal{A}}_n^2)^{-1}]\|_\mathcal{F} = O_p(n^{-1/2})$, hence this term decays at the desired rate.

 For the second, by definition of the $\ell^\infty(\mathcal{F}_1)$ norm,
 \begin{align*}
     &\left|\int \xi[P_m - \mathbb{P}_n](x) (I-\mathcal{A}^2)^{-1} \xi[P_m - \mathbb{P}_n](x)\, d(\mathbb{P} - \mathbb{P}_n)(x) \right|\\
     &\quad\quad \leq \| (I-\mathcal{A}^2)^{-1} \xi[P_m - \mathbb{P}_n]\|_\mathcal{F}\|\mathfrak{M}_{\xi[P_m - \mathbb{P}_n]} (\mathbb{P} - \mathbb{P}_n)\|_{\ell^\infty(\mathcal{F}_1)}\\
     &\quad\quad \leq C\| (I-\mathcal{A}^2)^{-1}\|_{\mathcal{F}} \|\xi[P_m - \mathbb{P}_n]\|_\mathcal{F} \|P_m - \mathbb{P}_n\|_{\ell^\infty(\mathcal{F}_1)} \|\mathbb{P}_n - \mathbb{P}\|_{\ell^\infty(\mathcal{F}_1)} = o_p(n^{-3/2}),
 \end{align*}
 where we apply Corollary \ref{cor: op bound} in the second inequality, and Lemma \ref{lem: xi representer} in the final bound.

 For the final term, a similar argument yields that
 \begin{align*}
    & \left|\int  [(I-\mathcal{A}^2)^{-1}(I - \mathcal{A})-(I-\hat{\mathcal{A}}_n^2)^{-1} (I - \hat{\mathcal{A}}_n)]\xi [P_m - \mathbb{P}_n](x)\, d(P_m - \mathbb{P}_n)(x)\right|\\
    &\leq \|[(I-\mathcal{A}^2)^{-1}(I - \mathcal{A})-(I-\hat{\mathcal{A}}_n^2)^{-1} (I - \hat{\mathcal{A}}_n)]\xi [P_m - \mathbb{P}_n]\|_\mathcal{F} \|P_m - \mathbb{P}_n\|_{\ell^\infty(\mathcal{F}_1)}\\
    &\leq \|[(I-\mathcal{A}^2)^{-1}(I - \mathcal{A})-(I-\hat{\mathcal{A}}_n^2)^{-1} (I - \hat{\mathcal{A}}_n)]\|_\mathcal{F} \|P_m - \mathbb{P}_n\|_{\ell^\infty(\mathcal{F}_1)}^2 = o_p(n^{-3/2}).
 \end{align*}
\end{proof}

Going forward, the pairing of elements in $\mathcal{F}$ with its dual $\ell^\infty(\mathcal{F}_1)$ and the resulting bilinear form will be of particular significance, $\langle f, \mu\rangle := \mu(f)$. We can convert this into an inner-product by lifting $\mu$ to the RKHS, similarly to Lemma \ref{Hadamard Optimality}, $\langle f, \mu\rangle = \langle f, K\mu\rangle_{\mathcal{F}}$, which will be of particular significance when studying quadratic forms. Likewise, when considering linear operators $A:\ell^\infty(\mathcal{F}_1) \to \mathcal{F}$ and the corresponding quadratic form $\langle A\mu, \mu\rangle$, we can lift this to a quadratic form in the RKHS, $\langle A_K K \mu, K \mu\rangle_{\mathcal{F}}$ where we can verify the existence of such an operator $A_K$ via an analysis similar to that presented in Lemma \ref{Hadamard Optimality}. The operator $A_K$ inherits properties such as being self-adjoint or psd from $A$.

\begin{lemma}\label{lem: emp operator comp}
For $\mathcal{H}$ the Gaussian kernel RKHS,
    \[
    \|\mathcal{A}_n - \hat{\mathcal{A}}_n\|_{\mathcal{H}} = O_p(n^{-1/2}).
    \]
\end{lemma}

\begin{proof}
    We first verify that these operators map $\mathcal{H} \to \mathcal{H}$. We consider only $\mathcal{A}_n$, as the argument is similar for $\hat{\mathcal{A}}_n.$
    Observe that
    \[
    \mathcal{A}_n f = \xi_n \circ \mathfrak{M}_f (\mathbb{P}_n)
    \]
    hence it suffices that $\xi_n: \ell^\infty(\mathcal{H}_1) \to \mathcal{H}$, which is immediate by Lemma \ref{lem: xi quad comp} and \citep[Theorem 12]{berlinet2011reproducing}.

    Applying Lemma \ref{uniform kernel equiv}, for any $f\in\mathcal{H}_1$,
    \begin{align*}
    &\|(\mathcal{A}_n - \hat{\mathcal{A}}_n) f\|_{\mathcal{H}} = \|(\xi_n - \xi) \circ \mathfrak{M}_f(\mathbb{P}_n)\|_{\mathcal{H}} \leq \|\xi_n - \xi\|_{\ell^\infty(\mathcal{H}_1) \to \mathcal{H}}\|\mathfrak{M}_f\|_{\ell^\infty(\mathcal{H}_1)}\|\mathbb{P}_n\|_{\ell^\infty(\mathcal{H}_1)}\\
    &\leq  C \|\xi_n - \xi\|_{\ell^\infty(\mathcal{H}_1) \to \mathcal{H}} \|f\|_{\mathcal{H}} \|\mathbb{P}_n\|_{\ell^\infty(\mathcal{H}_1)}.
    \end{align*}
    Thus, it suffices to show $\|\xi_n - \xi\|_{\ell^\infty(\mathcal{H}_1) \to \mathcal{H}} = O_p(n^{-1/2}).$ Observe,
    \begin{align*}
        \|(\xi_n - \xi) \mu\|_{\mathcal{H}}^2 &= \langle  (\xi_n - \xi) \mu, (\xi_n - \xi) \mu\rangle_\mathcal{H}\\
        &= \langle (\xi_n - \xi)_K K\mu,(\xi_n - \xi)_K K\mu\rangle\\
       & = O_p(n^{-1})\|\mu\|_{\ell^\infty(\mathcal{H}_1)}^{2},
    \end{align*}
    where the final bound follows immediately from \ref{lem: xi quad comp}, as this yields
    \begin{align*}
    |Q_{\xi}(\mu) - Q_{\xi_n}(\mu)| &= |\langle \mu, (\xi_n - \xi) \mu \rangle | = |\langle K\mu, (\xi_n - \xi) \mu \rangle_{\mathcal{H}} |\\
    &=  |\langle K\mu, (\xi_n - \xi)_K K\mu \rangle_{\mathcal{H}} | = O_p(n^{-1/2})\|\mu\|_{\ell^\infty(\mathcal{H}_1)}^2,
    \end{align*}
    that is, the operator norm of $(\xi_n - \xi)_K$ is $O_p(n^{-1/2})$. 
\end{proof}

\begin{theorem}\label{thm: plug-in estimator}
    If $\|P_m - \mathbb{P}_n\|_{\ell^\infty(\mathcal{H}_1)} = o_p(n^{-1/2})$ and $P_m \ll \mathbb{P}_n$ then
    \[
    |\ddot{S}_n(P_m) - \ddot{S}(P_m)| = \|P_m - \mathbb{P}_n\|_{\ell^\infty(\mathcal{F}_1)}^2 \,O_p(n^{-1/2}) = o_p(n^{-3/2}). 
    \]
\end{theorem}

\begin{proof}
    By Lemma \ref{lem: discrete quad},
    \[
    |\ddot{S}_n(P_m) - \ddot{S}(P_m)| \leq |\ddot{S}_n(P_m) - \ddot{S}_{n,\xi}(P_m)| + |\ddot{S}_{n,\xi}(P_m) - \ddot{S}(P_m)| = |\ddot{S}_n(P_m) - \ddot{S}_{n,\xi}(P_m)| + o_p(n^{-3/2}),
    \]
    hence we have reduced the problem to the difference between $\ddot{S}_n$ and $\ddot{S}_{n,\xi}$. By assumption, $P_m \ll \mathbb{P}_n$, thus we can use the latter expression of Theorem \ref{thm:Sink2} to express
    \begin{align*}
    \ddot{S}_n(P_m) &= \frac{\varepsilon}{2}\left\langle P_m  - \mathbb{P}_n,(I - \mathcal{A}_n^2)^{-1} \xi_n (P_m - \mathbb{P}_n)\right\rangle,\\
    \ddot{S}_{n,\xi}(P_m) &= \frac{\varepsilon}{2}\left\langle P_m  - \mathbb{P}_n,(I - \hat{\mathcal{A}}_n^2)^{-1} \xi (P_m - \mathbb{P}_n)\right\rangle.
    \end{align*}

    Taking the difference, we have
    \begin{align*}
    \ddot{S}_n(P_m) - \ddot{S}_{n,\xi}(P_m) &= \frac{\varepsilon}{2}\left\langle P_m  - \mathbb{P}_n,(I - \mathcal{A}_n^2)^{-1} (\xi_n - \xi) (P_m - \mathbb{P}_n)\right\rangle\\
    &\quad+ \frac{\varepsilon}{2}\left\langle P_m  - \mathbb{P}_n,[(I - \mathcal{A}_n^2)^{-1} - (I - \hat{\mathcal{A}}_n^2)^{-1}]\xi_n (P_m - \mathbb{P}_n)\right\rangle
    \end{align*}

    We consider this in two parts. For the first,

    \begin{align*}
        \left\langle P_m  - \mathbb{P}_n,(I - \mathcal{A}_n^2)^{-1} (\xi_n - \xi) (P_m - \mathbb{P}_n)\right\rangle &= \int (I - \mathcal{A}_n^2)^{-1} (\xi_n - \xi) (P_m - \mathbb{P}_n) d(P_m - \mathbb{P}_n)\\
        &\leq \|P_m - \mathbb{P}_n\|_{\ell^\infty(\mathcal{H}_1)}^2\|(I - \mathcal{A}_n^2)^{-1}\|_{\mathcal{H}} \|\xi_n - \xi\|_{\ell^\infty(\mathcal{H}_1)\to \mathcal{H}}\\
        &= o_p(n^{-3/2}),
    \end{align*}
    as argued above in Lemma \ref{lem: emp operator comp}. For the second, we have
    \begin{align*}
        &\left\langle P_m  - \mathbb{P}_n,[(I - \mathcal{A}_n^2)^{-1} - (I - \hat{\mathcal{A}}_n^2)^{-1}]\xi_n (P_m - \mathbb{P}_n)\right\rangle\\
        &= \int [(I - \mathcal{A}_n^2)^{-1} - (I - \hat{\mathcal{A}}_n^2)^{-1}]\xi_n (P_m - \mathbb{P}_n) d(P_m - \mathbb{P}_n)\\
        &\leq \|(I - \mathcal{A}_n^2)^{-1} - (I - \hat{\mathcal{A}}_n^2)^{-1}\|_\mathcal{H} \|\xi_n\|_{\ell^\infty(\mathcal{H}_1)\to \mathcal{H}}\|P_m - \mathbb{P}_n\|_{\ell^\infty(\mathcal{H}_1)}^2\\
        &= o_p(n^{-3/2}),
    \end{align*}
    following from Lemma \ref{uniform kernel equiv} and \citep[Lemma 8.6]{haase2018lectures}, and \ref{lem: emp operator comp}.    
\end{proof}

\begin{proof}[Proof of Theorem \ref{thm: Sink had approx}]
This is an immediate consequence of Theorem \ref{thm:Sink2} once we show the Gaussian kernel is equivalent to the Sinkhorn kernel. We start by verifying the claim for the embedding $\xi$. This is immediate, as it is a scaling of the Gaussian kernel by a smooth function, hence we can apply results such as Lemma \ref{uniform kernel equiv}. Thus $\xi$ embeds $\ell^\infty(\mathcal{H}_1)$ into $\mathcal{H}$ by \citep[Theorem 12]{berlinet2011reproducing}, and the claim is proven as $(I - \mathcal{A})^{-1}$ is a bounded linear operator on $\mathcal{H}$, which follows by a similar analysis to that presented at the beginning of Lemma \ref{lem: emp operator comp}.
    
\end{proof}

\begin{proof}[Proof of Theorem \ref{thm: Sink compression}]
    This follows immediately from Theorem \ref{thm: Sink had approx} combined with Theorem \ref{theo: functional compression} and Lemma \ref{lem:LN_sob}
\end{proof}

\section{Results for maximum mean discrepancy}

In this subsection, we consider the ``kernel thinning" problem as a special case of data compression, where we seek to construct $P_m$ such that $Q_K(\mathbb{P}_n - P_m) = O(n^{-1})$, for $K:\mathbb{P}(\mathcal{X})\to \mathcal{F}$, $K\mu = \int k(\cdot, x) d\mu$, the kernel mean embedding associated to a generic RKHS with kernel $k$. 
Our analysis concerns the decomposition
\[
 Q_K (P_m - \mathbb{P}) \leq 2Q_K (P_m - \mathbb{P}_n) + 2Q_K (\mathbb{P}_n -  \mathbb{P}),
\]
which follows from the basic inequality $(a+b)^2\le 2(a^2+b^2)$. The latter term on the right is $O_p(n^{-1})$ by classical V-statistic theory.  

\begin{lemma}\label{lem:Vstat}
If $k$ is continuous, $Q_K(\mathbb{P}_n-\mathbb{P}) = O_p(n^{-1}).$
\end{lemma}
\begin{proof}[Proof of Lemma~\ref{lem:Vstat}]
As $k$ is continous and compactly supported, $\| k\|_\infty < \infty$, and so
\[
\mathbb{E}[k(X,Y)]^2, \mathbb{E}[k(X,X)]<\infty,
\]
and the result follows from the limiting distribution of V-statistics \cite[Chapter 6.4, Theorem B]{GVK024353353}.
\end{proof}

Our main bound is taken with minor adjustments from \citep{NEURIPS2022_2dae7d1c}. In the following, we decompose our kernel $K = K_0 + K_1$, hence we denote the corresponding RKHS norms by $\|\cdot \|_{K_1}$.

\begin{lemma}\label{lem:kernel_bound}
    Let $K = K_0 + K_1$, $k,k_0,k_1$ the corresponding kernel functions, $K,K_0,K_1\succeq 0$, and assume $P_m$ is such that $K_0(\mathbb{P}_n - P_m) = 0$ and $\int k_1(x,x)(d\mathbb{P}_n - dP_m) \geq 0$. Then,
    \[
     Q_K (P_m - \mathbb{P}_n) \leq 4 \operatorname{tr}(K_1/n).
    \]
\end{lemma}

\begin{proof}[Proof of Lemma~\ref{lem:kernel_bound}]
    Let $f = \sum_{i=1}^n a_i k(X_i,\cdot) =: Ka$. Plugging in definitions and using the assumption that $K_0(P_m-\mathbb{P}_n)=0$, we have
    \[
    \int f dP_m - \int f d\mathbb{P}_n = (P_m - \mathbb{P}_n) K a = (P_m - \mathbb{P}_n) [K_1 + K_0] a = (P_m - \mathbb{P}_n) K_1 a, 
    \]
    giving us a corresponding element $\Tilde{f}$ of the RKHS corresponding to $k_1$. Observe further that $\|\Tilde{f}\|_{K_1}^2 = a^T K_1 a \leq a^T K a = \|f\|_{K}^2$. Applying these observations, 
    \begin{align*}
    (P_m - \mathbb{P}_n) K (P_m - \mathbb{P}_n) &= \sup_{\|f\|_K\leq 1}\left(\int f dP_m - \int f d\mathbb{P}_n\right)^2 \\
    &\leq \sup_{\|f\|_{K_1}\leq 1}\left(\int f dP_m - \int f d\mathbb{P}_n\right)^2\\
    & = \sup_{\|f\|_{K_1}\leq 1}\left(\int \langle f, k_1(x,\cdot)\rangle  dP_m - \int \langle f, k_1(x,\cdot)\rangle d\mathbb{P}_n\right)^2\\
    \tag{Cauchy-Schwarz} &\leq \left(\int \sqrt{k_1(x,x)}dP_m + \int \sqrt{k_1(x,x)}d\mathbb{P}_n \right)^2\\
    &\leq 2\int k_1(x,x)dP_m + 2\int k_1(x,x) d\mathbb{P}_n\\
    &\leq 4 \int k_1(x,x)d\mathbb{P}_n = \frac{4}{n}\operatorname{tr}(K_1),
    \end{align*}
    where the final inequality holds by an assumption of the lemma.
\end{proof}

These requirements are not exceptionally stringent, as such a $P_m$ can be generated through the recombination algorithm if $K_0$ is of rank $m$ with time complexity $O(m^2n + m^3)$, as further detailed in \citep{NEURIPS2022_2dae7d1c}. Combining this result with the Nystr\"{o}m algorithm \citep{tropp2017fixedrank}, we get the following refinement of Lemma \ref{lem: MMD compress}.\\

\begin{proof}[Proof of Lemma \ref{lem: MMD compress}]
    Using Algorithm~3 and Theorem 4.1 in \citep{tropp2017fixedrank}, in $O(n^2m + \theta^3 m^3)$ time complexity, we can construct $K_0$ of rank $m$ such that
    \[
    \|K_n - K_0\|_1 \leq \left(1 + \frac{m}{m(\theta - 1) - 1}\right) \|K_n-K_{[m]}\|_1 \leq   n\left(1 + \frac{m}{m(\theta - 1) - 1}\right) T(m),
    \]
    where we denote $K_{[m]}$ to be the best rank $m$ approximation of $K_n$ with respect to the matrix 2-norm. As detailed in \citep{NEURIPS2022_2dae7d1c}, the recombination algorithm can then be applied to construct $P_m$ matching the criteria of Lemma~\ref{lem:kernel_bound}, from which the result immediately follows.
\end{proof}

\section{Gram matrix spectral decay}\label{spectralDecay}

We now provide a brief argument relating the spectral decay of a Gram matrix to the population Mercer decomposition. Define the $\varepsilon$-covering number $N(\varepsilon, \mathcal{F}, \|\cdot\|) := \min \{k: \exists f_1,\dots,f_k\in \mathcal{F},\forall f\in\mathcal{F},\exists i\in \{1,\dots,k\}\textnormal{ s.t. } \|f-f_i\|<\varepsilon\}$, and the $\varepsilon$-linear covering number to be $LN(\varepsilon, \mathcal{F}, \|\cdot\|) := \min \{k: \exists f_1,\dots,f_k\in \mathcal{F},\forall f\in\mathcal{F},\exists g\in \operatorname{span}\{f_1,\dots,f_k\}\textnormal{ s.t. } \|f-g\|<\varepsilon\}$. We call $\mathcal{L} := \operatorname{span}\{f_1,\dots,f_k\}$ an $\varepsilon$-linear cover for $f_1,\dots,f_k$ achieving the above minimum. The first notion is classical in statistical learning, relating approximation error to the size of a function class. Linear covering numbers are appropriate in the setting of low-rank matrix approximation, as it is only necessary for the matrix row/column space to be spanned by a small number of functions up to the specified error threshold. To our knowledge, this is the first work to introduce linear covering numbers.

\begin{lemma}\label{lem:supnorm}
For any matrix $\Tilde{K}$,
\[
 Q_K (P_m - \mathbb{P}_n) \leq  Q_{\Tilde{K}} (P_m - \mathbb{P}_n) + 4 \|K_n - \Tilde{K}\|_\infty.
\]

\end{lemma}

\begin{proof}
    We begin by noting that
    \begin{align*}
        (P_m - \mathbb{P}_n) K (P_m - \mathbb{P}_n)&= (P_m - \mathbb{P}_n) \Tilde{K} (P_m - \mathbb{P}_n)  + (P_m - \mathbb{P}_n) [K - \Tilde{K}] (P_m - \mathbb{P}_n) \\
        &= (P_m - \mathbb{P}_n) \Tilde{K} (P_m - \mathbb{P}_n) \\
        &\quad + \left(P_m[K - \Tilde{K}] P_m  + \mathbb{P}_n [K - \Tilde{K}] \mathbb{P}_n + 2 \mathbb{P}_n[K - \Tilde{K}]P_m\right).
    \end{align*}
    The latter term on the right upper bounds by $4 \|K_n - \Tilde{K}\|_\infty$ by H\"{o}lder's inequality with $(p,q)=(1,\infty)$. 
\end{proof}

Let $B_\mathcal{X} := \{k(x,\cdot):x\in \mathcal{X}\}$. Note that $B_\mathcal{X}$ is bounded in $\mathcal{F}$ as $\sup_{x\in \mathcal{X}} \|k(x,\cdot)\|_\mathcal{F} = \|\sqrt{k(x,x)}\|_\infty < \infty$.

\begin{lemma}\label{lem:lowrank_aprox}
     There exists $\Tilde{K}_n$ of rank less than $LN(\varepsilon, B_\mathcal{X}, \|\cdot\|_\infty)$ such that $\|K_n - \Tilde{K}_n\|_\infty < \varepsilon$.
\end{lemma}

\begin{proof}

Let $\mathcal{L}$ be an $\varepsilon$-linear cover of $B_\mathcal{X}$. For each $k(x,\cdot)$, take $g_x \in \mathcal{L}$ such that $\|k(x,\cdot) - g_x\|_\infty < \varepsilon$. If $x_1,\dots,x_n$ are the evaluation points, it follows that the matrix $A:= [g_{x_i}(x_j)]_{i,j=1}^n$ has rank no more than $LN(\varepsilon, B_\mathcal{X}, \|\cdot\|_\infty)$ and $\|K_n - A\|_\infty <\varepsilon.$
\end{proof}

\begin{lemma}\label{lem:lowrank_spec}
    For any matrix $K$ with singular values $\sigma_i$, $\varepsilon>0$, if there exists $\Tilde{K}$ of rank $R$ such that $\|K - \Tilde{K}\| < \varepsilon$, then $\sigma_{R+1} < \varepsilon$.
\end{lemma}

\begin{proof}
    This follows immediately as the rank $r$ singular value decomposition $K_{[r]}$ provides the best low-rank approximation. Precisely,
    \[
    \sigma_{R+1} = \|K - K_{[r]}\| \leq  \|K - \Tilde{K}\| = \varepsilon.
    \]
\end{proof}

We now relate the linear covering number to spectral decay.

\begin{lemma}\label{lem:LN_RKHS}
    Let $\mathcal{F}$ be an RKHS with kernel Mercer decomposition such that the eigenfunctions are uniformly bounded in sup-norm by $M$. Then, for $B$ a bounded subset of $\mathcal{F}$, there exists a sequence of linear covers $\mathcal{L}_m$ of dimension $m$, $m\in\mathbb{N}$, such that $\sup_{f\in B} \inf_{g\in \mathcal{L}_m} \|f-g\|_\infty = O(\sqrt{\sum_{i=m+1}^\infty \sigma_i}).$ That is, for $\Tilde{T}(m):= \sqrt{\sum_{i=m+1}^\infty \sigma_i}$, there exists $C>0$ such that $LN(\varepsilon,B,\|\cdot\|_\infty) = O(\Tilde{T}^{-1}(C\varepsilon))$, for $\Tilde{T}^{-1}(\varepsilon) := \min \{m : \Tilde{T}(m) \leq \varepsilon\}.$ 
\end{lemma}

\begin{proof}
    Our proof follows Lemma D.2 in \citep{yang2020function}. Define $\Pi_m$ to be the projection onto the span of the eigenfunctions $\psi_1,\dots,\psi_m$. Since $B$ is bounded and $(\sqrt{\sigma_i} \psi_i)_{i=1}^\infty$ is an orthonormal basis for $\mathcal{F}$, there exists $M'>0$ such that for all $f\in B$, $f = \sum_{i=1}^\infty w_i \sqrt{\sigma_i} \psi_i$, $\sqrt{\sum_{i=1}^\infty w_i^2} \leq M'$. It follows that
    \begin{align*}
        \|f - \Pi_m f\|_\infty &= \sup_x \left|\sum_{i=m+1}^\infty w_i \sqrt{\sigma_i} \psi_i(x)\right|\\
        &\leq \sum_{i=m+1}^\infty w_i \sqrt{\sigma_i} M\\
        & \leq M\sqrt{\sum_{i=m+1}^\infty w_i^2}\sqrt{\sum_{i=m+1}^\infty \sigma_i} \tag{Cauchy-Schwarz}\\
        & = O\left(\sqrt{\sum_{i=m+1}^\infty \sigma_i}\right).
    \end{align*}    
\end{proof}

This leads to the following decay bounds due to the compact support of the sampled data.

\begin{lemma}
    In the setting of Lemma~\ref{lem:LN_RKHS}, for $\sigma_i^n$ the $i$th singular value of $K_n/n$ the normalizd kernel Gram matrix, $\sup_n \sigma_{i}^n = O(\Tilde{T}(i))$.
\end{lemma}
\begin{proof}
    Let $\Tilde{K}_n$ be as in Lemma~\ref{lem:lowrank_aprox}. We use the relationship between the operator and element-wise norm (see \cite[Theorem 5.6.9 paired with Example 5.6.0.3]{horn2012matrix})
    \[
    \|K_n/n - \Tilde{K}_n/n\| \leq n\|K_n/n - \Tilde{K}_n/n\|_\infty = \|K_n - \Tilde{K}_n\|_\infty.
    \]    
     Hence, by Lemmas \ref{lem:lowrank_aprox} and \ref{lem:lowrank_spec}, we have, for any $\varepsilon>0$, $\sigma_{LN(\varepsilon, \mathcal{F}, \|\cdot\|_\infty) +1}^n < \varepsilon$. Thus, by Lemma \ref{lem:LN_RKHS}, as $LN(\varepsilon,B_\mathcal{X},\|\cdot\|_\infty) = O(\Tilde{T}^{-1}(\varepsilon))$, it holds that $\sigma_i^n = \sigma_{\Tilde{T}^{-1}\circ \Tilde{T}(i)}^n = O(\Tilde{T}(i))$.
\end{proof}

The following two bounds verify (1) and (2) of Theorem \ref{theo: functional compression} respectively.

\begin{lemma}\label{lem:LN_sob}
    For $B$ a bounded subset of $C^{\alpha}(\mathcal{X})$, $\alpha > d/2$, $\Tilde{T}(m) = O(m^{(1 - 2\alpha/d)/2})$. For $B$ a bounded subset of the Gaussian kernel RKHS, there exists a constant $\beta>0$ such that $\Tilde{T}(m)=O(e^{-\beta m^{1/d}})$.
\end{lemma}
\begin{proof}
    We note that the eigenfunctions of the respective kernels have been shown to be appropriately bounded as verified in \citep{yang2020function}.

    In the first case, eigenvalues of the Sobolev RKHS decay at a rate $\sigma_i = O(i^{-2\alpha/d})$, as seen in Appendix A of \citep{bach2015equivalence}. Hence, we can upper bound
    \[
    \sum_{i=m+1}^\infty \sigma_i = O\left(\int_m^\infty x^{-2\alpha/d}\right)\,dx = O(m^{1-2\alpha/d}),
    \]
    and thus $\Tilde{T}(m) = O(m^{(1 - 2\alpha/d)/2}).$

    For the Gaussian kernel RKHS, $\sigma_i = O(\exp[-2\beta i^{1/d}])$ for some $\beta>0$. Hence,
    \begin{align*}
    \sum_{i=m+1}^\infty \sigma_i &= O\left(\int_m^\infty \exp(-2\beta x^{1/d})\,dx\right)\\
    &= O\left(\int_{2\beta m^{1/d}}^\infty u^{d-1} \exp(-u)\, du\right) \\
    &= O(\Gamma(d, 2\beta m^{1/d})),
    \end{align*}
    for $\Gamma(s,x)$ the upper incomplete Gamma function. From Theorem 1.1 in \cite{pinelis2020exactlowerupperbounds}, 
    \[
    \Gamma(d, 2\beta m^{1/d}) = O(m^{(d-1)/d} \exp(-2\beta m^{1/d})) = O(\exp[-\beta m^{1/d}])
    .
    \]
    Thus, $\Tilde{T}(m)=O(\exp[-\beta m^{1/d}])$, as desired.
    
\end{proof}

\end{document}